\documentclass{siamart251216}
\usepackage[english]{babel}

\usepackage{placeins}
\usepackage{subcaption}
\usepackage{algorithm,algorithmic}
\usepackage{amsmath,amsfonts,amssymb}
\newtheorem{conjecture}{Conjecture}[theorem]
\usepackage{graphicx}
\usepackage{subcaption}
\usepackage{caption}
\title{The Frame Kernel Method for Multiscale Operator Learning\thanks{Submitted to the editors on August 17, 2026. \funding{BF, RMK, and VS were supported by Air Force Office of Scientific Research (AFOSR) LRIR grant FA9550-25-1-0042. RW was supported by AFOSR LRIR grant 24RXCOR005. VS was also supported by the U.S. National Science Foundation, Division of Mathematical Sciences, under award 2505986.}}}

\author{
Branden Frieden\thanks{Kahlert School of Computing, University of Utah,
Salt Lake City, UT 84112, USA.}
\and
Ryan Whitehead\thanks{U.S. Air Force Research Laboratory,
Wright-Patterson AFB, OH 45433, USA.}
\and
M. Keith Ballard\footnotemark[3]
\and
Robert M. Kirby\footnotemark[2]
\and
Varun Shankar\thanks{Kahlert School of Computing, University of Utah,
Salt Lake City, UT 84112, USA. Corresponding author.}
}
\begin{document}
\maketitle

\begin{abstract}
We present a natively multiscale operator learning method for the surrogate modeling of (numerical solvers for) multiscale partial differential equations (PDEs). The primary novelty of our method lies in a novel multiscale kernel frame function approximation technique. Leveraging this new kernel frame technique, we cast the operator learning problem as one of learning frame coefficients of output functions as a function of frame coefficients of input functions. The generalization step then automatically allows for a multiscale decomposition of the output functions. Our method is applicable to both tensor-product grids and point clouds. We present interpolation proofs, error estimates, and numerical convergence rates for our frame approximation. We the demonstrate the applicability of our method for the surrogate modeling of inherently multiscale PDEs. The new multiscale frame kernel method is significantly more accurate than popular neural operators on challenging problems from the literature, while simultaneously admitting an a posteriori multiscale decomposition upon generalization.
\end{abstract}



\section{Introduction}
\label{sec:intro}

Many applications require thousands or millions of solves for design, optimization, uncertainty quantification, inverse problems, control, and digital twins. Surrogate models replace these repeated solves with rapidly evaluable approximations trained from simulation or experimental data~\cite{SacksEtAl1989,KennedyOHagan2001,BennerGugercinWillcox2015,PeherstorferWillcoxGunzburger2018,BruntonNoackKoumoutsakos2020,KochkovEtAl2021}. Operator learning goes further by approximating maps between function spaces. Neural operators parametrize the map using neural networks and come in multiple varieties. Examples include deep operator networks (DeepONets)~\cite{LuEtAl2021DeepONet}; Fourier neural operators (FNOs), which parameterize nonlocal layers in Fourier space~\cite{LiEtAl2021FNO,KovachkiEtAl2023}; kernel neural operators (KNOs), which use trainable closed-form kernels and quadrature~\cite{LoweryEtAl2026KNO}; and Transolver and Transolver++, which use physics-aware attention on general, including million-point, geometries~\cite{WuEtAl2024Transolver,LuoEtAl2025TransolverPP}. Basis-to-Basis (B2B) operator learning makes the representation problem explicit by learning input and output bases together with a coefficient map~\cite{IngebrandEtAl2025B2B}.

Kernel methods offer a compelling alternative to neural operators. Batlle et al.~combine observation and recovery maps with finite-dimensional kernel regression and show that standard kernels can compete with widely used neural architectures~\cite{BatlleEtAl2024}. Turnage et al.~develop a complementary weighted least-squares theory with optimal sampling measures, stability guarantees, and explicit finite-dimensional operator spaces~\cite{TurnageEtAl2025}. These approaches build on scalar-, vector-, and operator-valued kernel theory~\cite{MicchelliPontil2005,KadriEtAl2016}, while retaining a crucial modeling choice: the user selects the input and output approximation spaces (much like in B2B or DeepONets). In each case, the representation directly controls the difficulty of the learned map.

This choice becomes central for multiscale PDEs. Heterogeneous media, singular perturbations, turbulence, porous flow, composite materials, and high-frequency waves generate interacting spatial scales that may differ by orders of magnitude~\cite{HouWu1997,EEngquist2003,WeinanE2011,EfendievGalvisHou2013}. Operator-learning methods have only recently begun to encode comparable structure. Multiwavelet and wavelet neural operators use repeated multiresolution transforms~\cite{GuptaXiaoBogdan2021,TripuraChakraborty2023}; hierarchical attention neural operators address spectral bias against fine scales~\cite{LiuEtAl2024HANO}; M2NO combines multiwavelets with algebraic multigrid~\cite{LiEtAl2024M2NO}; locally subspace-informed neural operators learn GMsFEM spaces~\cite{RudikovEtAl2025}; and MscaleFNO uses parallel rescaled Fourier branches for highly oscillatory operator maps~\cite{YouXuCai2026MscaleFNO}. These methods all recognize that multiscale surrogates require representations that expose the relevant scales.

Kernel operator learning lets us impose such a representation directly without architecture tuning, but the appropriate multiscale space is not obvious. Fourier bases provide global spectral efficiency and fast transforms on regular grids, yet encode scale through frequency and fit periodic or rectangular domains most naturally~\cite{Trefethen2000,CanutoEtAl2006,LiEtAl2021FNO}. Wavelets and multiwavelets localize in space and frequency and support sparse multiresolution representations, but require refinement structures, filter banks, or mesh-dependent machinery~\cite{Daubechies1992,Mallat2009,Cohen2003}. Hierarchical finite elements, splines, and multigrid bases handle complex PDE discretizations but usually inherit mesh-connectivity and conformity requirements~\cite{BrennerScott2008,BriggsHensonMcCormick2000,EfendievGalvisHou2013}. %Samplets extend wavelet-like analysis to scattered data through signed measures with vanishing moments and support data and kernel-matrix compression~\cite{HarbrechtMulterer2022,AvesaniEtAl2024}. 
Radial basis function (RBF) and kernel spaces work directly on scattered sites and general geometries, but a single length scale does not provide an explicit multiscale representation~\cite{Wendland2005,FasshauerMcCourt2015}.

We seek coordinates tied directly to multiple \emph{spatial length scales}, rather than frequency separation alone. For implementation simplicity, we also require one local construction that applies without modification to tensor-product grids, vertices of general simulation meshes, and unstructured point clouds; interpolates the sampled data exactly; and exposes each scale after prediction. %These requirements point to a redundant space of compactly supported kernels with scale-dependent support radii.

We introduce a multiscale kernel frame, built from compactly supported radial basis functions centered on nested subsets of the observation sites, that satisfies these requirements. Each level uses a different center density and support radius, and all levels form one finite redundant dictionary. While earlier work developed multiscale reproducing kernels, tight-frame expansions, and multilevel radial basis function algorithms with changing scales~\cite{Opfer2006Multiscale,Opfer2006TightFrame,Wendland2010,GriebelRiegerZwicknagl2015,LeGiaSloanWendland2017,AvesaniEtAl2024}, our construction differs algebraically and in purpose: it requires neither a refinement equation, an orthogonal or tight basis, a dyadic lattice, nor sequential residual correction. Instead, we determine all scale coefficients simultaneously. A kernel translate at every observation site in the finest block guarantees exact interpolation, while coarser blocks add redundancy and produce a scale-indexed representation. The same construction works on regular grids and scattered points without mesh connectivity.

We establish the frame property and exact interpolation of its canonical minimum-norm reconstruction, then derive Sobolev error estimates from scattered-zeros inequalities. The resulting baseline native-space rates do not explain the substantially faster convergence observed for smooth targets. An exact induced-kernel interpretation identifies the frame kernel as a finite multiscale autocorrelation and motivates a doubled-native-space conjecture consistent with the observed orders.

We next combine the frame with the kernel operator-learning framework of Batlle et al.~\cite{BatlleEtAl2024}, which we call the \emph{vanilla kernel method} (VKM). Our \emph{frame kernel method} (FKM) represents every input and output in a multiscale frame, normalizes input coefficients level by level, and organizes the predicted output coefficients by frame level. FKM thus maps between explicit multiscale coordinates rather than raw nodal values. It uses the same training data, yet our benchmarks show reductions of $1$--$4$ orders of magnitude in prediction error relative to VKM and competing operator-learning methods, including neural operators. Because the predicted coefficients retain their frame-level organization, FKM also provides an a posteriori scale-indexed decomposition of every prediction.

The remainder of the paper is organized as follows. Section~\ref{sec:methods} defines the multiscale kernel frame, its nested center hierarchy, the required numerical linear algebra, and the FKM. Section~\ref{sec:frame_interpolation_error} proves exact interpolation, derives the baseline Sobolev estimate, and states the doubled-order conjecture. Section~\ref{sec:results} presents the frame-approximation (Section~\ref{sec:frame_results}) and operator-learning results (Section~\ref{sec:ol_results}), including scale-resolved reconstructions (Section~\ref{sec:mult-analysis}). Section~\ref{sec:conclusions} summarizes the results and discusses future directions.

\section{Methods}
\label{sec:methods}

\subsection{Multiscale kernel-frame approximation}
\label{sec:multiscale-frame}

Let $\Omega\subset\mathbb{R}^d$, and let
$X=\{x_i\}_{i=1}^M\subset\Omega$ denote a set of sampling sites with observations
$u_i\approx u(x_i)$, collected in
$\mathbf{u}=[u_1,\ldots,u_M]^\top\in\mathbb{R}^M$.
We construct a nested hierarchy of primary center sets,
\begin{equation}
C_0=X\supset C_1\supset\cdots\supset C_J,
\qquad C_j\subset X.
\label{eq:center-hierarchy}
\end{equation}
Level $0$ is the finest level, while increasing $j$ corresponds to progressively coarser center sets. When additional boundary coverage is needed, we augment $C_j$ with a set $G_j$ of auxiliary centers and write
\[
\Xi_j=C_j\cup G_j.
\]
The primary sets $C_j$ remain nested, whereas the auxiliary centers need not belong to $X$ or satisfy a nesting relation. If no boundary augmentation is used, then $G_j=\varnothing$ and $\Xi_j=C_j$.

At level $j$, let $\rho_j>0$ be a support radius and define
\begin{equation}
\psi_{j,k}(x)
=
\phi\!\left(\frac{\|x-\xi_{j,k}\|_2}{\rho_j}\right),
\qquad
\xi_{j,k}\in\Xi_j,
\qquad
k=1,\ldots,N_j,
\label{eq:frame-atoms}
\end{equation}
where $N_j=|\Xi_j|$. We choose $\phi$ from the Wendland family of compactly supported positive-definite radial basis functions~\cite{Wendland1995,Wendland2005,FasshauerMcCourt2015}. The multiscale approximant is
\begin{equation}
\widehat{u}(x)
=
\sum_{j=0}^J\sum_{k=1}^{N_j}c_{j,k}\psi_{j,k}(x).
\label{eq:multiscale-approximant}
\end{equation}
Outer level weights could be introduced by replacing $\psi_{j,k}$ with $w_j\psi_{j,k}$. We omit them here because the principal scale dependence is already encoded through $\rho_j$.

Because $C_0=X$ and the coarser levels contribute additional basis functions, the total number of functions,
\[
N=\sum_{j=0}^J N_j,
\]
typically exceeds the number $M$ of observations. The representation is therefore redundant: a sampled function may be reproduced by more than one coefficient vector. This redundancy is intentional. The finest level provides sufficient resolution at the sampling sites, while the additional levels allow the representation to distribute information across multiple spatial scales.

The natural mathematical language for such a redundant spanning family is that of a frame. A collection $\{a_n\}_{n=1}^N\subset\mathbb{R}^M$ is a frame for $\mathbb{R}^M$ if there exist constants $0<a\leq b<\infty$ such that
\begin{equation}
a\|v\|_2^2
\leq
\sum_{n=1}^N|\langle v,a_n\rangle|^2
\leq
b\|v\|_2^2,
\qquad
v\in\mathbb{R}^M.
\label{eq:frame-definition}
\end{equation}
In finite dimensions, the lower frame bound is equivalent to the vectors spanning $\mathbb{R}^M$, while $a$ and $b$ quantify the stability of that spanning property~\cite{DuffinSchaeffer1952,Christensen2016}. In the present construction, the frame vectors are the sampled basis functions
\[
a_{j,k}
=
[\psi_{j,k}(x_1),\ldots,\psi_{j,k}(x_M)]^\top,
\]
namely, the columns of the full evaluation matrix defined below. Theorem~\ref{thm:multiscale_frame_interpolation} shows that their union forms a frame for the discrete data space and that the resulting minimum-norm approximation interpolates the observations.

\paragraph{Scale selection}
We choose
\begin{equation}
\rho_j=\eta\,2^j s_0,
\label{eq:rho-schedule}
\end{equation}
where $\eta>0$ controls support overlap and $s_0$ is a representative local spacing on the finest level. This quantity is used only to set the support radii and is distinct from the fill distance used in the convergence analysis. On a tensor-product grid, we take $s_0$ to be the minimum grid spacing. For a general point cloud, we use the median nearest-neighbor distance,
\begin{equation}
s_0
=
\operatorname{median}_{x\in X}
\min_{x'\in X\setminus\{x\}}\|x-x'\|_2,
\label{eq:point-cloud-spacing}
\end{equation}
which provides a robust estimate of the local sampling scale~\cite{Wendland2005,Buhmann2003}.

\paragraph{Discrete evaluation operators}
For each level, define the sparse evaluation matrix
\begin{equation}
B_j\in\mathbb{R}^{M\times N_j},
\qquad
(B_j)_{i,k}
=
\phi\!\left(\frac{\|x_i-\xi_{j,k}\|_2}{\rho_j}\right).
\label{eq:level-evaluation-matrix}
\end{equation}
Compact support implies $(B_j)_{i,k}=0$ whenever $\|x_i-\xi_{j,k}\|_2>\rho_j$. The full evaluation matrix is the block concatenation
\begin{equation}
A
=
[\,B_0\;B_1\;\cdots\;B_J\,]
\in\mathbb{R}^{M\times N}.
\label{eq:A-def}
\end{equation}
We assemble each block using radius queries implemented with a k-d-tree range-search structure~\cite{Bentley1975}.

For query sites $X_q=\{x_{q,i}\}_{i=1}^Q$, the same construction gives
\[
B_{qj}\in\mathbb{R}^{Q\times N_j},
\qquad
(B_{qj})_{i,k}=\psi_{j,k}(x_{q,i}),
\]
and
\begin{equation}
A_q=[\,B_{q0}\;B_{q1}\;\cdots\;B_{qJ}\,].
\label{eq:query-evaluation-matrix}
\end{equation}
For a coefficient vector $\mathbf{c}$, evaluation at the query sites is then $\widehat{u}(X_q)=A_q\mathbf{c}$.

\subsubsection{Nested center selection}
\label{sec:nested-center-selection}

The primary center sets satisfy $C_{j+1}\subset C_j\subset X$. Auxiliary centers may subsequently be added near the boundary to improve geometric coverage.

\paragraph{Dyadic thinning on tensor-product grids}
Assume that $X$ is a $d$-dimensional tensor product grid with sizes $n_1,\ldots,n_d$ and a fixed flattening order consistent with Cartesian product indexing. At level $j$, let $s_j=2^j$ and define
\begin{equation}
\mathcal{I}_j
=
\left\{
(i_1,\ldots,i_d):
i_\ell\in\{1,1+s_j,1+2s_j,\ldots\},
\quad \ell=1,\ldots,d
\right\}.
\label{eq:dyadic-index-set}
\end{equation}
The corresponding primary center set is
\[
C_j=\{x(i):i\in\mathcal{I}_j\}.
\]
Because every index selected with stride $2^{j+1}$ is also selected with stride $2^j$, we have $C_{j+1}\subset C_j$. The center spacing therefore grows geometrically with $j$, consistently with the support-radius schedule~\eqref{eq:rho-schedule}. This construction is the grid analogue of multiresolution subsampling used in multilevel kernel approximation~\cite{Wendland2005,FasshauerMcCourt2015}.

\paragraph{Farthest-point thinning on point clouds}
For a general point cloud, we construct the nested hierarchy using farthest-first traversal, also known as Gonzalez traversal~\cite{Gonzalez1985}. Beginning from an arbitrary point $p_1\in X$, define
\begin{equation}
p_{m+1}
\in
\operatorname*{arg\,max}_{x\in X}
\min_{1\leq r\leq m}\|x-p_r\|_2.
\label{eq:farthest-first}
\end{equation}
This produces a single ordering $p_m=x_{\pi(m)}$ of the points in $X$. We compute the ordering once and define each primary center set as a prefix,
\begin{equation}
C_j=\{p_1,\ldots,p_{m_j}\},
\qquad
M=m_0>m_1>\cdots>m_J,
\label{eq:farthest-prefixes}
\end{equation}
so that $C_{j+1}\subset C_j$ automatically. The target cardinalities may be chosen to mimic geometric coarsening, for example $m_j\approx M/2^{jd}$, or to produce a prescribed geometric increase in separation distance.

Farthest-first traversal provides good coverage of the observed point cloud but does not guarantee comparable coverage outside its convex hull. When evaluation is performed on a bounding box, poorly covered boundary layers can dominate the error. We therefore augment each primary set $C_j$, when necessary, with auxiliary samples $G_j$ placed on the boundary of the evaluation region~\cite{Wright2002}. These auxiliary centers are included in the level dictionary $\Xi_j=C_j\cup G_j$ but are not part of the nested primary hierarchy.

\subsubsection{Linear algebra}
\label{sec:frame-linear-algebra}

Collect the coefficients levelwise as
\[
\mathbf{c}
=
[\mathbf{c}_0^\top,\ldots,\mathbf{c}_J^\top]^\top
\in\mathbb{R}^N,
\qquad
\mathbf{c}_j\in\mathbb{R}^{N_j}.
\]
Then $\widehat{u}(X)=A\mathbf{c}$. Because $A$ is wide, the interpolation equations generally admit multiple solutions. We select the canonical minimum-norm coefficient vector
\begin{equation}
\mathbf{c}(\mathbf{u})
=
\operatorname*{arg\,min}_{\mathbf{z}\in\mathbb{R}^N}
\|\mathbf{z}\|_2
\quad\text{subject to}\quad
A\mathbf{z}=\mathbf{u}.
\label{eq:minimum-norm-coefficients}
\end{equation}
Theorem~\ref{thm:multiscale_frame_interpolation} guarantees that $A$ has full row rank, so~\eqref{eq:minimum-norm-coefficients} is feasible for every $\mathbf{u}\in\mathbb{R}^M$ and has a unique solution. We compute this solution using a thin QR factorization of $A^\top$,
\begin{equation}
A^\top=QR,
\qquad
Q\in\mathbb{R}^{N\times M},
\qquad
R\in\mathbb{R}^{M\times M},
\label{eq:frame-qr}
\end{equation}
where $Q^\top Q=I_M$ and $R$ is nonsingular and upper triangular. Solving
\[
R^\top\mathbf{y}=\mathbf{u}
\]
and setting
\begin{equation}
\mathbf{c}(\mathbf{u})=Q\mathbf{y}
\label{eq:frame-qr-solution}
\end{equation}
yields the minimum-norm interpolating coefficients. Equivalently,
\[
\mathbf{c}(\mathbf{u})
=A^\top(AA^\top)^{-1}\mathbf{u}.
\]
Thus, a single factorization of $A^\top$ can be reused for every function sampled on the same sites. In our experiments, this construction required no additional regularization. Reconstruction at the observation and query sites is additive across levels:
\begin{equation}
\widehat{u}(X)
=
\sum_{j=0}^J B_j\mathbf{c}_j,
\qquad
\widehat{u}(X_q)
=
\sum_{j=0}^J B_{qj}\mathbf{c}_j.
\label{eq:frame-reconstruction}
\end{equation}

\subsection{The frame kernel method for multiscale operator learning}
\label{sec:fkm}

We now use the multiscale frame representation within a kernel method for operator learning~\cite{BatlleEtAl2024}. Let
\[
\mathcal{G}:\mathcal{U}\rightarrow\mathcal{V}
\]
be an operator, and suppose that the training data consist of pairs
\[
\{(u_\ell,v_\ell)\}_{\ell=1}^{N_T},
\qquad
v_\ell=\mathcal{G}(u_\ell).
\]
The input and output functions may be defined on different domains and sampled at different sets of sites. We therefore construct separate input and output frames with evaluation matrices
\[
A_u\in\mathbb{R}^{M_u\times N_u},
\qquad
A_v\in\mathbb{R}^{M_v\times N_v}.
\]
The two frames may differ in their numbers of levels, centers, and coefficients.

The vanilla kernel method (VKM) learns a map directly from sampled input functions to sampled output functions. The proposed \emph{frame kernel method} (FKM) instead uses input-frame coefficients as kernel features and organizes the output coefficients according to the levels of the multiscale frame. For each training pair, we compute the canonical coefficient vectors
\begin{equation}
\mathbf{c}_{u_\ell}=A_u^\dagger u_\ell(X_u),
\qquad
\mathbf{c}_{v_\ell}=A_v^\dagger v_\ell(X_v),
\label{eq:fkm-frame-coefficients}
\end{equation}
using the QR procedure in~\eqref{eq:frame-qr}--\eqref{eq:frame-qr-solution}. Partition these vectors according to their frame levels:
\[
\mathbf{c}_{u_\ell}
=
[\mathbf{c}_{u_\ell,0}^\top,\ldots,
 \mathbf{c}_{u_\ell,J_u}^\top]^\top,
\qquad
\mathbf{c}_{u_\ell,j}\in\mathbb{R}^{N_{u,j}},
\]
and
\[
\mathbf{c}_{v_\ell}
=
[\mathbf{c}_{v_\ell,0}^\top,\ldots,
 \mathbf{c}_{v_\ell,J_v}^\top]^\top,
\qquad
\mathbf{c}_{v_\ell,j}\in\mathbb{R}^{N_{v,j}}.
\]

\paragraph{Scale-dependent normalization of the input coefficients}
The coefficient magnitudes can differ substantially across input-frame levels because the levels use different center densities and support radii. If kernel distances are computed directly from the raw coefficient vectors, levels with larger coefficients can dominate the induced geometry. We therefore normalize each input block by a single scalar estimated from the training set. Define
\begin{equation}
D_j
=
[\mathbf{c}_{u_1,j},\ldots,\mathbf{c}_{u_{N_T},j}]
\in\mathbb{R}^{N_{u,j}\times N_T}.
\label{eq:fkm-training-block}
\end{equation}
For level $j$, set
\begin{equation}
\beta_j
=
\left(
\frac{1}{N_{u,j}N_T}
\sum_{\ell=1}^{N_T}
\|\mathbf{c}_{u_\ell,j}\|_2^2
+\varepsilon
\right)^{-1/2},
\label{eq:fkm-level-normalization}
\end{equation}
where $\varepsilon>0$ prevents division by zero. The normalized input coefficients are
\begin{equation}
\widetilde{\mathbf{c}}_{u_\ell,j}
=\beta_j\mathbf{c}_{u_\ell,j},
\qquad
\widetilde{\mathbf{c}}_{u_\ell}
=
[\widetilde{\mathbf{c}}_{u_\ell,0}^\top,\ldots,
 \widetilde{\mathbf{c}}_{u_\ell,J_u}^\top]^\top.
\label{eq:fkm-normalized-features}
\end{equation}
If
\[
s_j
=
\frac{1}{N_{u,j}N_T}
\sum_{\ell=1}^{N_T}
\|\mathbf{c}_{u_\ell,j}\|_2^2,
\]
then
\begin{equation}
\frac{1}{N_{u,j}N_T}
\sum_{\ell=1}^{N_T}
\|\widetilde{\mathbf{c}}_{u_\ell,j}\|_2^2
=
\frac{s_j}{s_j+\varepsilon}.
\label{eq:fkm-normalized-rms}
\end{equation}
Hence, whenever $s_j\gg\varepsilon$, the transformation gives each level unit empirical mean-squared coefficient magnitude per degree of freedom. The factors $\beta_j$ are estimated from the training inputs only and then applied unchanged to validation, test, and prediction inputs. No centering or covariance transformation is performed.

\paragraph{Levelwise multi-output kernel regression}
Let
\[
\kappa:\mathbb{R}^{N_u}\times\mathbb{R}^{N_u}\rightarrow\mathbb{R}
\]
be a positive-definite kernel, where $N_u=\sum_{j=0}^{J_u}N_{u,j}$, and define the training Gram matrix
\begin{equation}
K_{\ell m}
=
\kappa(\widetilde{\mathbf{c}}_{u_\ell},
       \widetilde{\mathbf{c}}_{u_m}),
\qquad
K\in\mathbb{R}^{N_T\times N_T}.
\label{eq:fkm-gram-matrix}
\end{equation}
For output level $j$, collect the corresponding training coefficients rowwise as
\begin{equation}
Y_j
=
\begin{bmatrix}
\mathbf{c}_{v_1,j}^\top\\
\vdots\\
\mathbf{c}_{v_{N_T},j}^\top
\end{bmatrix}
\in\mathbb{R}^{N_T\times N_{v,j}}.
\label{eq:fkm-output-block}
\end{equation}
For each output level, we solve
\begin{equation}
W_j=(K+\lambda I_{N_T})^{-1}Y_j,
\qquad
j=0,\ldots,J_v,
\label{eq:fkm-levelwise-training}
\end{equation}
where $\lambda>0$ is the regularization parameter. Because the same Gram matrix and regularization parameter are used at every level, these solves are algebraically equivalent to a single multi-output kernel ridge regression with the concatenated response matrix
\[
Y=[\,Y_0\;\cdots\;Y_{J_v}\,].
\]
We retain the levelwise form because it exposes the scale-indexed coefficient blocks used in reconstruction, while allowing the factorization of $K+\lambda I_{N_T}$ to be reused for every right-hand side.

For a new input function $u_*$, we first compute its frame coefficients $\mathbf{c}_{u_*}$ and apply the training-set normalization factors,
\begin{equation}
\widetilde{\mathbf{c}}_{u_*,j}
=
\beta_j\mathbf{c}_{u_*,j},
\qquad
\widetilde{\mathbf{c}}_{u_*}
=
[\widetilde{\mathbf{c}}_{u_*,0}^\top,\ldots,
 \widetilde{\mathbf{c}}_{u_*,J_u}^\top]^\top.
\label{eq:fkm-test-normalization}
\end{equation}
Define
\begin{equation}
\mathbf{k}_*
=
\begin{bmatrix}
\kappa(\widetilde{\mathbf{c}}_{u_*},\widetilde{\mathbf{c}}_{u_1})\\
\vdots\\
\kappa(\widetilde{\mathbf{c}}_{u_*},\widetilde{\mathbf{c}}_{u_{N_T}})
\end{bmatrix}
\in\mathbb{R}^{N_T}.
\label{eq:fkm-cross-kernel-vector}
\end{equation}
The predicted coefficient block at output level $j$ is
\begin{equation}
\widehat{\mathbf{c}}_{v_*,j}
=
W_j^\top\mathbf{k}_*
=
Y_j^\top(K+\lambda I_{N_T})^{-1}\mathbf{k}_*.
\label{eq:fkm-levelwise-prediction}
\end{equation}
The predicted coefficients remain organized by frame level and are combined during reconstruction. The distinction from VKM therefore lies in the multiscale representation of the inputs and outputs, rather than in a different algebraic form of the shared multi-output kernel ridge-regression solve.

Finally, the predicted output function is synthesized from the levelwise coefficient blocks:
\begin{equation}
\widehat{\mathcal{G}}(u_*)(x)
=
\sum_{j=0}^{J_v}
\sum_{k=1}^{N_{v,j}}
(\widehat{\mathbf{c}}_{v_*,j})_k\,
\psi^v_{j,k}(x).
\label{eq:fkm-output-reconstruction}
\end{equation}
At output query sites $X_{v,q}$, this becomes
\begin{equation}
\widehat{\mathcal{G}}(u_*)(X_{v,q})
=
\sum_{j=0}^{J_v}
B^v_{qj}\widehat{\mathbf{c}}_{v_*,j}.
\label{eq:fkm-output-query-reconstruction}
\end{equation}
Relative to VKM, FKM changes the representation used by the regression stage: it forms kernel similarities from normalized multiscale input-frame coefficients, organizes the predicted output coefficients by frame level, and synthesizes the resulting function from those levelwise blocks.

\section{Exact interpolation and convergence rates}
\label{sec:frame_interpolation_error}

\subsection{Frame property and exact interpolation}
\label{sec:exact_interp}
The sampled basis functions
\[
a_{j,k}
=
[\psi_{j,k}(x_1),\ldots,\psi_{j,k}(x_M)]^\top
\]
are the columns of the multiscale evaluation matrix $A$. The following result shows that they form a finite-dimensional frame for $\mathbb{R}^M$ and that the redundant approximation interpolates arbitrary data on the observation set~\cite{DuffinSchaeffer1952,Christensen2016}. The proof uses only strict positive definiteness at the finest level.

We distinguish the number $M=|X|$ of sampling sites from the number $N_j=|\Xi_j|$ of frame functions at level $j$. If auxiliary boundary centers are present, then $N_0$ may exceed $M$. Let $B_{0,X}\in\mathbb{R}^{M\times M}$ denote the submatrix of $B_0$ associated with the primary centers $C_0=X$; if level zero contains no auxiliary centers, then $B_{0,X}=B_0$.

\begin{theorem}[Frame property and exact interpolation]
\label{thm:multiscale_frame_interpolation}
Let $X=\{x_i\}_{i=1}^M\subset\Omega$ consist of distinct sites, let
$C_0=X\supset C_1\supset\cdots\supset C_J$, and suppose that
$\phi(\|\cdot\|_2/\rho_0)$ is strictly positive definite on $\mathbb{R}^d$.
Then the columns of $A=[\,B_0\;\cdots\;B_J\,]$ form a finite-dimensional frame for $\mathbb{R}^M$, with
\begin{equation}
\sigma_{\min}(B_{0,X})^2\|v\|_2^2
\leq \|A^\top v\|_2^2
\leq \|A\|_2^2\|v\|_2^2,
\qquad v\in\mathbb{R}^M.
\label{eq:multiscale_frame_bounds}
\end{equation}
Consequently, every $\mathbf{u}\in\mathbb{R}^M$ admits coefficients $c_{j,k}$ satisfying
$\widehat{u}(x_i)=u_i$, $i=1,\ldots,M$. If $A^\top=QR$ is a thin QR factorization, then
\[
\mathbf{c}^\star=QR^{-\top}\mathbf{u}
\]
is the unique minimum-$\ell^2$-norm coefficient vector producing this interpolant.
\end{theorem}

\begin{proof}
Since $C_0=X$, the matrix $B_{0,X}$ is, up to a column permutation, the kernel matrix
\[
(B_{0,X})_{i,k}
=
\phi\!\left(\frac{\|x_i-x_k\|_2}{\rho_0}\right).
\]
Strict positive definiteness and distinctness of the sites make $B_{0,X}$ nonsingular~\cite{Wendland1995,Wendland2005}. Hence
\[
\|A^\top v\|_2^2
=
\sum_{j=0}^J\|B_j^\top v\|_2^2
\geq
\|B_{0,X}^\top v\|_2^2
\geq
\sigma_{\min}(B_{0,X})^2\|v\|_2^2,
\]
and the upper bound follows from $\|A^\top v\|_2\leq\|A\|_2\|v\|_2$. Thus $A$ has full row rank. Since $A=R^\top Q^\top$,
\[
AQR^{-\top}=I_M,
\]
so $\mathbf{c}^\star=QR^{-\top}\mathbf{u}$ interpolates the data. Moreover,
\[
QR^{-\top}=A^\top(AA^\top)^{-1}=A^\dagger,
\]
which is the Moore--Penrose pseudoinverse and hence the minimum-norm right inverse~\cite{GolubVanLoan2013}.
\end{proof}

The finest primary block therefore guarantees exact interpolation, while the remaining columns add redundancy and allow the canonical QR solution to distribute the representation across scales.

\subsection{Convergence rates}
\label{sec:conv_rates}
To express the error directly in terms of the number $M$ of sampling sites, define the fill distance
\begin{equation}
h_M
:=
h_{X_M,\Omega}
=
\sup_{x\in\Omega}\min_{x_i\in X_M}\|x-x_i\|_2.
\label{eq:fill_distance}
\end{equation}
We assume that the sampling sets under consideration satisfy
\begin{equation}
h_M\leq C_h M^{-1/d},
\label{eq:fill_distance_cardinality}
\end{equation}
with $C_h$ independent of $M$. No assumption is made on the separation radius or mesh ratio. Thus, the analysis below depends only on coverage of $\Omega$ through $h_M$, and not on a lower bound for pairwise distances between sampling sites.

For the $C^{2\ell}$ Wendland kernel $\phi_{d,\ell}$, the native space is norm-equivalent to $H^{\tau_\ell}(\mathbb{R}^d)$, where
\[
\tau_\ell:=\frac d2+\ell+\frac12,
\]
and the cases $\ell=1,2,3$ correspond to the $C^2$, $C^4$, and $C^6$ kernels~\cite{Wendland1998,Wendland2005,SchabackWendland2006}. Translation and positive scaling preserve this Sobolev regularity, so every finite multiscale interpolant belongs to $H^{\tau_\ell}(\Omega)$.

Let $E_M:H^{\tau_\ell}(\Omega)\to\mathbb{R}^M$ denote sampling on $X_M$, let $T_M:\mathbb{R}^{N(M)}\to H^{\tau_\ell}(\Omega)$ denote synthesis by the multiscale frame functions, and write
\[
\mathcal{I}_M=T_MA_M^\dagger E_M.
\]
Because $\tau_\ell>d/2$, the sampling map $E_M$ is bounded by Sobolev embedding. For each fixed discretization, $A_M^\dagger$ is a bounded finite-dimensional map and $T_M$ is bounded because it synthesizes a finite collection of functions in $H^{\tau_\ell}(\Omega)$. Hence $\mathcal{I}_M:H^{\tau_\ell}(\Omega)\to H^{\tau_\ell}(\Omega)$ is bounded for each fixed $M$~\cite{AdamsFournier2003}. Define
\begin{equation}
\Lambda_{\ell,J}(M)
:=
\sup_{0\neq u\in H^{\tau_\ell}(\Omega)}
\frac{\|\mathcal{I}_Mu\|_{H^{\tau_\ell}(\Omega)}}
     {\|u\|_{H^{\tau_\ell}(\Omega)}}.
\label{eq:frame_stability_factor}
\end{equation}
Thus $\Lambda_{\ell,J}(M)<\infty$ for every fixed discretization; uniform boundedness in $M$ is a separate stability question.

\begin{theorem}[Sobolev error estimate]
\label{thm:multiscale_sobolev_error}
Let $\Omega\subset\mathbb{R}^d$ be a bounded Lipschitz domain satisfying an interior cone condition, and let $X_M\subset\Omega$ be sets of $M$ distinct sampling sites satisfying~\eqref{eq:fill_distance_cardinality}. Let $\phi=\phi_{d,\ell}$ be the $C^{2\ell}$ Wendland kernel. Then there exist $h_*>0$ and $C>0$, independent of $M$ and $u$, such that, whenever $h_M\leq h_*$, every $u\in H^{\tau_\ell}(\Omega)$ satisfies
\begin{equation}
\|u-\mathcal{I}_Mu\|_{H^\mu(\Omega)}
\leq
C M^{-(\tau_\ell-\mu)/d}
\bigl(1+\Lambda_{\ell,J}(M)\bigr)
\|u\|_{H^{\tau_\ell}(\Omega)},
\qquad 0\leq\mu\leq\tau_\ell.
\label{eq:multiscale_sobolev_error}
\end{equation}
If $\Lambda_{\ell,J}(M)$ remains bounded as $M\to\infty$, then
\[
\|u-\mathcal{I}_Mu\|_{H^\mu(\Omega)}
=
\mathcal{O}\!\left(M^{-(\tau_\ell-\mu)/d}\right).
\]
\end{theorem}

\begin{proof}
Set $e:=u-\mathcal{I}_Mu$. Theorem~\ref{thm:multiscale_frame_interpolation} gives $e|_{X_M}=0$. For $h_M\leq h_*$, the scattered-zeros inequality gives
\[
\|e\|_{H^\mu(\Omega)}
\leq
C h_M^{\tau_\ell-\mu}\|e\|_{H^{\tau_\ell}(\Omega)}
\]
with $C$ independent of $M$ and $e$~\cite{NarcowichWardWendland2005,Wendland2005}. Using~\eqref{eq:fill_distance_cardinality},
\[
\|e\|_{H^\mu(\Omega)}
\leq
C M^{-(\tau_\ell-\mu)/d}\|e\|_{H^{\tau_\ell}(\Omega)}.
\]
Finally, the definition of $\Lambda_{\ell,J}(M)$ and the triangle inequality give
\[
\|e\|_{H^{\tau_\ell}(\Omega)}
\leq
\bigl(1+\Lambda_{\ell,J}(M)\bigr)
\|u\|_{H^{\tau_\ell}(\Omega)},
\]
which proves~\eqref{eq:multiscale_sobolev_error}.
\end{proof}

\paragraph{Example} \label{ex:wendland_frame_rates} For $d=2$, $\tau_\ell=\ell+3/2$. Under bounded stability, the baseline $L^2$ rates are \begin{equation} C^2:\ \mathcal{O}\!\left((M^{1/2})^{-5/2}\right), \qquad C^4:\ \mathcal{O}\!\left((M^{1/2})^{-7/2}\right), \qquad C^6:\ \mathcal{O}\!\left((M^{1/2})^{-9/2}\right). \label{eq:wendland_l2_rates} \end{equation} 
The corresponding $H^1$ rates are \[ C^2:\ \mathcal{O}\!\left((M^{1/2})^{-3/2}\right), \qquad C^4:\ \mathcal{O}\!\left((M^{1/2})^{-5/2}\right), \qquad C^6:\ \mathcal{O}\!\left((M^{1/2})^{-7/2}\right). \] The observed discrete relative $\ell_2$ orders of approximately $5.3$, $7.3$, and $8.75$, measured against $M^{1/2}$, for the $C^2$, $C^4$, and $C^6$ kernels are numerically close to twice the baseline exponents. These results suggest superconvergence beyond Theorem~\ref{thm:multiscale_sobolev_error}, while the theorem itself concerns continuous Sobolev norms.

\subsection{A route to superconvergence}
\label{sec:superconvergence_conjecture}

The baseline estimate uses only exact interpolation and a scattered-zeros inequality. For smooth targets, the observed rates instead suggest the classical doubling phenomenon associated with source conditions in kernel interpolation~\cite{Schaback1999,SloanKaarnioja2025}. The minimum-coefficient-norm property alone, however, does not place the QR interpolant in the standard native-space orthogonal-projection framework used by classical doubling arguments.

The QR approximation nevertheless has an exact kernel interpretation. Let $\Psi_M(x)\in\mathbb{R}^{N(M)}$ collect all multiscale basis functions and define
\begin{equation}
K_M(x,y)
:=
\Psi_M(x)^\top\Psi_M(y)
=
\sum_{j=0}^{J(M)}\sum_{k=1}^{N_j}
\psi_{j,k}(x)\psi_{j,k}(y).
\label{eq:induced_frame_kernel_definition}
\end{equation}
This feature-map construction defines a positive-semidefinite kernel~\cite{Aronszajn1950}. Since
\[
\mathbf{c}^\star=A^\top(AA^\top)^{-1}\mathbf{u}_M,
\qquad
K_M(X_M,X_M)=AA^\top,
\]
the QR interpolant satisfies
\begin{equation}
\mathcal{I}_Mu(x)
=
K_M(x,X_M)K_M(X_M,X_M)^{-1}\mathbf{u}_M.
\label{eq:induced_frame_kernel}
\end{equation}
Thus, the multiscale QR approximation is itself kernel interpolation with a finite, discretization-dependent multiscale autocorrelation kernel. Unfortunately, this identity does not by itself imply a doubled rate. To see the obstruction, write
\[
A=[\,B_{0,X}\;C\,],
\qquad
D:=B_{0,X}^{-1}C,
\qquad
\mathbf{a}_0:=B_{0,X}^{-1}\mathbf{u}_M,
\]
where $C$ contains every frame column not belonging to the primary finest block. Let $T_0$ and $T_c$ denote synthesis by the corresponding finest and remaining basis functions. Writing the coefficient vector as $(\mathbf{a},\mathbf{b})$, the interpolation constraint is
\[
\mathbf{a}+D\mathbf{b}=\mathbf{a}_0.
\]
Hence the minimum-norm problem reduces to
\[
\min_{\mathbf{b}}
\|\mathbf{a}_0-D\mathbf{b}\|_2^2+\|\mathbf{b}\|_2^2,
\]
whose normal equations are
\[
(I+D^\top D)\mathbf{b}=D^\top\mathbf{a}_0.
\]
Substituting $\mathbf{a}=\mathbf{a}_0-D\mathbf{b}$ into $T_0\mathbf{a}+T_c\mathbf{b}$ gives
\begin{equation}
\mathcal{I}_Mu-\mathcal{S}_Mu
=
(T_c-T_0D)(I+D^\top D)^{-1}D^\top \mathbf{a}_0,
\label{eq:exact_enrichment_identity}
\end{equation}
where $\mathcal{S}_Mu:=T_0\mathbf{a}_0$ is the ordinary finest-level interpolant. Although
\[
\|(I+D^\top D)^{-1}D^\top\|_2\leq\frac12,
\]
this estimate contains no factor of $M^{-\tau_\ell/d}$ and controls coefficient space rather than the $H^{\tau_\ell}$ norm. A superconvergence proof therefore requires additional function-space stability and approximation estimates.

Let
\[
W_M:=\operatorname{range}(\mathcal{I}_M)
\subset H^{\tau_\ell}(\Omega).
\]
With $\mathcal{I}_M=T_MA_M^\dagger E_M$ and $E_MT_M=A_M$, the Moore--Penrose identity $A_M^\dagger A_MA_M^\dagger=A_M^\dagger$ gives
\[
\mathcal{I}_M^2
=
T_MA_M^\dagger A_MA_M^\dagger E_M
=
\mathcal{I}_M.
\]
Hence $\mathcal{I}_M$ is a projector onto $W_M$. The following result isolates sufficient conditions for doubling.

\begin{theorem}[Conditional doubled-rate estimate]
\label{thm:conditional_multiscale_doubling}
Let $\Omega\subset\mathbb{R}^d$ be a bounded Lipschitz domain satisfying an interior cone condition, and let $X_M\subset\Omega$ be sets of $M$ distinct sampling sites satisfying~\eqref{eq:fill_distance_cardinality}. Let $\mathcal{I}_M$ be the associated multiscale QR interpolants. Suppose $(\mathcal{B}_{\ell,0}(\Omega),\|\cdot\|_{\mathcal{B}_{\ell,0}(\Omega)})$ is a normed space continuously embedded in $H^{2\tau_\ell}(\Omega)$ and that there are constants $C_{\mathrm{stab}}$ and $C_{\mathrm{app}}$, independent of $M$, such that
\begin{align}
\|\mathcal{I}_M\|_{H^{\tau_\ell}(\Omega)\to H^{\tau_\ell}(\Omega)}
&\leq C_{\mathrm{stab}},
\label{eq:uniform_frame_stability}
\\
\inf_{w\in W_M}
\|u-w\|_{H^{\tau_\ell}(\Omega)}
&\leq
C_{\mathrm{app}}M^{-\tau_\ell/d}
\|u\|_{\mathcal{B}_{\ell,0}(\Omega)}
\label{eq:frame_jackson_estimate}
\end{align}
for every $u\in\mathcal{B}_{\ell,0}(\Omega)$. Then, for all sufficiently large $M$,
\begin{equation}
\|u-\mathcal{I}_Mu\|_{L^2(\Omega)}
\leq
C M^{-2\tau_\ell/d}
\|u\|_{\mathcal{B}_{\ell,0}(\Omega)},
\label{eq:multiscale_doubled_order}
\end{equation}
where $C$ is independent of $M$ and $u$.
\end{theorem}

\begin{proof}
For any $w\in W_M$, the projector property gives
\[
u-\mathcal{I}_Mu
=
(I-\mathcal{I}_M)(u-w).
\]
Therefore, by~\eqref{eq:uniform_frame_stability} and~\eqref{eq:frame_jackson_estimate},
\[
\|u-\mathcal{I}_Mu\|_{H^{\tau_\ell}(\Omega)}
\leq
(1+C_{\mathrm{stab}})
C_{\mathrm{app}}M^{-\tau_\ell/d}
\|u\|_{\mathcal{B}_{\ell,0}(\Omega)}.
\]
The error vanishes on $X_M$. For all sufficiently large $M$, the scattered-zeros inequality with $\mu=0$ and~\eqref{eq:fill_distance_cardinality} give
\[
\|u-\mathcal{I}_Mu\|_{L^2(\Omega)}
\leq
C h_M^{\tau_\ell}
\|u-\mathcal{I}_Mu\|_{H^{\tau_\ell}(\Omega)}
\leq
C M^{-\tau_\ell/d}
\|u-\mathcal{I}_Mu\|_{H^{\tau_\ell}(\Omega)}.
\]
Combining the two estimates proves~\eqref{eq:multiscale_doubled_order}.
\end{proof}

Here $\mathcal{B}_{\ell,0}(\Omega)$ denotes a localized doubled source class for which a global doubling estimate is meaningful; on bounded domains, additional smoothness alone need not supply the localization or boundary compatibility required by classical superconvergence theory~\cite{Schaback1999,SloanKaarnioja2025}. For the convolution kernel
\[
\kappa_{d,\ell}:=\phi_{d,\ell}*\phi_{d,\ell},
\]
the corresponding whole-space native norm is equivalent to the $H^{2\tau_\ell}$ norm because
\[
\widehat{\kappa}_{d,\ell}
=|\widehat{\phi}_{d,\ell}|^2,
\qquad
\widehat{\phi}_{d,\ell}(\omega)
\asymp
(1+\|\omega\|_2^2)^{-\tau_\ell}
\]
\cite{Wendland2005,SchabackWendland2006}.

Theorem~\ref{thm:conditional_multiscale_doubling} reduces the observed doubling phenomenon to two concrete properties of the multiscale range $W_M$: uniform $H^{\tau_\ell}$ stability of the QR projector and an $H^{\tau_\ell}$ Jackson estimate on the localized doubled source class. We state these properties as the central conjecture.

\begin{conjecture}[Uniform stability and doubled-space approximation]
\label{conj:multiscale_doubled_order}
Let $X_M\subset\Omega$ be sets of $M$ distinct sampling sites satisfying~\eqref{eq:fill_distance_cardinality}, with no separation-radius assumption, and let
\[
C_{0,M}=X_M
\supset C_{1,M}
\supset\cdots\supset C_{J(M),M}
\]
be the nested hierarchy constructed as in Section~\ref{sec:nested-center-selection}. Let $s_{0,M}$ denote the representative finest-level spacing used in~\eqref{eq:point-cloud-spacing}, assume $s_{0,M}\asymp M^{-1/d}$, and set
\[
s_{j,M}=2^js_{0,M},
\qquad
\rho_{j,M}=\eta\,s_{j,M}
\]
for fixed $\eta>0$. For the multiscale QR interpolant built from the $C^{2\ell}$ Wendland kernel, the uniform stability estimate~\eqref{eq:uniform_frame_stability} and approximation estimate~\eqref{eq:frame_jackson_estimate} hold on $\mathcal{B}_{\ell,0}(\Omega)$.
\end{conjecture}

\paragraph{Example}
\label{ex:conjectured_frame_rates}
If Conjecture~\ref{conj:multiscale_doubled_order} holds, Theorem~\ref{thm:conditional_multiscale_doubling} gives 
\[ \|u-\mathcal{I}_Mu\|_{L^2(\Omega)} = \mathcal{O}(M^{-2\tau_\ell/d}). \] 
For $d=2$, this predicts \begin{equation} C^2:\ \mathcal{O}((M^{1/2})^{-5}), \qquad C^4:\ \mathcal{O}((M^{1/2})^{-7}), \qquad C^6:\ \mathcal{O}((M^{1/2})^{-9}). \label{eq:conjectured_2d_rates} \end{equation} The observed discrete relative $\ell_2$ orders are numerically close to these predictions.

\section{Numerical Results}
\label{sec:results}
We now present numerical results for both the multiscale frame approximant (Section~\ref{sec:frame_results}) and the FKM (Section~\ref{sec:ol_results}). First, in Section~\ref{sec:frame_results}, we present traditional numerical convergence studies of the multiscale frame approximation for target functions of different smoothness, with comparisons to our theorems and the doubling conjecture. We also study the influence of design parameters such as the kernel support size and the density of the design matrix. Then, in Section~\ref{sec:ol_results}, we present results on several standard operator learning benchmarks, with comparisons against neural operators and the VKM. Finally, in Section~\ref{sec:mult-analysis}, we present a qualitative analysis of the \emph{a posteriori} multiscale decompositions obtained from a subset of the operator learning benchmark problems.

Unless otherwise stated, all errors are reported using the relative $\ell_2$ error $e_{\ell_2} = \frac{\|y-\hat y\|_2}{\|y\|_2}$, where $y$ denotes the ground truth solution vector and $\hat y$ the prediction. 

\subsection{Frame Approximation Results}
\label{sec:frame_results}

\begin{figure}[htbp]
    \centering
    \includegraphics[width=0.8\textwidth]{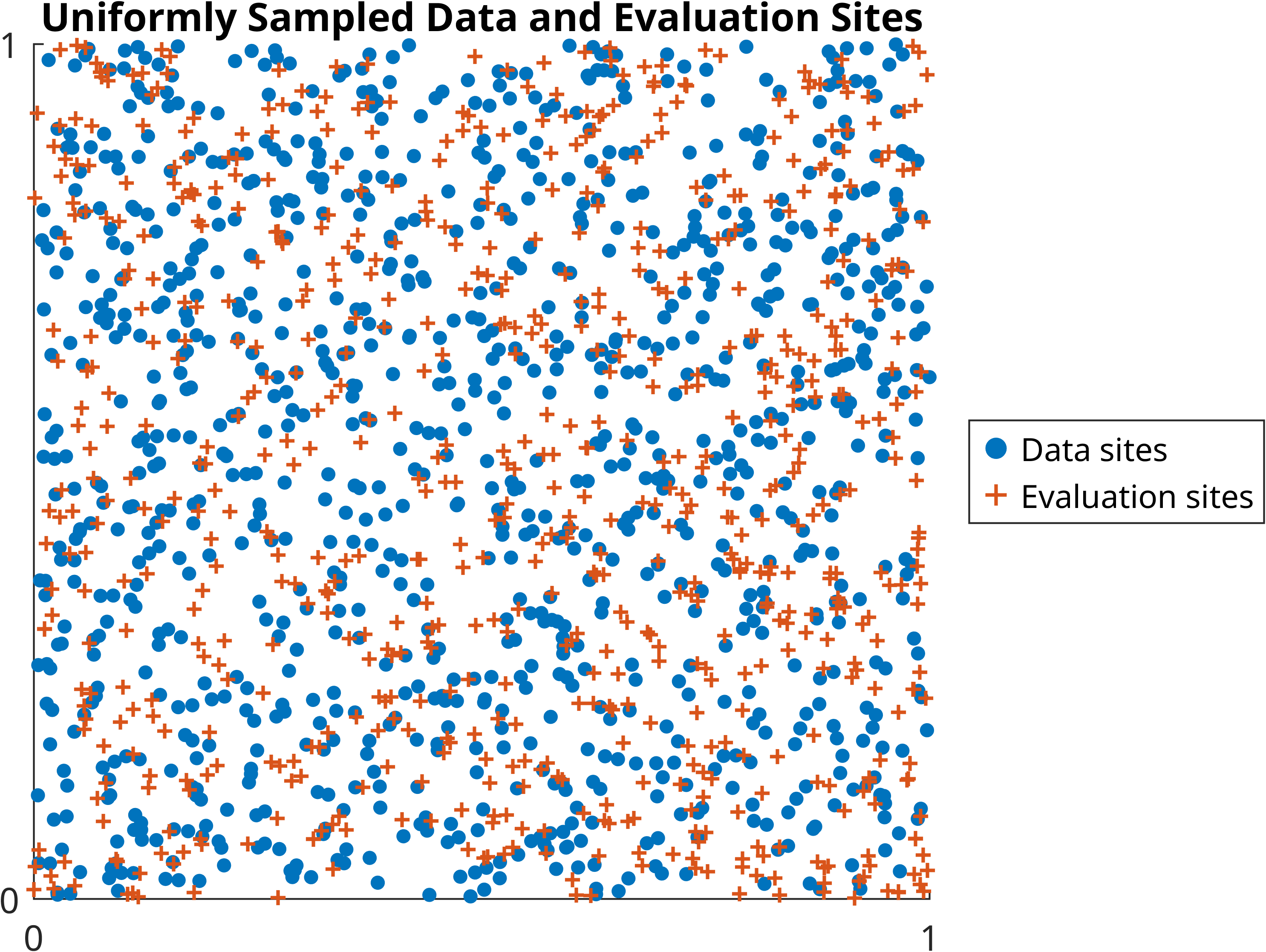}
    \caption{Example of a randomly sampled point cloud of 1000 data sites for decomposition and 700 evaluation sites.}
    \label{fig:point_cloud_example}
\end{figure}
We first evaluate the ability of the proposed frame representation to accurately reconstruct functions from their multiscale coefficients. Here, we restrict ourselves to synthetic test functions and study the influence of key design parameters within the frame. To assess reconstruction fidelity across different regularity classes, we consider two families of test functions with varying smoothness: analytic (\(C^\omega(\mathbb{R}^d)\)) and finitely smooth (\(C^2(\mathbb{R}^d)\)). The corresponding functions in both two and three dimensions are shown in Figure~\ref{fig:synthetic_functions} of the appendix. For the discussion that follows, let $\Omega=[0,1]^d$ with \(d\in\{2,3\}\), and let \(x\in\Omega\). As shown in Figure~\ref{fig:point_cloud_example}, the decomposition and evaluation sites are sampled independently and uniformly from $\Omega$; no minimum-separation constraint is imposed. The target functions are:

\paragraph{Analytic ($C^\omega$)}
\begin{align}
f(x)=\exp\!\left(\frac{-\left(\sum_{i=1}^{d}(x_i-0.5)^2\right)}{c_d}\right),
\qquad
c_d=
\begin{cases}
0.2, & d=2,\\
0.8, & d=3.
\end{cases}
\end{align}
\paragraph{Finitely smooth ($C^2$)}
\begin{align}
\ f(x)=\left(\sum_{i=1}^{d}(x_i-0.5)^2\right)^{3/2}.
\end{align}
Together, these functions provide a controlled setting for comparing kernel choices and evaluating the effect of the decomposition parameters on reconstruction accuracy. 

\subsubsection{Convergence as a Function of the Number of Data Sites $M$}
We use a design-matrix density of $\delta=0.2$ as the reference density in the remaining approximation experiments; in the appendix, section ~\ref{supplement:density_sweep} reports a detailed density sweep. We then investigate how reconstruction accuracy changes as the number of decomposition sites $M$ increases under two protocols:
\begin{enumerate}
    \item For each $M$, choose the finest-level support radius $\rho_0$ so that the resulting design-matrix density is approximately $\delta=0.2$.
    \item Choose $\rho_0$ at the largest value of $M$ so that the design-matrix density is approximately $0.2$, and then hold this support radius fixed for all smaller values of $M$.
\end{enumerate}
All experiments use three decomposition scales ($J=2$) for simplicity, although an arbitrary number of scales can be used in practice. Farthest-point thinning is performed with target cardinalities $m_j\approx M/2^{jd}$, producing scale-indexed fine, intermediate, and coarse frame levels.

\paragraph{Fixed density}
\begin{figure}[htbp]
    \centering
    \includegraphics[width=0.8\textwidth]{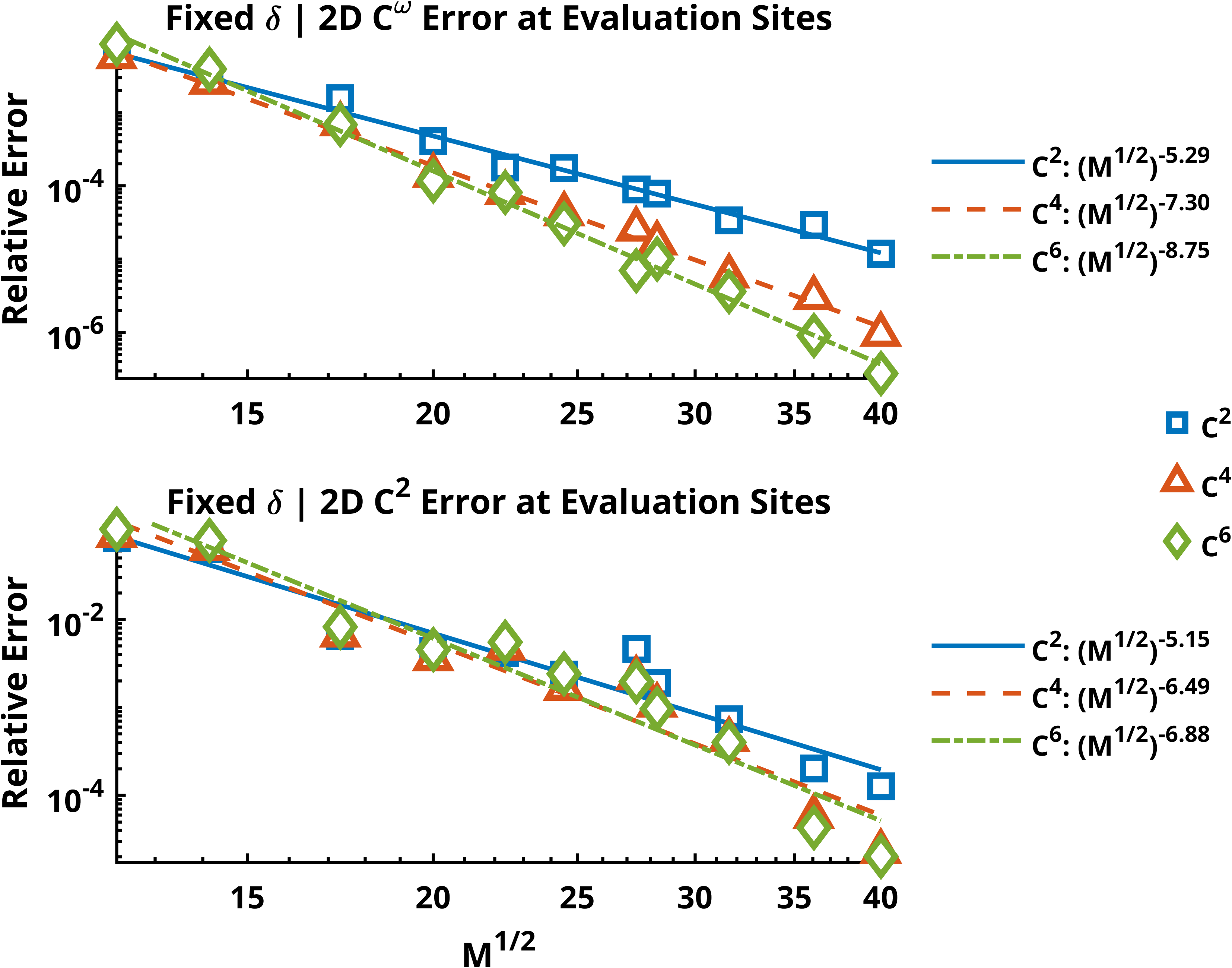}
    \caption{$\ell_2$ error at evaluation sites as a function of the number of data sites $M$ while holding the design matrix density at approximately $0.2$ (2D).}
    \label{fig:PC_N_Error_pin_density_off_node_2d}
\end{figure}
\begin{figure}[htbp]
    \centering
    \includegraphics[width=0.8\textwidth]{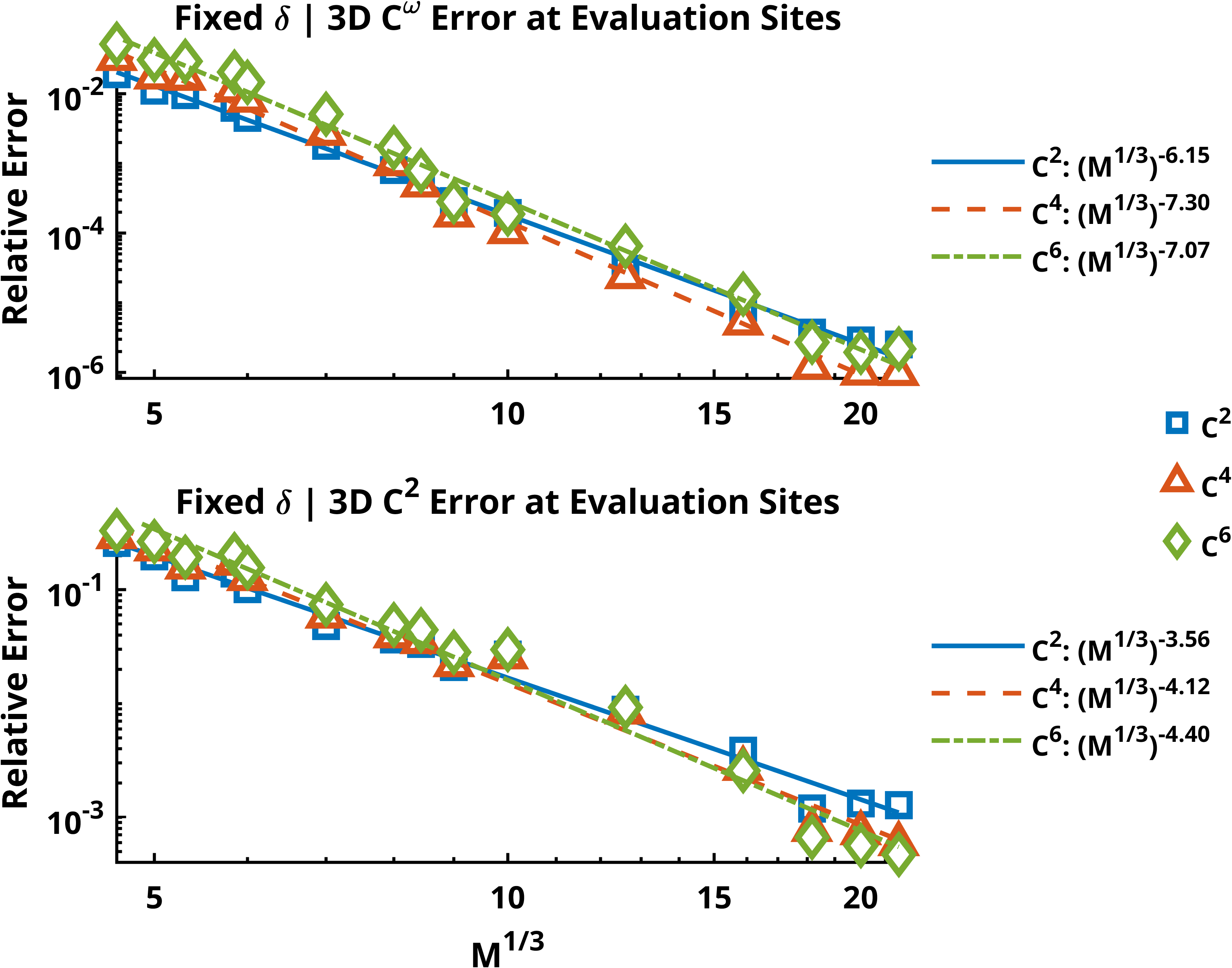}
    \caption{$\ell_2$ error at evaluation sites as a function of the number of decomposition points while holding the design matrix density at approximately $0.2$ (3D).}
    \label{fig:PC_N_Error_pin_density_off_node_3d}
\end{figure}
We first consider the fixed-density case. Let $\delta$ denote the fraction of nonzero entries in the design matrix. Holding $\delta$ approximately constant fixes the fraction of active kernel evaluations and therefore provides a convenient way to compare sparsity across resolutions. In this protocol, the support radius $\rho_0$ is adjusted separately for each $M$ to attain the target density; we make no assumption here about an asymptotic scaling law for $\rho_0$ with $M$.

%%%%%%%%%%%%%%%%%%%%%%%%%%%%%%%%% Hold density steady%%%%%%%%%%%%%%%%%%%%%%%%%%%%%%%

With $\delta$ fixed at approximately $0.2$, Figures~\ref{fig:PC_N_Error_pin_density_off_node_2d} and~\ref{fig:PC_N_Error_pin_density_off_node_3d} show that the evaluation-site error decreases rapidly as $M$ increases. For the analytic targets, the observed slopes are numerically close to the doubled rates conjectured in Section~\ref{sec:superconvergence_conjecture}. The $C^2$ targets exhibit faster than expected rates over the tested range, which we interpret as preasymptotic behavior. Note the different $y$-axis scales for the analytic and finitely smooth targets in both 2D and 3D. The $C^6$ kernel in 3D shows a leveling off under the fixed-density protocol. Because $\rho_0$ is retuned as $M$ changes in this experiment, we examine fixed support separately below.
%%%%%%%%%%%%%%%%%%%%%%%%% Hold rho steady after solved on fine scale%%%%%%%%%%%%%%%%%%
\paragraph{Fixed support}
\begin{figure}[htbp]
    \centering
    \includegraphics[width=0.8\textwidth]{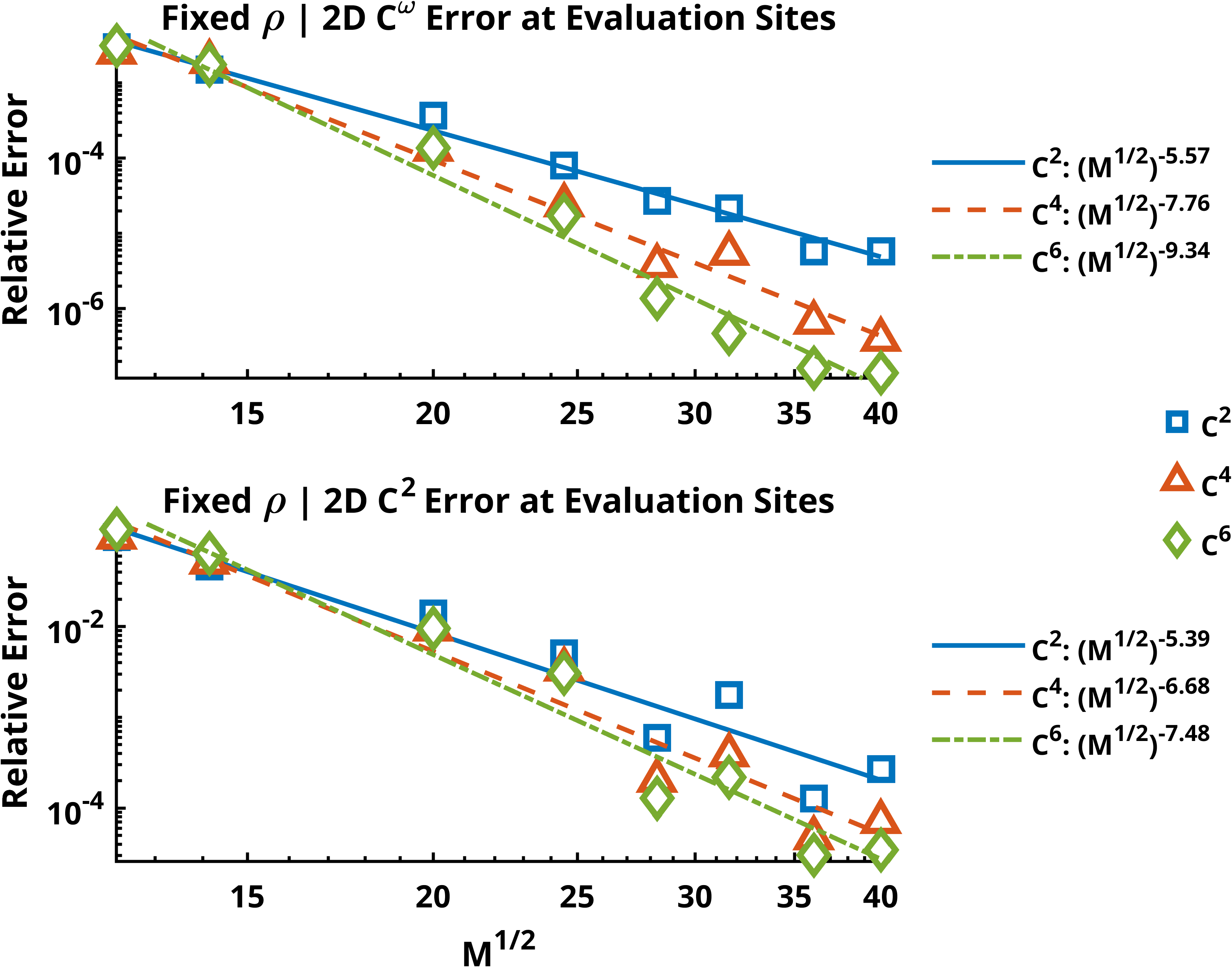}
    \caption{$\ell_2$ error as a function of the number of points in the domain. The kernel support radius ($\rho$) is chosen on the finest scale to produce a design matrix density of approximately $0.2$, and this value is then held fixed as the number of points changes.}
    \label{fig:point_cloud_off_node_finest_scale_fixed_support_2d}
\end{figure}

\begin{figure}[htbp]
    \centering
    \includegraphics[width=0.8\textwidth]{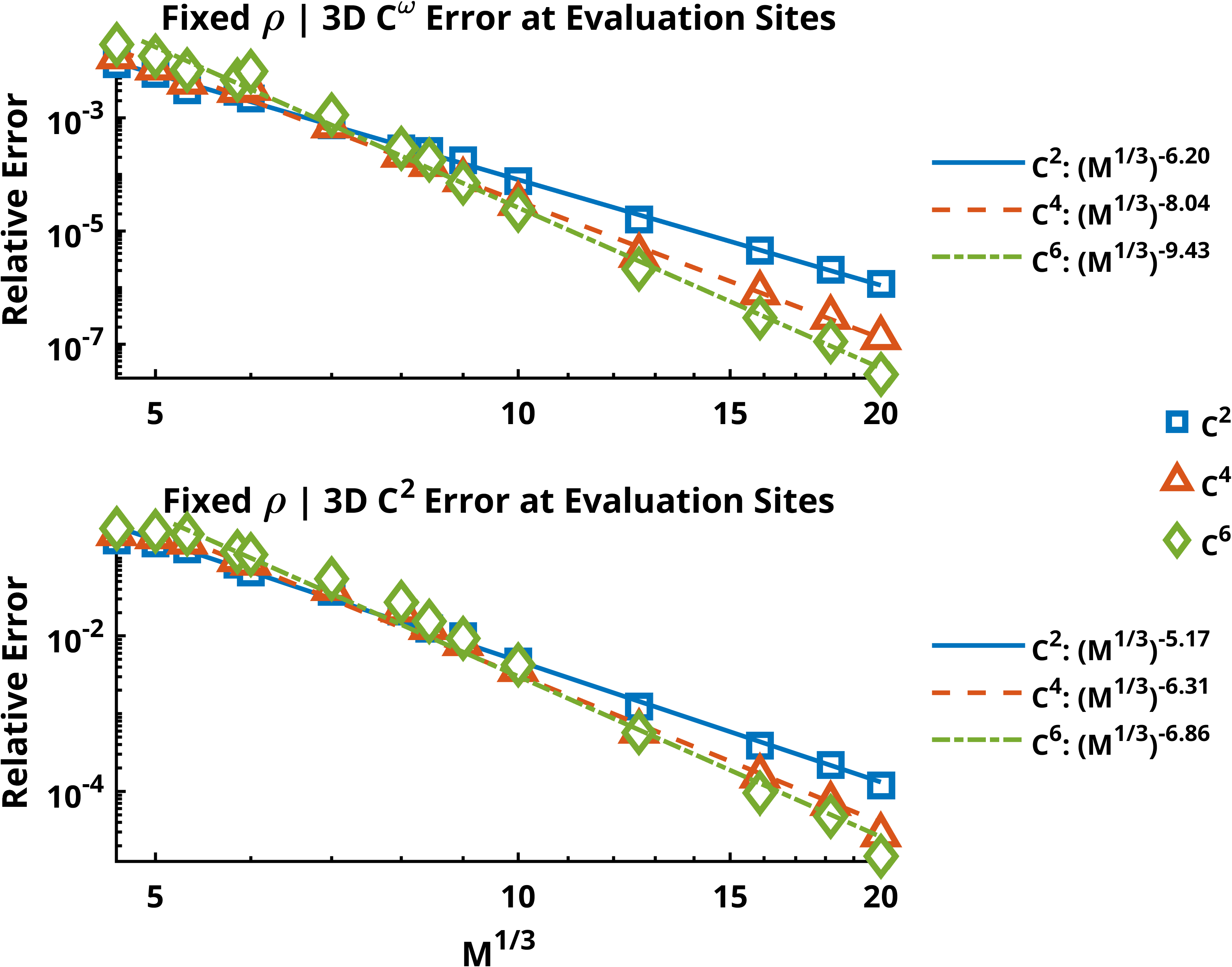}
    \caption{$\ell_2$ error as a function of the number of points in the domain. The kernel support radius ($\rho$) is chosen on the finest scale to produce a design matrix density of approximately $0.2$, and this value is then held fixed as the number of points changes.}
    \label{fig:point_cloud_off_node_finest_scale_fixed_support_3d}
\end{figure}
Next, we examine convergence with a fixed finest-level support radius $\rho_0$. We first choose $\rho_0$ at the largest value of $M$ so that the design matrix has density approximately $0.2$, and then use this same value of $\rho_0$ for all remaining resolutions; consequently, $\delta$ is allowed to vary with $M$. The evaluation-site errors in Figures~\ref{fig:point_cloud_off_node_finest_scale_fixed_support_2d} and~\ref{fig:point_cloud_off_node_finest_scale_fixed_support_3d} again decrease with $M$, rapidly for the analytic target and more slowly for the finitely smooth target. Under this fixed-support protocol, the observed slopes for the analytic targets are numerically close to the rates predicted by the doubling conjecture in Section~\ref{sec:superconvergence_conjecture}.

\FloatBarrier
\subsection{Operator Learning Results}
\label{sec:ol_results}
We also evaluated the FKM on a set of operator learning benchmark problems spanning both standard PDEs and problems with inherent multiscale structure. The standard PDE benchmarks allow comparison against existing operator learning methods, while the multiscale benchmarks test the effect of an explicitly scale-indexed representation. We used a $C^4$ Matérn kernel for the operator interpolant and a $C^4$ Wendland kernel for the frame.
% VERIFY BEFORE SUBMISSION: specify the Matérn smoothness parameter \nu or the exact kernel formula.
%We also experimented with Fourier, wavelet, and hybrid coefficient representations, but these consistently produced higher prediction errors and are therefore omitted from the final comparison.

In selecting a design-matrix density for the operator-learning problems, we found that an initial value of $\eta\approx2$ in~\eqref{eq:rho-schedule} provided a useful balance between computational efficiency and reconstruction accuracy on unit domains. We then adjusted the density for each benchmark; the values used are reported in Tables~\ref{tab:base_results} and~\ref{tab:multiscale_results}. The resulting levelwise frame components are illustrated in Figures~\ref{fig:cylinder_flow_shedding_scales}--\ref{fig:species_transport_z_scales}. Unlike the approximation convergence studies, these experiments report generalization errors at fixed spatial resolutions. We compare against our own implementations of Geo-FNO, Transolver, and VKM.

\subsubsection{Base PDEs}
To compare our approach to the state-of-the-art methods, we first present results for the following problems that use a mix of gridded data, triangular meshes, and point clouds.

\textbf{2D Lid-Driven Cavity Flow:} Maps an initial velocity field that is zero in the interior, with a rightward-moving lid, to the velocity field at $T=5$ s, i.e., $\mathcal{G}:u(x,0)\to u(x,T)$, for a square cavity with a randomly initialized lid velocity~\cite{sharma2026fluids} (triangular mesh).

\textbf{2D Darcy Flow (PWC):} Maps a piecewise constant (PWC) permeability field $c: [0, 1]^2 \to \mathbb{R}$ to the pressure field \(u : [0,1]^2 \to \mathbb{R}\) via the Darcy flow equation~\cite{lu2022neuraloperators} (gridded data).

\textbf{2D Darcy Flow on a Triangle:} Maps a boundary-condition field to the pressure field $h(x,y)$ on a triangular domain through Darcy's equation~\cite{lu2022neuraloperators}. The permeability is set to $0.1$, the forcing is set to $-1.0$, and the boundary-condition field is sampled from a Gaussian random field (triangular mesh).

\textbf{3D Reaction-Variable-Coefficient-Diffusion on a Point Cloud:} Maps a spatially constant initial chemical concentration $c(y,0):\mathbb{R}^3\to\mathbb{R}$, whose scalar value is sampled uniformly at random for each realization, to the concentration $c(y,0.5)$ inside the unit ball, sampled on a point cloud, as described in~\cite{sharma2025ensemble}. Due to a discontinuity in the reaction term, the final solution exhibits a sharp gradient across the plane $y_1=0$.
\begin{table}[h]
\centering
\caption{Comparison of operator learning methods on standard PDE benchmarks. Errors are reported as relative $\ell_2$ errors. Here, ``Density'' refers to the density of the design matrix, $\kappa$ denotes the target condition number of the operator-learning Gram matrix $K$, $\lambda$ is the regularization parameter, $N_T$ is the number of training functions, and $M$ is the number of spatial locations.}
% VERIFY BEFORE SUBMISSION: confirm that the reported target is cond(K), rather than cond(K+\lambda I).
\resizebox{\textwidth}{!}{%
\begin{tabular}{c c c c c c c c c c}
\hline
Problem & Density & $\kappa$ & $\lambda$ & $N_T$ & $M$ & Geo-FNO & Transolver & VKM & FKM \\
\hline
Cavity Flow      & .01 & 1e15 & 1e-7 & 10000 & 9566 & 2.19e-4 & 1.89e-4 & 1.45e-4 & \textbf{4.06e-7} \\
Darcy (PWC)      & .1 & 1e6 & 1e-4 & 1000 & 841 & \textbf{1.79e-2} & 1.99e-2 & 3.06e-2 & 3.28e-2 \\
Darcy (triangle) & .1 & 1e15 & 1e-12 & 1800 & 861 & 4.3e-4 & - & 3.3e-4 & \textbf{5.05e-6} \\
React.-Diff.     & .02& 1e15 & 1e-14 & 1000 & 4325 & 6.68e-5 & 8.09e-5 & 8.75e-7 & \textbf{3.9e-7} \\
\hline
\end{tabular}
}

\label{tab:base_results}
\end{table}

Table~\ref{tab:base_results} demonstrates that the FKM achieves the lowest generalization error on three of the four benchmark problems, with absolutely no changes to the input or output data. The only exception is the piecewise-constant Darcy problem, where discontinuities violate the smoothness assumptions of the kernel-frame representation. Nevertheless, even on this problem, the FKM remains competitive with the single-scale VKM while offering an \emph{a posteriori} multiscale decomposition.

The largest improvement is observed for the cavity-flow dataset, where FKM reduces the prediction error by more than two orders of magnitude relative to the closest competing method. On the triangular Darcy problem, the reduction is nearly two orders of magnitude. For reaction--diffusion, FKM improves on VKM by a factor of approximately two and on Geo-FNO and Transolver by more than two orders of magnitude. These results indicate that the scale-indexed kernel-frame representation can substantially improve prediction accuracy across several operator-learning problems.

\subsubsection{Multiscale PDEs}
To understand the impacts of multiscale structure within the PDEs themselves, we further evaluated our method on the following benchmarks:

\textbf{Flow Past a Cylinder (Laminar):} Predict the incompressible velocity field $u(x,y,t=10\,\mathrm{s})$ from the initial uniform field $u(x,y,0)$. We impose a no-slip condition on the cylinder, zero pressure at the outlet, and prescribed velocity on the top, bottom, and inlet boundaries. We choose Reynolds numbers $Re\in[25,64]$ such that the flow regime is laminar~\cite{sharma2026fluids}. The geometry introduces distinct length scales associated with the cylinder radius and the domain length.

\textbf{Flow Past a Cylinder (Vortex Shedding):} Predict the incompressible velocity field $u(x,y,t=10\,\mathrm{s})$ from the initial uniform velocity field $u(x,y,0)$. We impose a no-slip condition on the cylinder, zero pressure at the outlet, and prescribed velocity on the top, bottom, and inlet boundaries. We choose Reynolds numbers $Re\in[112,199]$ such that the flow regime exhibits vortex shedding~\cite{sharma2026fluids}. The resulting vortex train introduces additional spatial structure between the cylinder and domain length scales.

\textbf{3D Compressible Navier--Stokes (NS), Mach 1.0:} Predict the velocity field $v:[0,1]^3\to\mathbb{R}$ from an initial velocity field $v_0:[0,1]^3\to\mathbb{R}$ using the compressible Navier--Stokes equations. A Mach number of $1.0$ introduces strong compressibility effects and additional spatial structure~\cite{takamoto2022pdebench}. The associated loss of smoothness makes this a challenging benchmark.
% VERIFY BEFORE SUBMISSION: if all velocity components are predicted, change the codomain above to \mathbb{R}^3; otherwise call v a velocity component.

\textbf{3D Species Transport:} Learn the operator mapping inlet velocity fields $u_{\text{inlet}}$ to the velocity field $u(y, z=0.5)$ at final time $t = 0.5$ for turbulent mixing of air and methane in a static Kenics mixer with three twisted blades. Air enters for $y<0$ and methane for $y>0$~\cite{moralesubal2024kenics}. We model this as an incompressible flow problem with the $k-\omega$-SST turbulence model; turbulence here creates multiple length scales.

\begin{table}[htbp]
\centering
\caption{Comparison of operator learning methods on multiscale PDE benchmarks. Errors are reported as relative $\ell_2$ errors. Here, ``Density'' refers to the density of the design matrix, $\kappa$ denotes the target condition number of the operator-learning Gram matrix $K$, $\lambda$ is the regularization parameter, $N_T$ is the number of training functions, and $M$ is the number of output points.}
% VERIFY BEFORE SUBMISSION: confirm that the reported target is cond(K), rather than cond(K+\lambda I).
\resizebox{\textwidth}{!}{%
\begin{tabular}{c c c c c c c c c c}
\hline
Problem & Density & $\kappa$ & $\lambda$ & $N_T$ & $M$ & Geo-FNO & Transolver & VKM & FKM \\
\hline
Cylinder Flow (Laminar)         & .01 & 1e15 & 1e-7 & 10000 & 9520 & 9.88e-4 & 8.02e-3 & 5.32e-5 & \textbf{2.82e-7} \\
Cylinder Flow (Vortex Shedding) & .01 & 1e15 & 1e-7 & 10000 & 9520 & 4.62e-5 & 1.19e-6 & 2.34e-6 & \textbf{5.71e-7} \\
NS-Mach 1.0                     & .005 & 1e12 & 1e-6 & 50 & 10000 & 0.581 & \textbf{0.481} & 0.541 & 0.4937 \\
Species Transport               & .005 & 1e11 & 1e-4 & 10000 & 11388 & 5.11e-4 & 2.26e-3 & 2.35e-4 & \textbf{1.10e-4} \\
\hline
\end{tabular}
}

\label{tab:multiscale_results}
\end{table}

Table~\ref{tab:multiscale_results} demonstrates that FKM achieves the lowest prediction error on three of the four multiscale benchmark problems. The only exception is the compressible Navier--Stokes benchmark, where FKM remains competitive with the best-performing method despite the increased difficulty of the problem and the limited training set of only 50 functions.

The largest improvement is observed for the laminar cylinder-flow problem, where FKM reduces the prediction error by more than two orders of magnitude relative to the closest competing method. For vortex shedding, FKM reduces the error by a factor of about two relative to Transolver, the closest competitor. On species transport, FKM achieves approximately half the error of VKM and larger improvements over Geo-FNO and Transolver. These results suggest that organizing the learned representation by kernel-frame level can be advantageous for PDEs with spatial structure across multiple length scales.

\subsubsection{Qualitative Multiscale Analysis}
\label{sec:mult-analysis}
While the improved prediction accuracies demonstrate the effectiveness of FKM, they do not fully capture the representational information provided by the frame. The predicted coefficients retain their frame-level organization, so the contribution from each level can be synthesized separately before the contributions are summed to recover the full prediction. Because the frame is redundant and its levels are generally nonorthogonal, these levelwise components should be interpreted as scale-indexed frame contributions rather than disjoint spectral bands. Figures~\ref{fig:cylinder_flow_shedding_scales}, \ref{fig:reaction_diffusion_scales}, and \ref{fig:species_transport_z_scales} illustrate these components for selected operator-learning benchmarks.
\begin{figure}[htbp]
    \centering
    \begin{subfigure}[b]{0.49\textwidth}
        \centering
        \includegraphics[width=\textwidth]{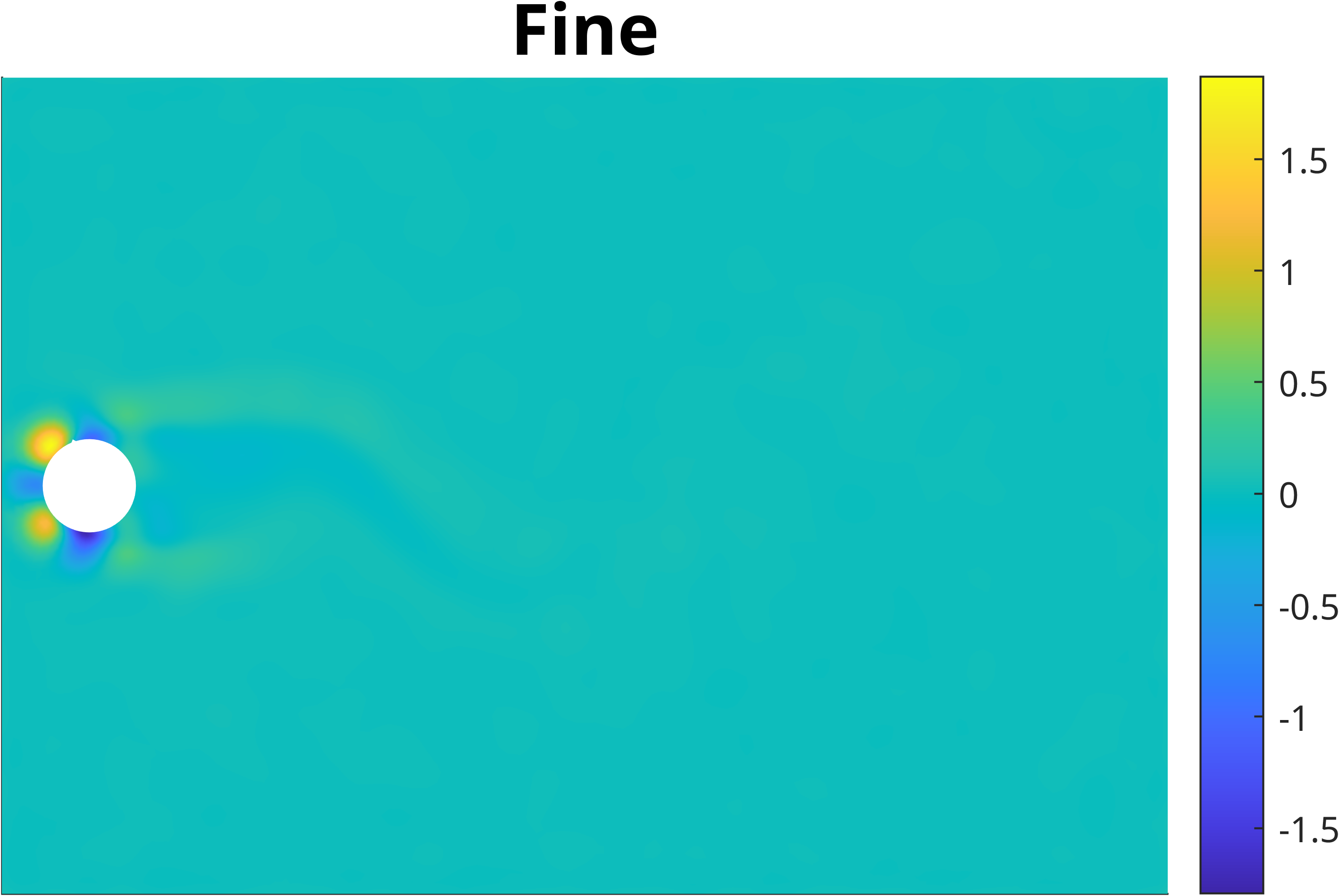}
        \caption{}
        \label{fig:image1}
    \end{subfigure}
    \hfill
    \begin{subfigure}[b]{0.49\textwidth}
        \centering
        \includegraphics[width=\textwidth]{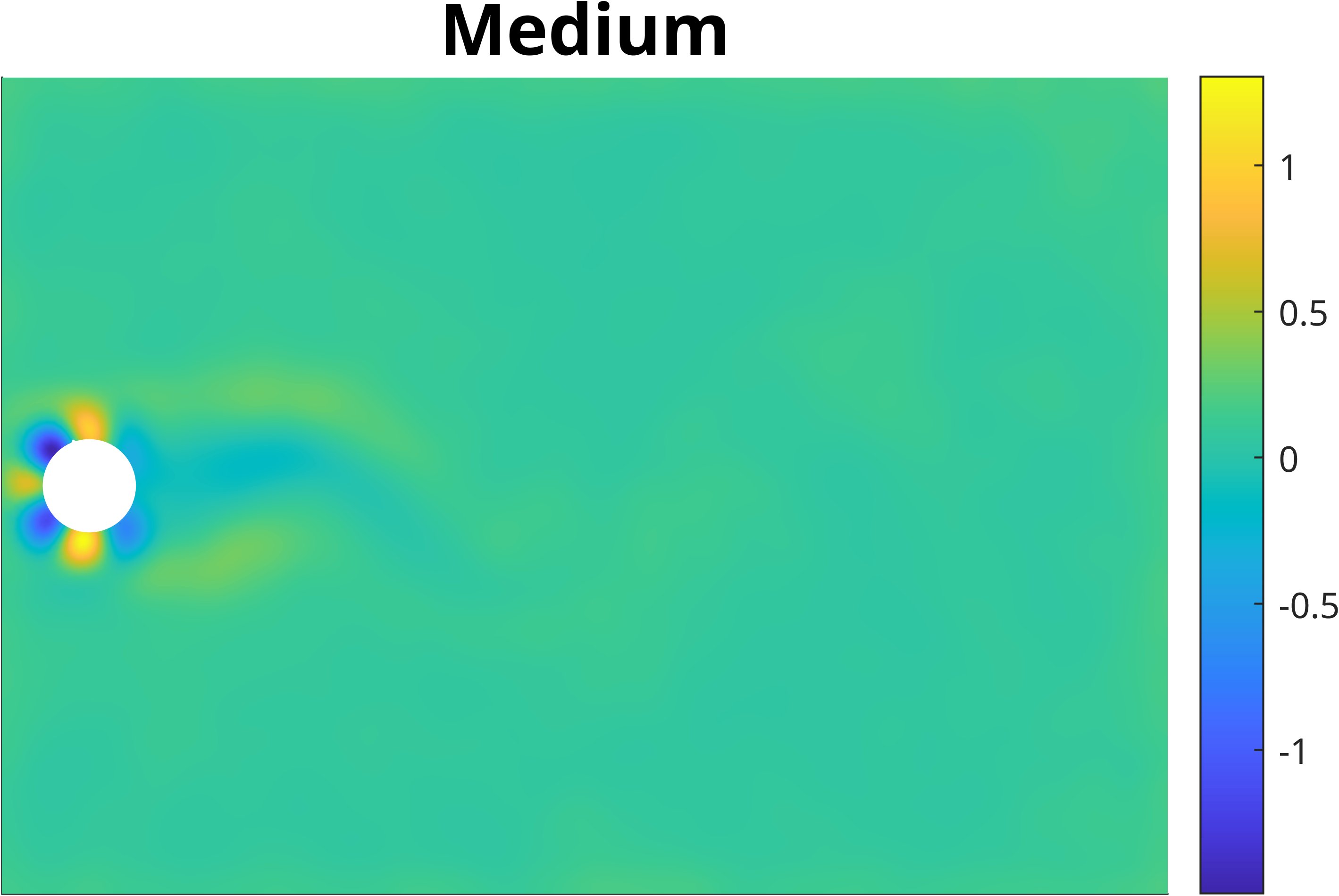}
        \caption{}
        \label{fig:image2}
    \end{subfigure}
    \hfill
    \begin{subfigure}[b]{0.49\textwidth}
        \centering
        \includegraphics[width=\textwidth]{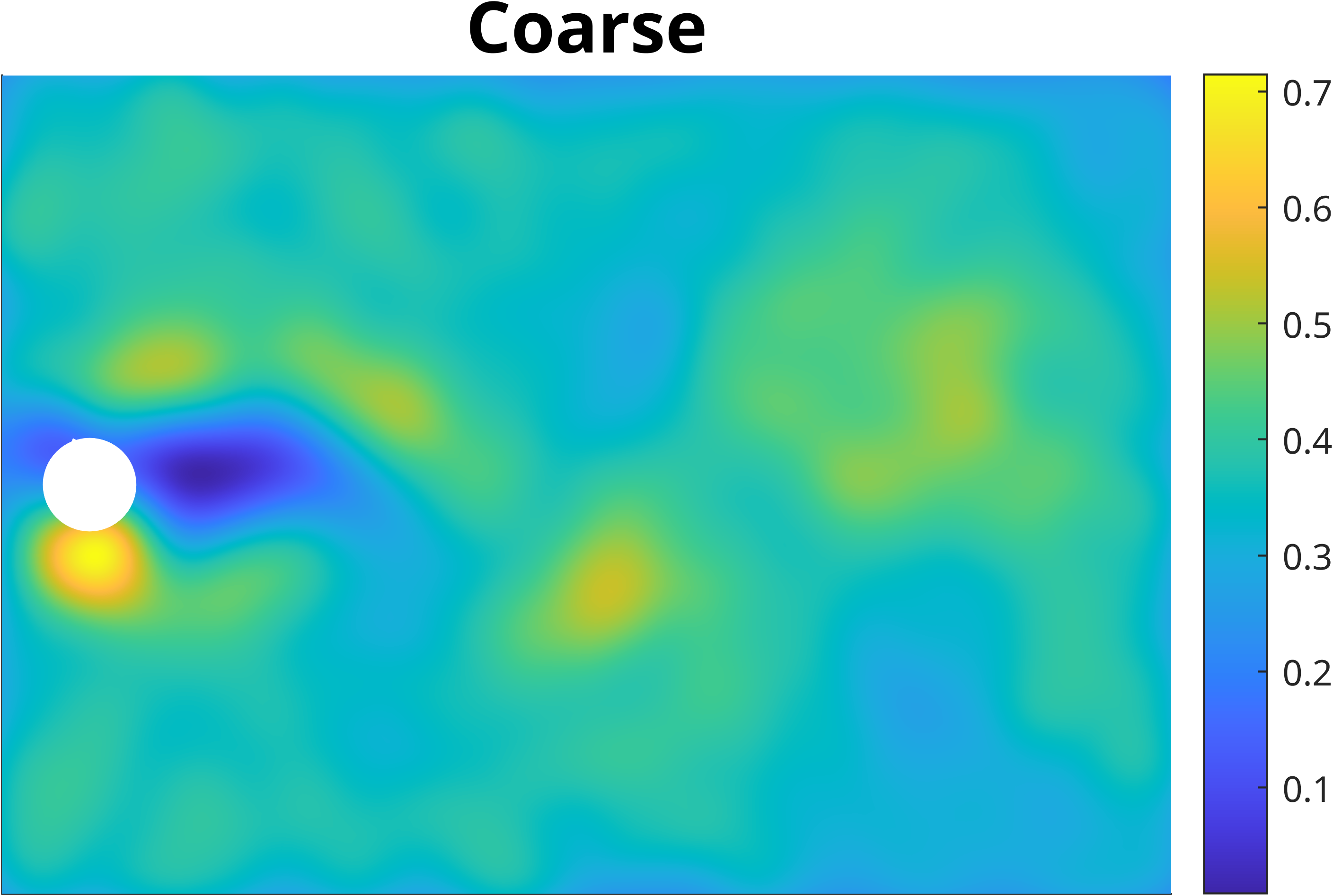}
        \caption{}
        \label{fig:image3}
    \end{subfigure}
    \hfill
    \begin{subfigure}[b]{0.49\textwidth}
        \centering
        \includegraphics[width=\textwidth]{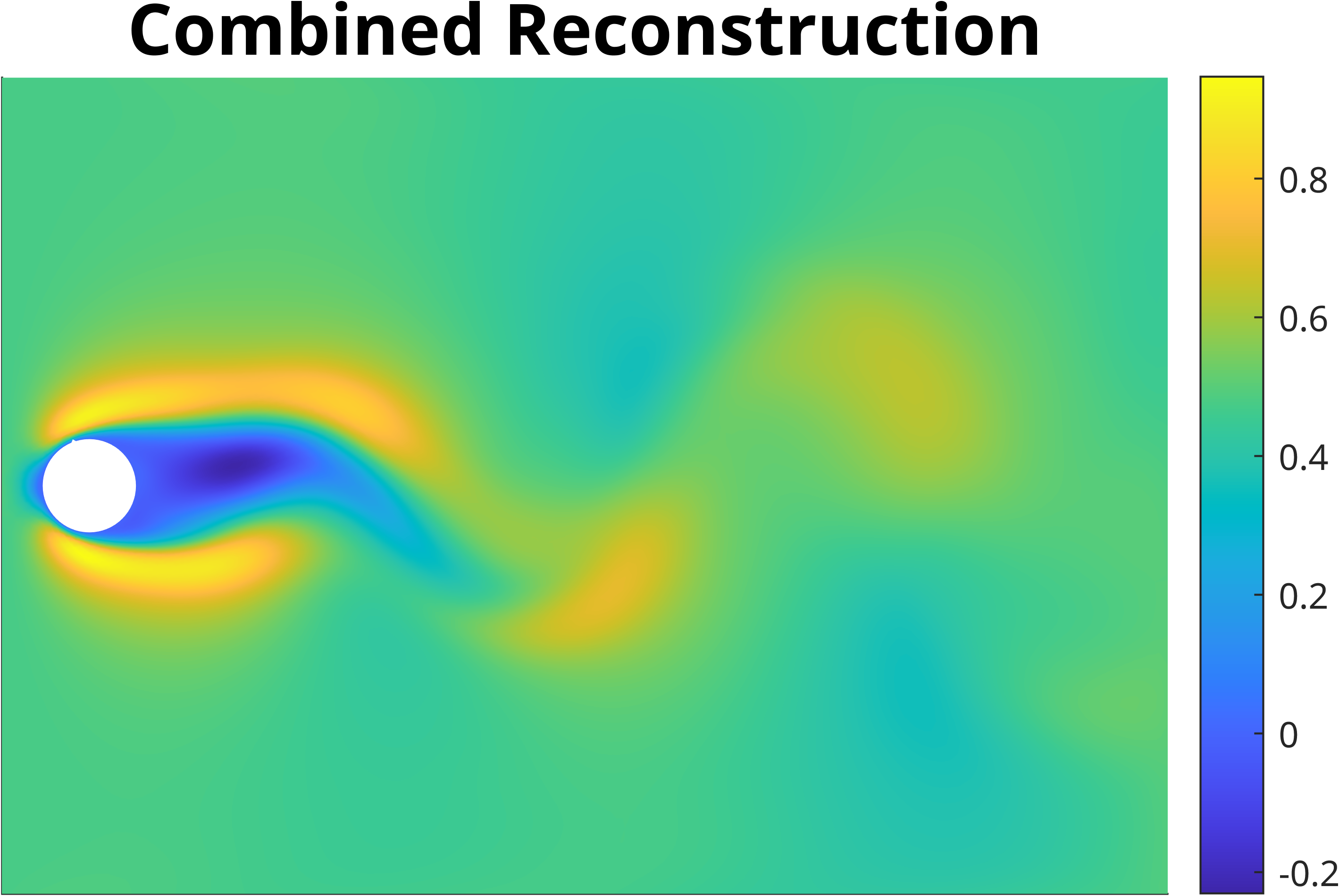}
        \caption{}
        \label{fig:image4}
    \end{subfigure}

    \caption{Scale-indexed frame decomposition of a predicted solution to the cylinder-flow problem with vortex shedding at three levels. The bottom right image shows the combined reconstruction of the problem.}
    \label{fig:cylinder_flow_shedding_scales}
\end{figure}
\paragraph{Flow past a cylinder}
Figure~\ref{fig:cylinder_flow_shedding_scales} presents the levelwise frame contributions for a single component of the velocity field in the vortex-shedding problem. In this example, the coarsest-level contribution contains the primary wake and large downstream vortical structures, while the finer-level contributions contain increasingly localized structure near the cylinder and in the wake. Their sum gives the combined reconstruction shown in the bottom-right panel.

\FloatBarrier
\begin{figure}[htbp]
    \centering

    \begin{subfigure}[b]{0.49\textwidth}
        \centering
        \includegraphics[width=\textwidth]{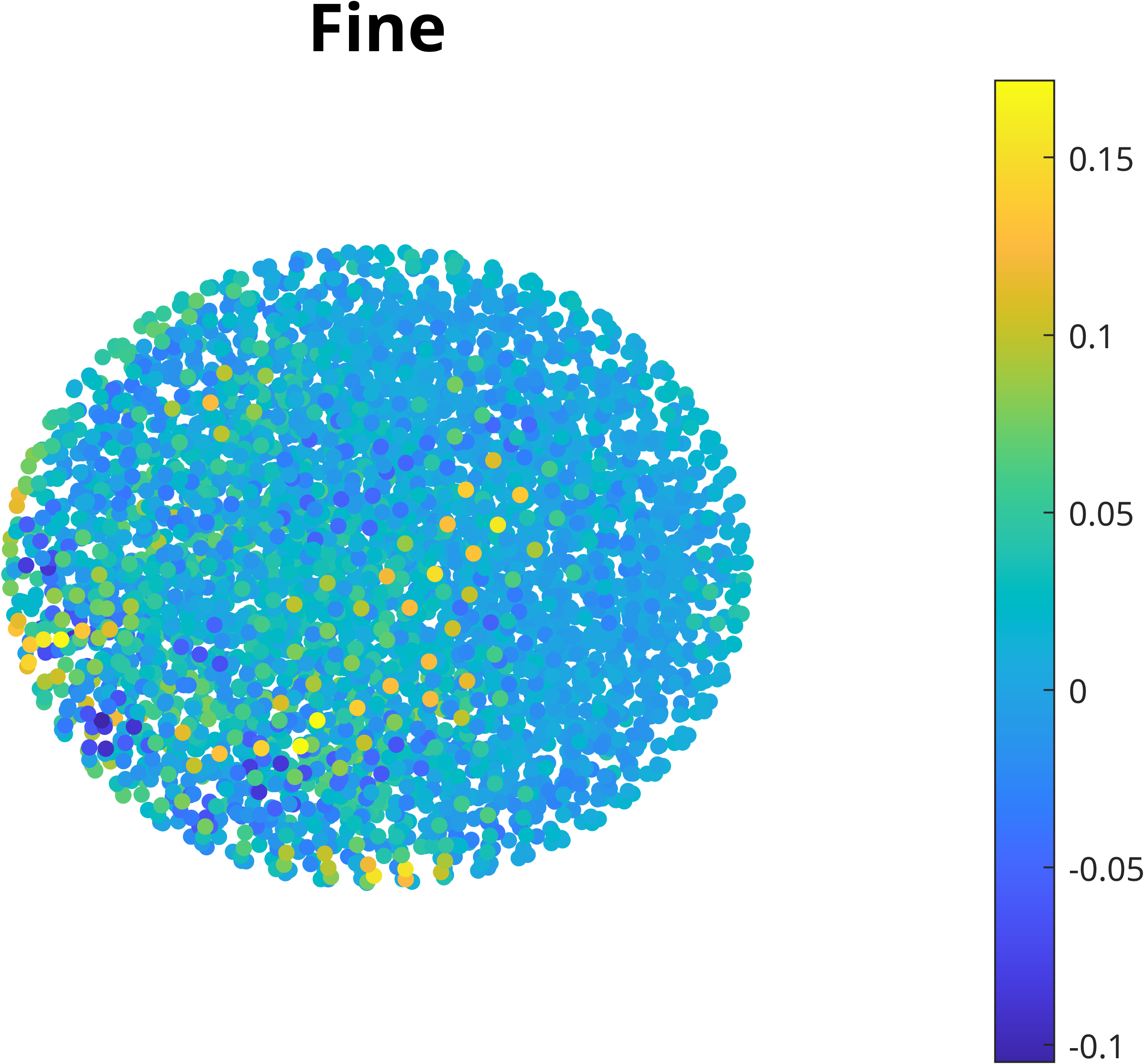}
        \caption{}
    \end{subfigure}
    \hfill
    \begin{subfigure}[b]{0.49\textwidth}
        \centering
        \includegraphics[width=\textwidth]{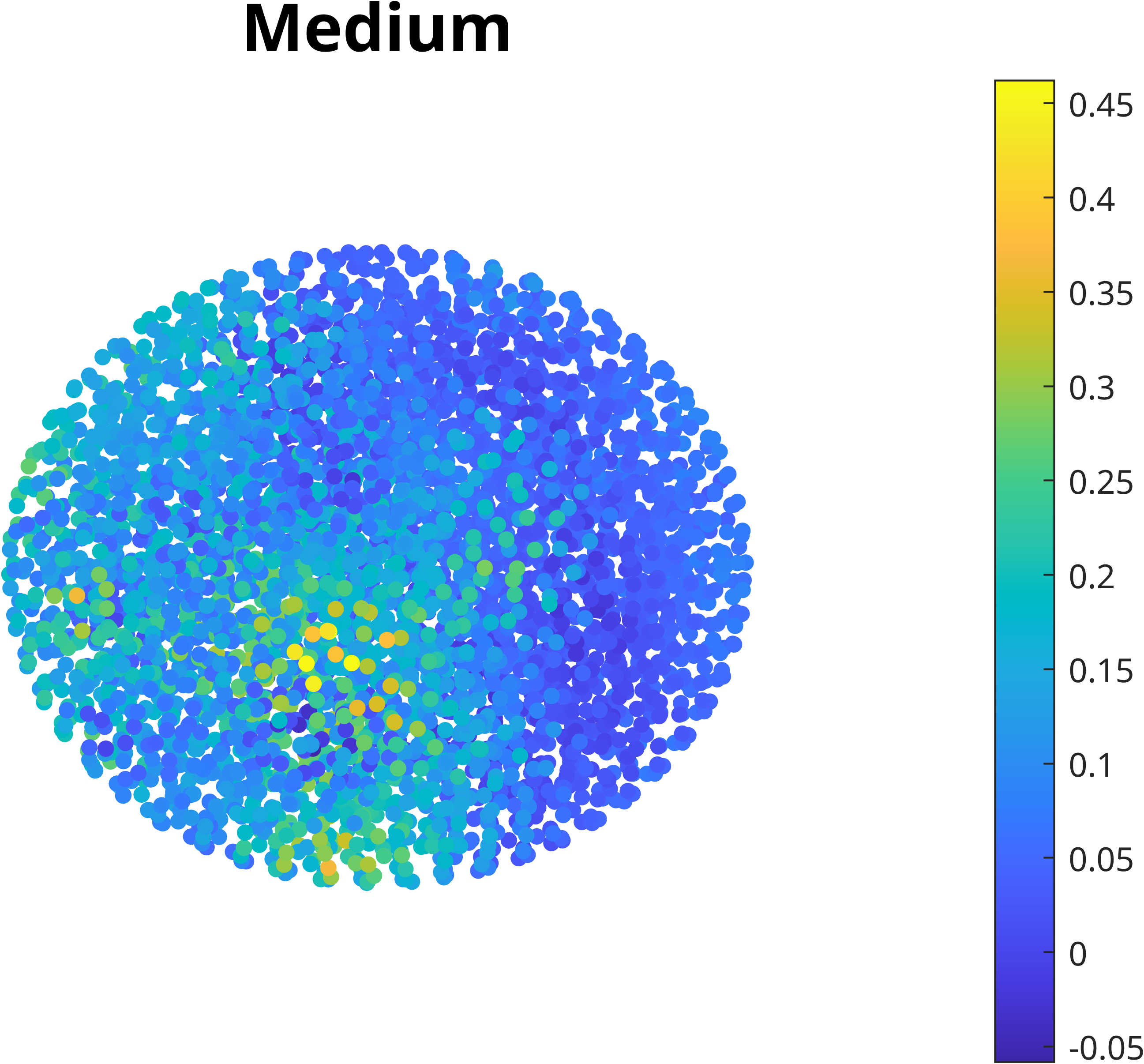}
        \caption{}
    \end{subfigure}
    \hfill
    \begin{subfigure}[b]{0.49\textwidth}
        \centering
        \includegraphics[width=\textwidth]{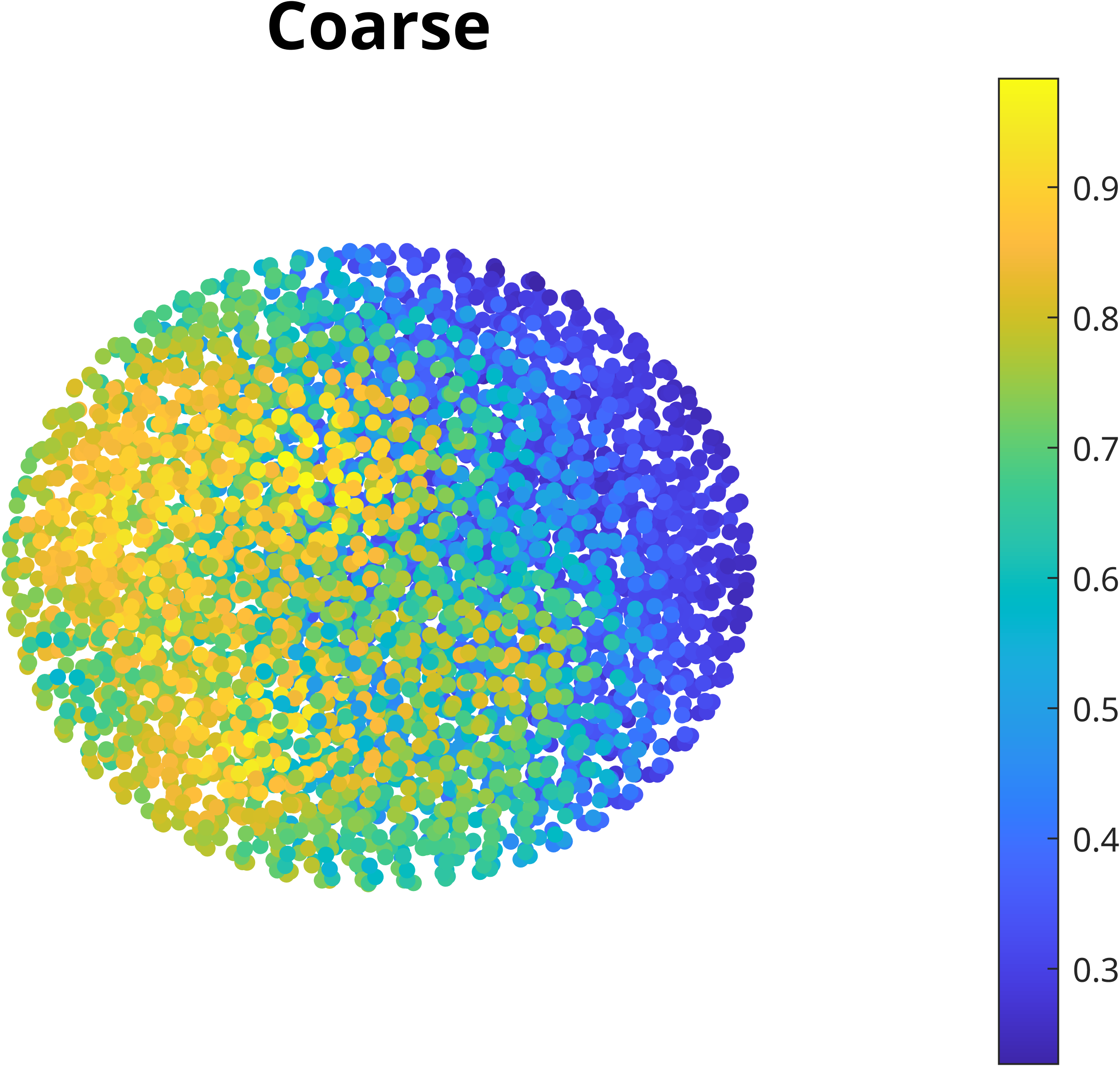}
        \caption{}
    \end{subfigure}
    \hfill
    \begin{subfigure}[b]{0.49\textwidth}
        \centering
        \includegraphics[width=\textwidth]{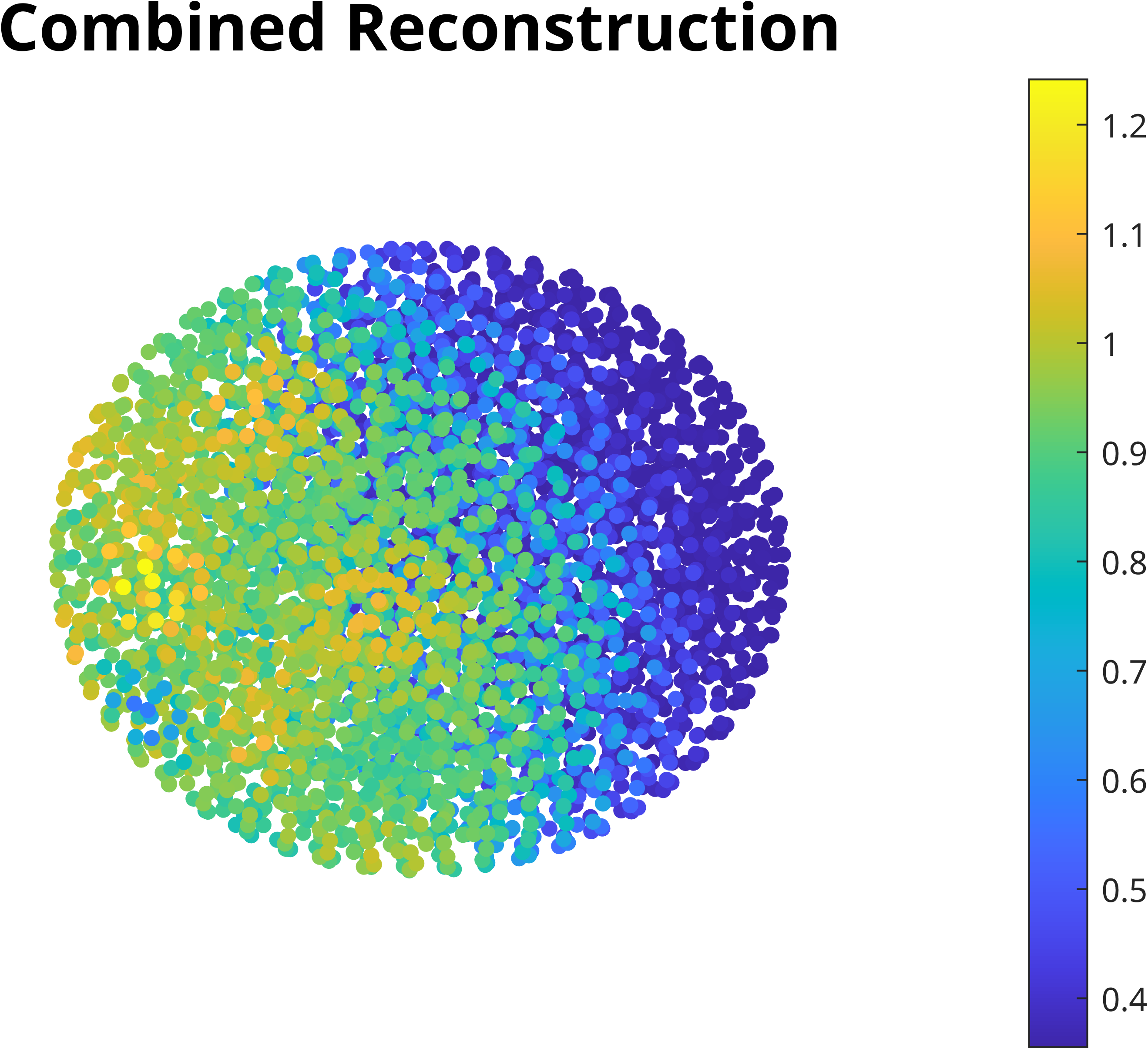}
        \caption{}
    \end{subfigure}

    \caption{Scale-indexed frame decomposition of a predicted solution to the reaction--diffusion problem at three levels. The bottom right image shows the combined reconstruction.}
    \label{fig:reaction_diffusion_scales}
\end{figure}
\paragraph{Reaction--Diffusion}
Figure~\ref{fig:reaction_diffusion_scales} illustrates the levelwise frame contributions to the concentration field in the reaction--diffusion problem, shown by coloring the collocation points directly. The coarsest-level contribution contains the dominant global concentration variation and transition interface, while the intermediate and finest levels contain progressively more localized variations. The components are nonorthogonal and should therefore be read as frame-level contributions rather than a strict frequency decomposition.
\FloatBarrier
\begin{figure}[htbp]
    \centering

    \begin{subfigure}[t]{0.49\textwidth}
        \centering
        \includegraphics[width=\linewidth]{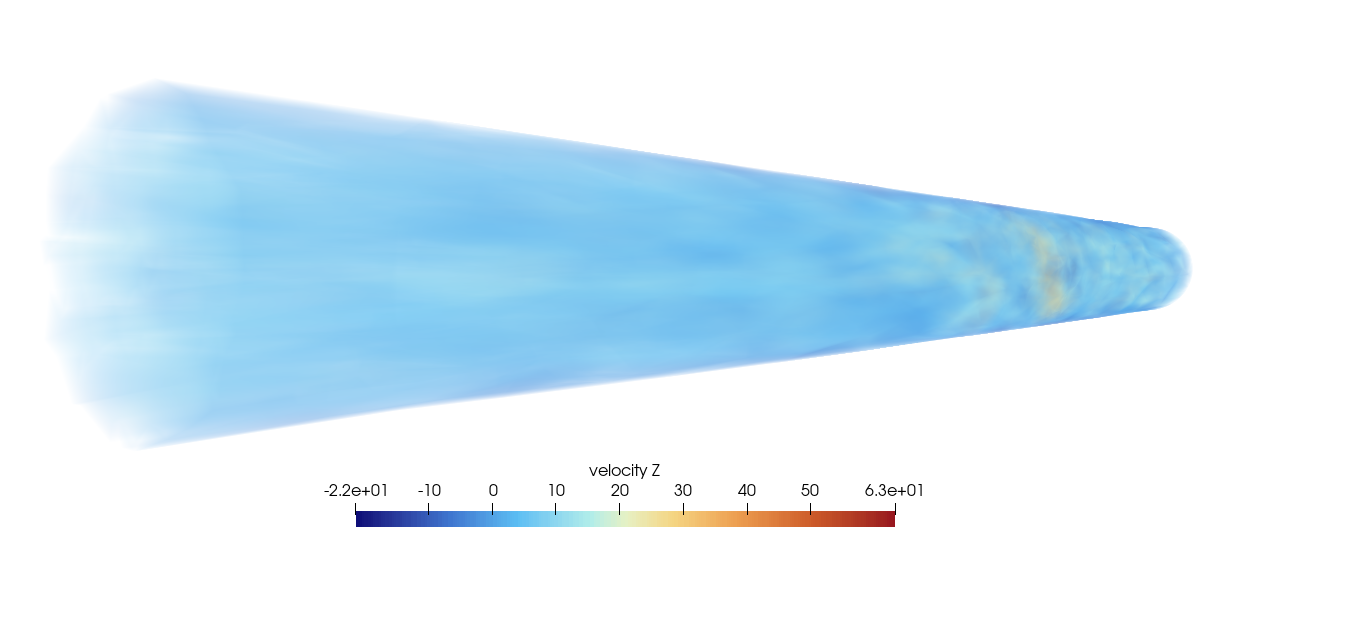}
        \caption{Fine}
    \end{subfigure}
    \hfill
    \begin{subfigure}[t]{0.49\textwidth}
        \centering
        \includegraphics[width=\linewidth]{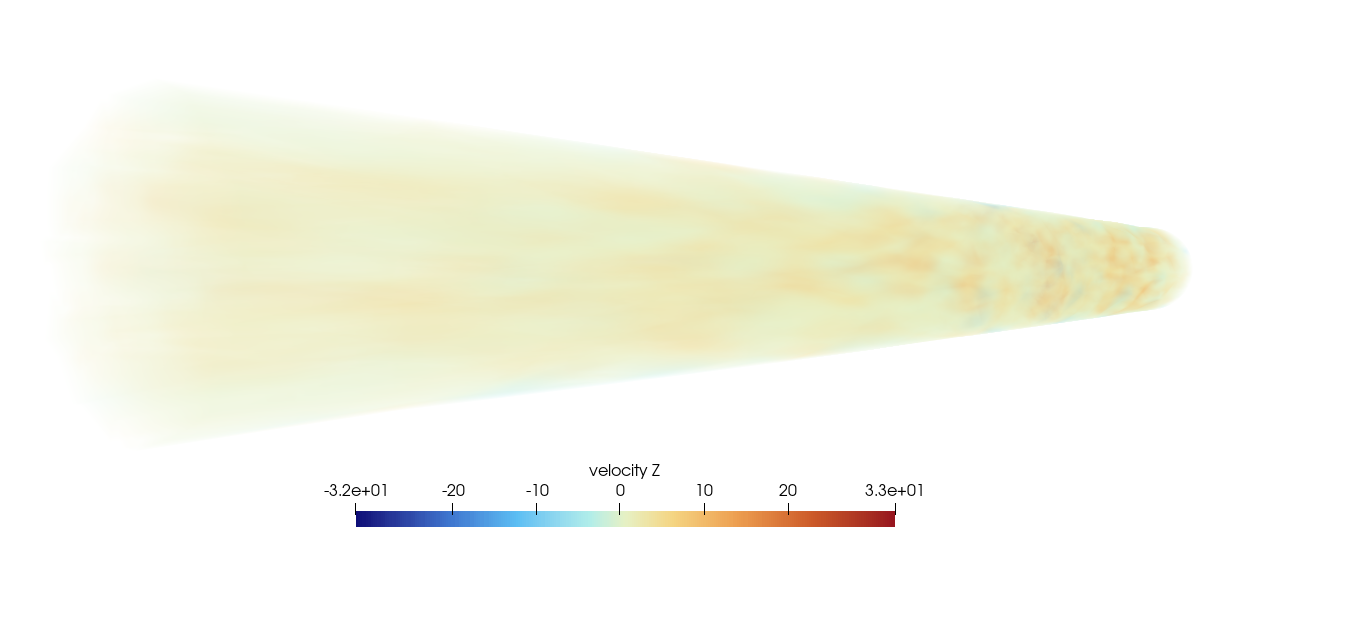}
        \caption{Medium}
    \end{subfigure}

    \begin{subfigure}[t]{0.49\textwidth}
        \centering
        \includegraphics[width=\linewidth]{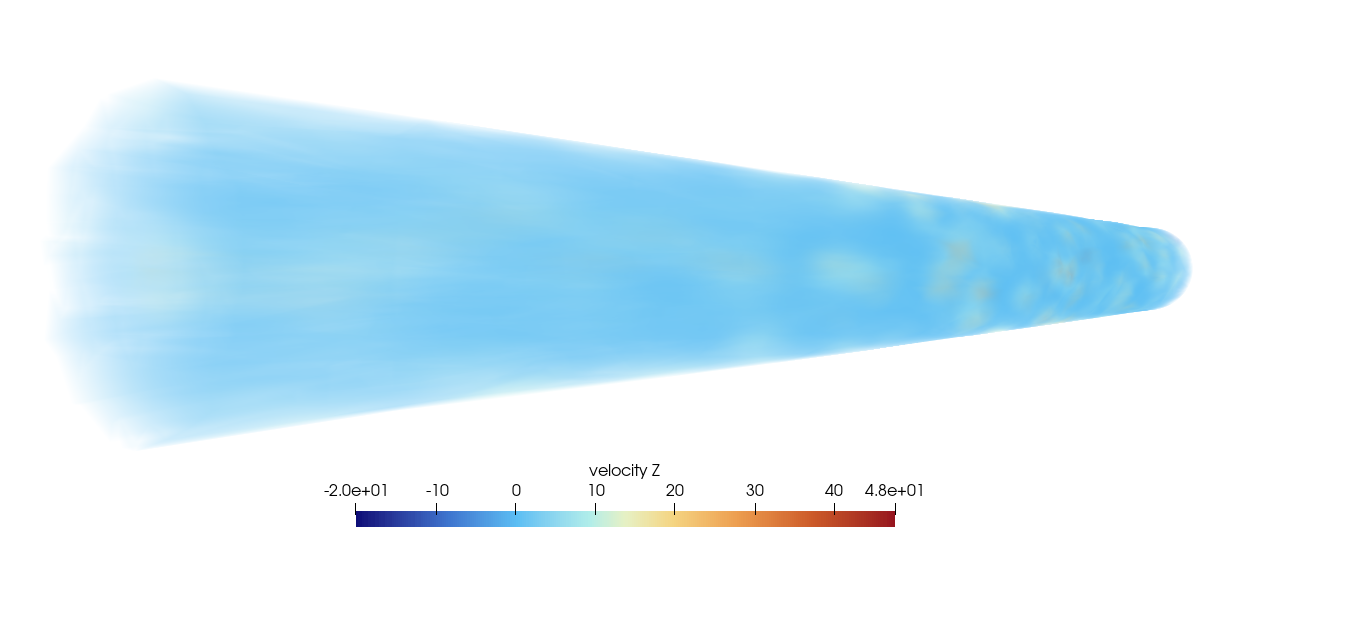}
        \caption{Coarse}
    \end{subfigure}
    \hfill
    \begin{subfigure}[t]{0.49\textwidth}
        \centering
        \includegraphics[width=\linewidth]{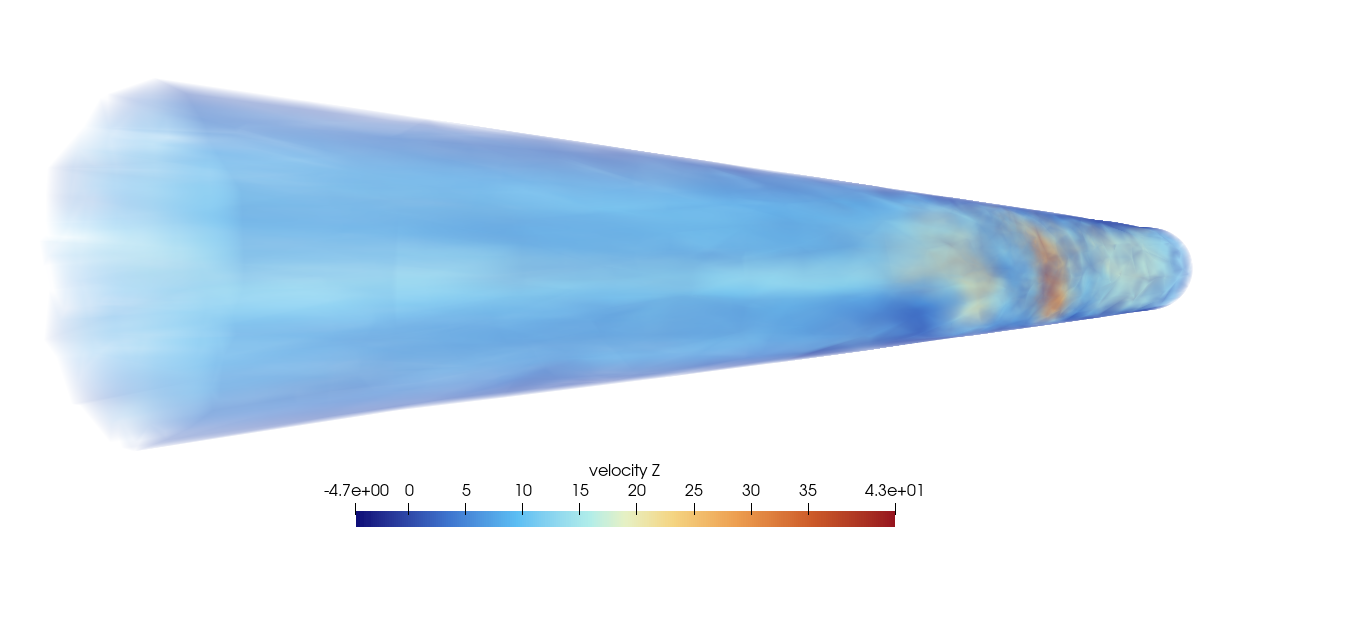}
        \caption{Combined}
    \end{subfigure}

    \caption{Scale-indexed frame decomposition of a predicted solution to the species-transport problem for the $z$-component of velocity at three levels. The bottom right image shows the combined reconstruction.}
    \label{fig:species_transport_z_scales}
\end{figure}
\paragraph{Species Transport}
Finally, Figure~\ref{fig:species_transport_z_scales} illustrates the levelwise frame contributions to the $z$-component of the velocity field in the species-transport problem. The coarsest-level contribution contains the dominant transport pattern, while the intermediate and finest levels show increasingly localized vortical and mixing-induced structure downstream of the blades. As above, these panels visualize the contributions associated with the frame levels and do not imply an orthogonal separation of physical scales.

\FloatBarrier
Similar scale-indexed behavior is observed across the remaining benchmark problems in appendix section~\ref{supplement:multiscale_decomps}. These levelwise outputs provide additional interpretability and a representation that may be useful for coupling operator learning with multiscale numerical methods beyond end-to-end solution prediction.

\section{Conclusions and Future Work}
\label{sec:conclusions}
We presented a novel multiscale kernel frame approximation technique that allows for the decomposition of target functions into multiple spatial length scales. We proved that the frame is interpolatory, established formal error estimates, and presented a doubling conjecture that matches observed numerical convergence rates. We then leveraged the frame as input and output approximation spaces within a kernel method for operator learning, leading to the Frame Kernel Method (FKM). We found that the FKM produced more accurate predictions than several standard operator learning architectures without any changes to the data.

For future work, we aim to prove the doubling conjecture and thereby establish the doubled convergence rates. The sparse QR factorization currently limits our 3D implementation to approximately $O(10^5)$ points. To allow for scaling to millions of spatial locations, we will explore sketching approaches, preconditioning techniques within iterative least squares solvers, and block QR decompositions that take advantage of the block structure of our approximant. Finally, we aim to incorporate the FKM within multiscale-multigrid solvers designed for challenging problems arising from cohesive zone formulations and nonlinear elasticity formulations for solids, textiles, and composites.

\bibliographystyle{siamplain}
% \bibliographystyle{siam}
%\FloatBarrier
\bibliography{references}

\appendix
%\appendix
\section{Target functions}
\label{supplement:synthetic_functions}
\begin{figure}[htbp]
    \centering

    \begin{subfigure}[t]{0.49\textwidth}
        \centering
        \includegraphics[width=\linewidth]{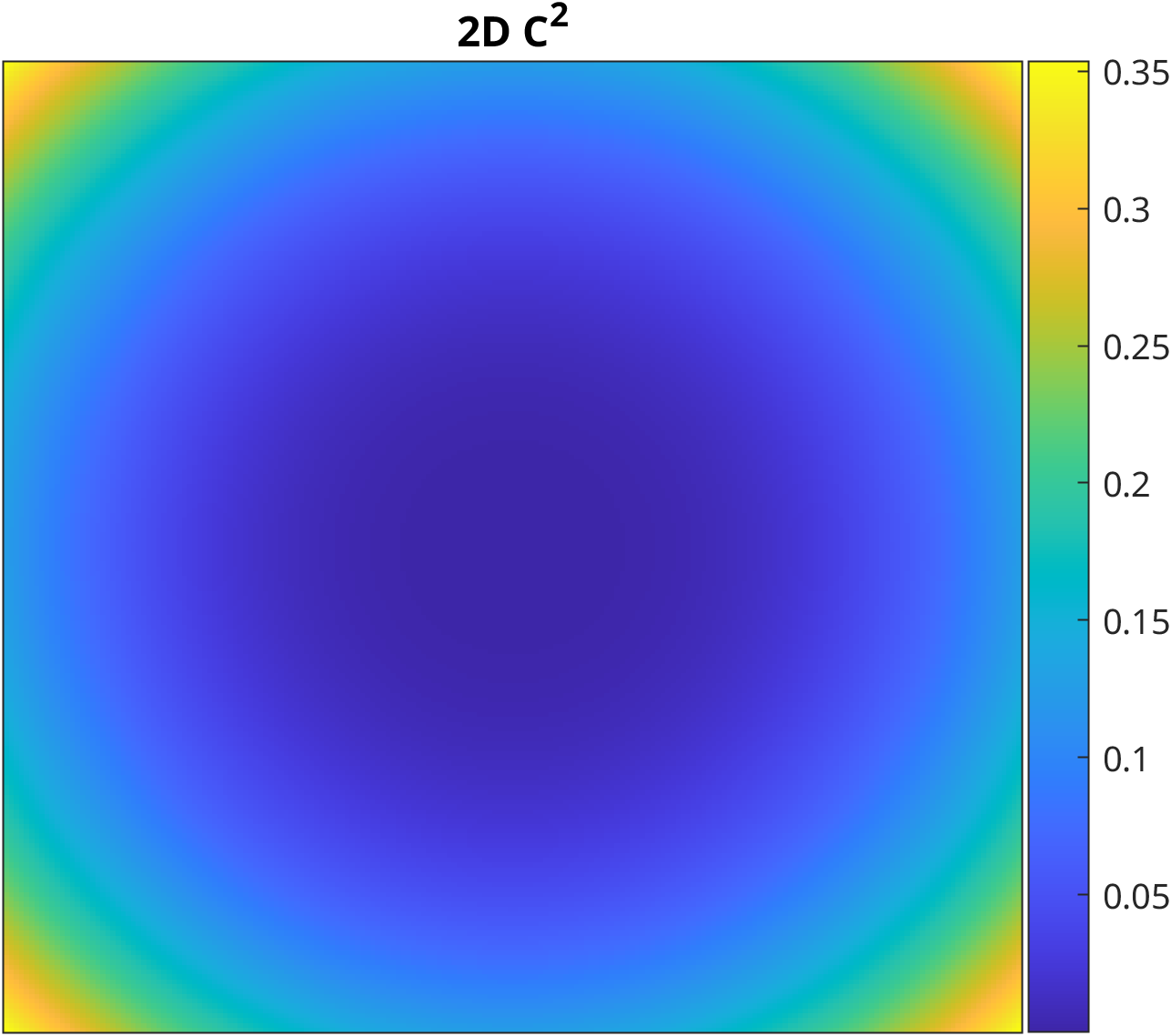}
        \label{fig:synthetic_function_c1_2d}
    \end{subfigure}
    \hfill
    \begin{subfigure}[t]{0.49\textwidth}
        \centering
        \includegraphics[width=\linewidth]{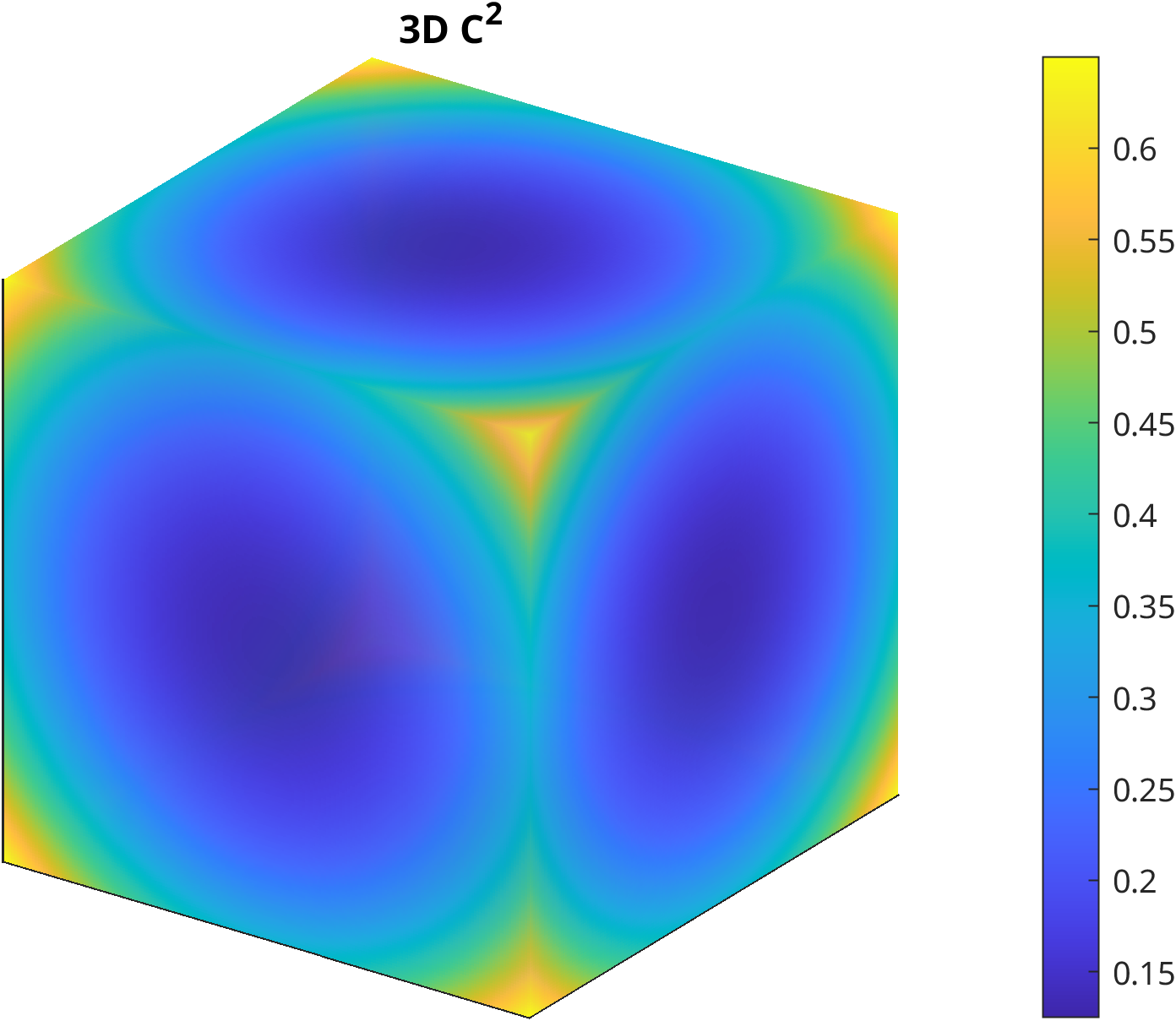}
        \label{fig:synthetic_function_c1_3d}
    \end{subfigure}

    \vspace{0.5em}

    \begin{subfigure}[t]{0.49\textwidth}
        \centering
        \includegraphics[width=\linewidth]{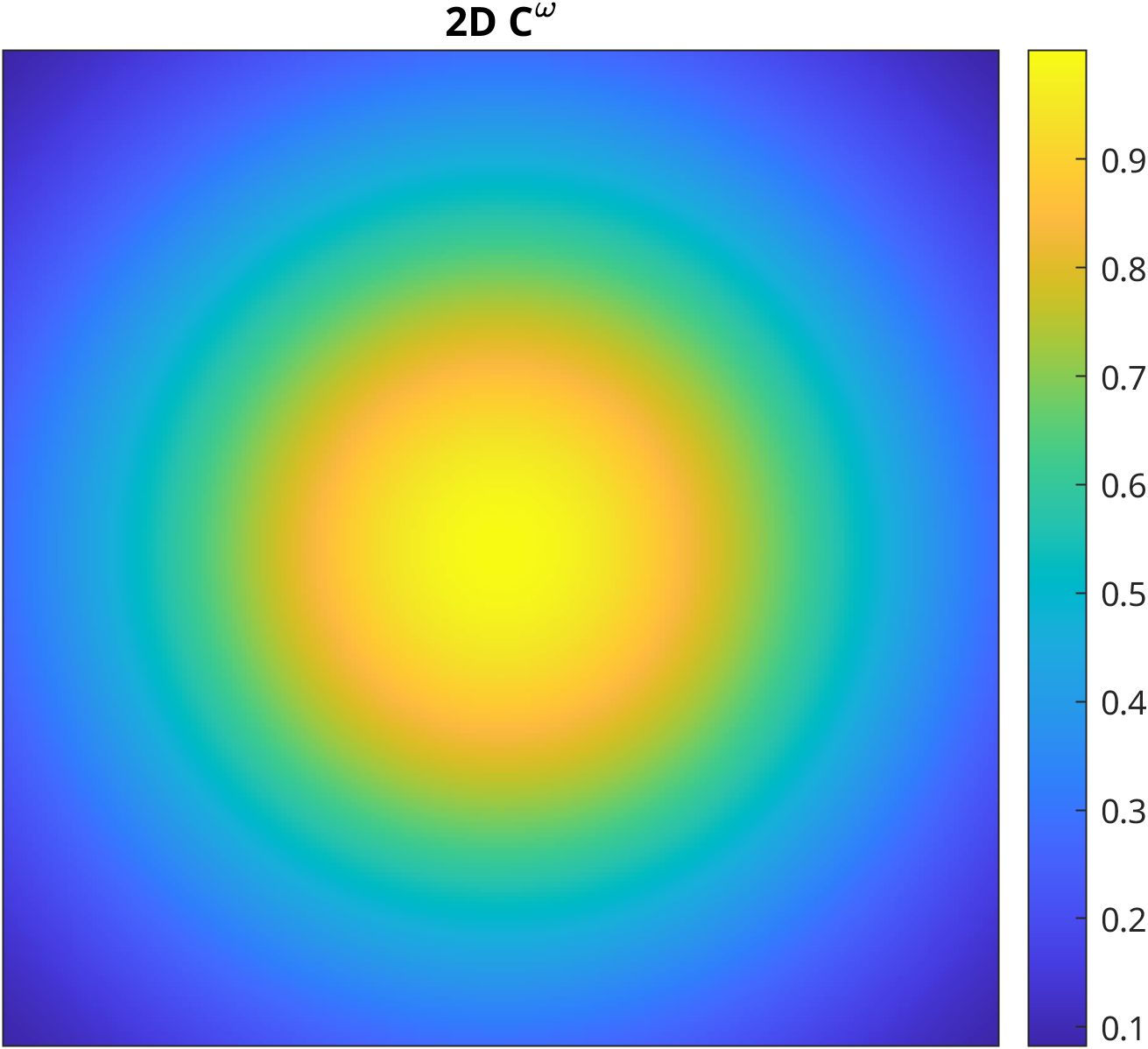}
        \label{fig:synthetic_function_comega_2d}
    \end{subfigure}
    \hfill
    \begin{subfigure}[t]{0.49\textwidth}
        \centering
        \includegraphics[width=\linewidth]{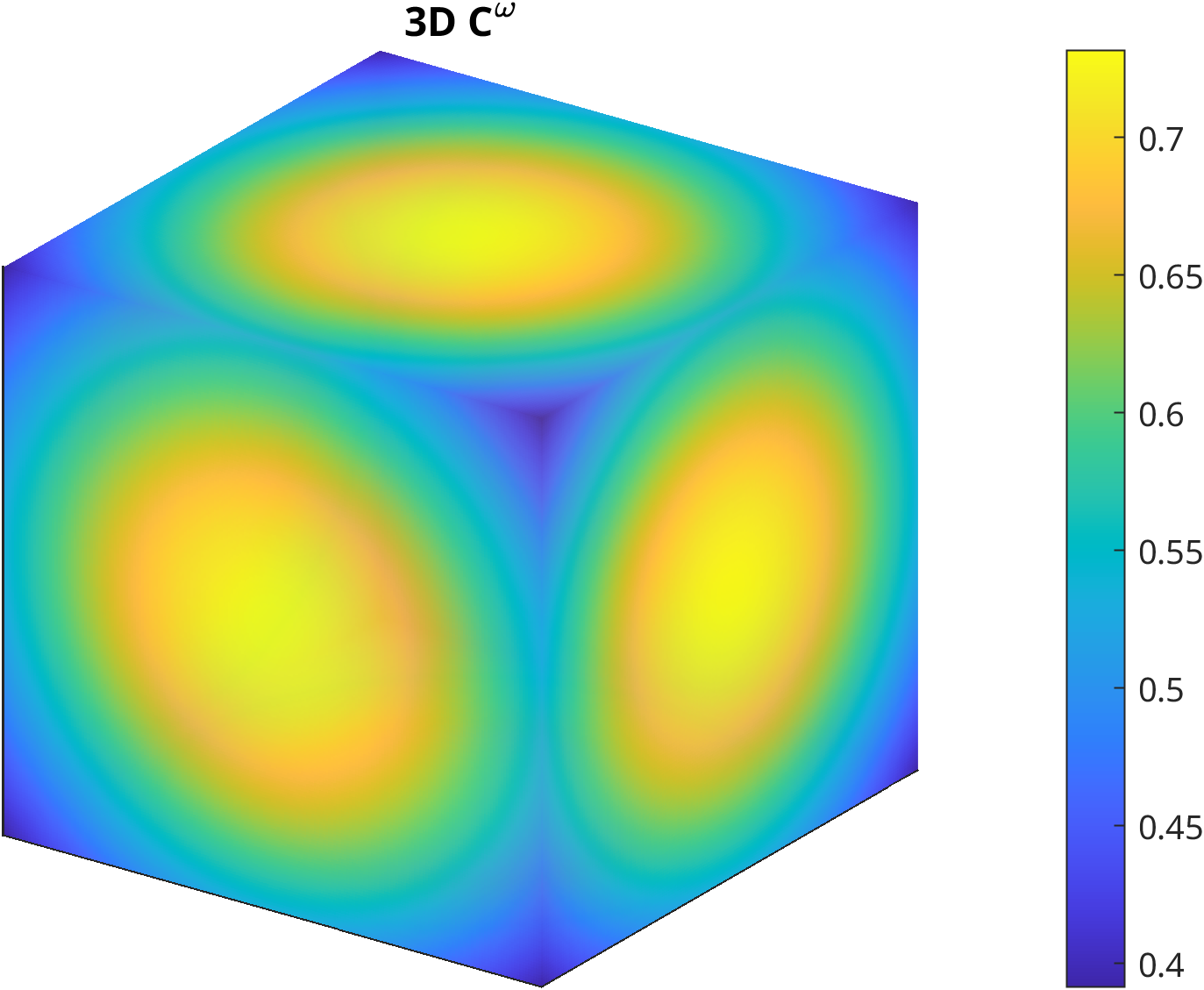}
        \label{fig:synthetic_function_comega_3d}
    \end{subfigure}

    \caption{Synthetic test functions used throughout the numerical experiments. The top row shows the $C^2$ function in two and three dimensions, while the bottom row shows the $C^{\omega}$ function in two and three dimensions.}
    \label{fig:synthetic_functions}
\end{figure}
Figure \ref{fig:synthetic_functions} shows the target functions used for testing numerical convergence rates.

\section{Selection of Kernel Support Size}
\label{supplement:density_sweep}
The kernel support size determines the locality of the multiscale frame and directly influences the sparsity of the design matrix. Rather than varying the support radius directly, we vary the design matrix density, which provides a domain-independent measure of the average kernel support. Figures~\ref{fig:frame_scarcity_combined_2d} and \ref{fig:frame_scarcity_combined_3d} illustrate the tradeoff between reconstruction accuracy at the data sites and at unseen evaluation sites.

At low densities, the kernel supports are highly localized, with few points to fit in a given kernel, producing near-exact reconstruction at the data sites but relatively poor generalization. Increasing the density enlarges the kernel supports, allowing neighboring kernels to overlap more substantially and reducing the evaluation-site error. %This improvement comes at the cost of a gradual increase in the decomposition-site error.

For the analytic function, the evaluation-site error decreases by more than five orders of magnitude as the density increases. The \(C^2\) function exhibits the same overall trend, although the tradeoff is more pronounced. %The evaluation-site error decreases by roughly five orders of magnitude.% while the data site error remains low at moderate densities before increasing rapidly once the density becomes large. This behavior is consistent with the reduced regularity of the \(C^2\) function, which benefits less from increasing kernel overlap.
As the density approaches one, the kernels become nearly global, effectively eliminating the locality and multiscale structure of the frame. Although this dense representation achieves excellent reconstruction accuracy, it sacrifices the sparsity and computational advantages of the multiscale formulation. These results motivate the use of a design matrix density of \(20\%\) throughout the experiments presented in the main text, as it provides a favorable balance between reconstruction accuracy, sparsity, and preservation of the multiscale structure.

\begin{figure}[!t]
    \centering
    \includegraphics[width=\linewidth]{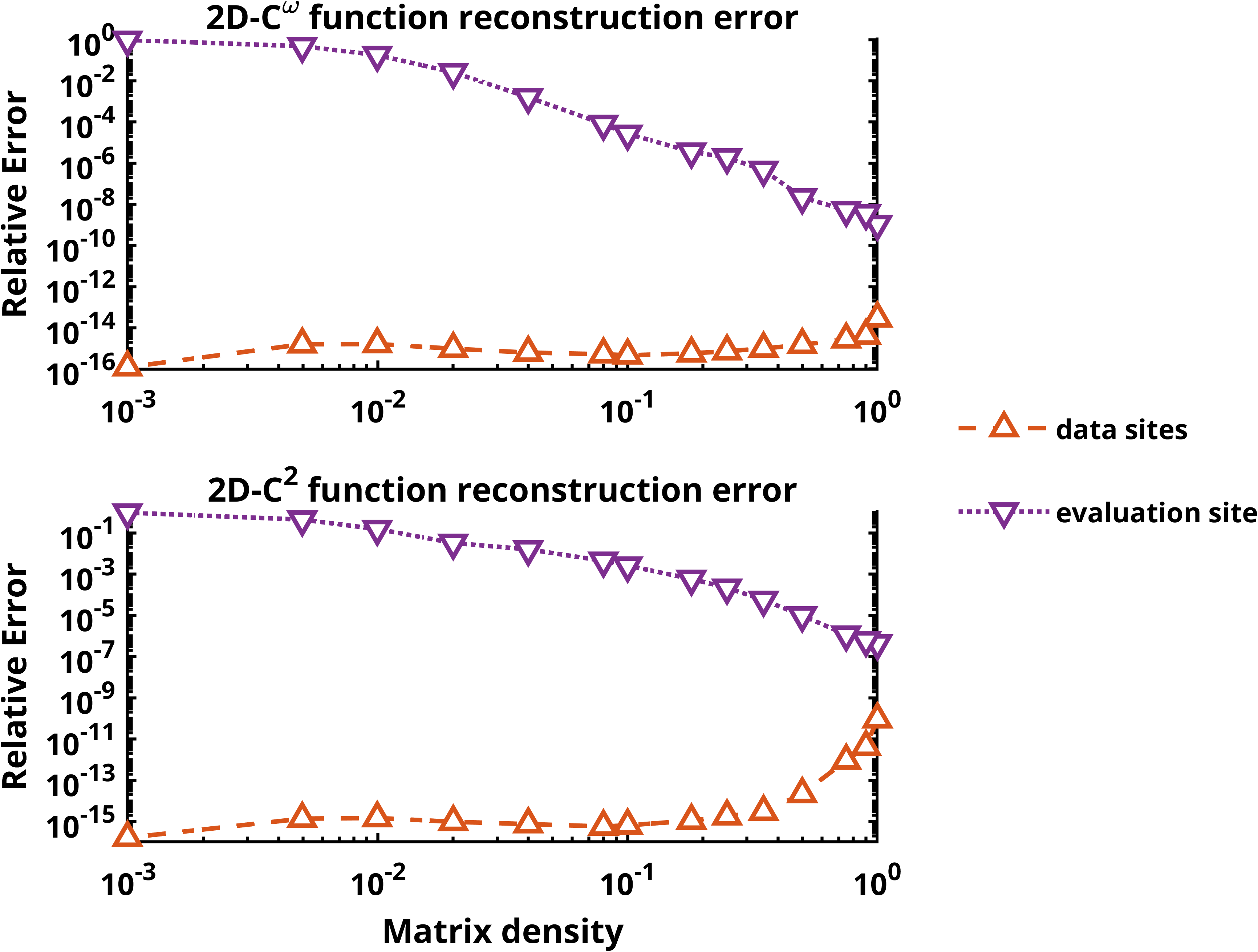}
    \caption{Reconstruction accuracy as a function of design matrix density for the test families in two dimensions. Results are shown for the analytic (\(C^\omega\)) and finitely smooth (\(C^2\)) functions.}
    \label{fig:frame_scarcity_combined_2d}
\end{figure}

\begin{figure}[!t]
    \centering
    \includegraphics[width=\linewidth]{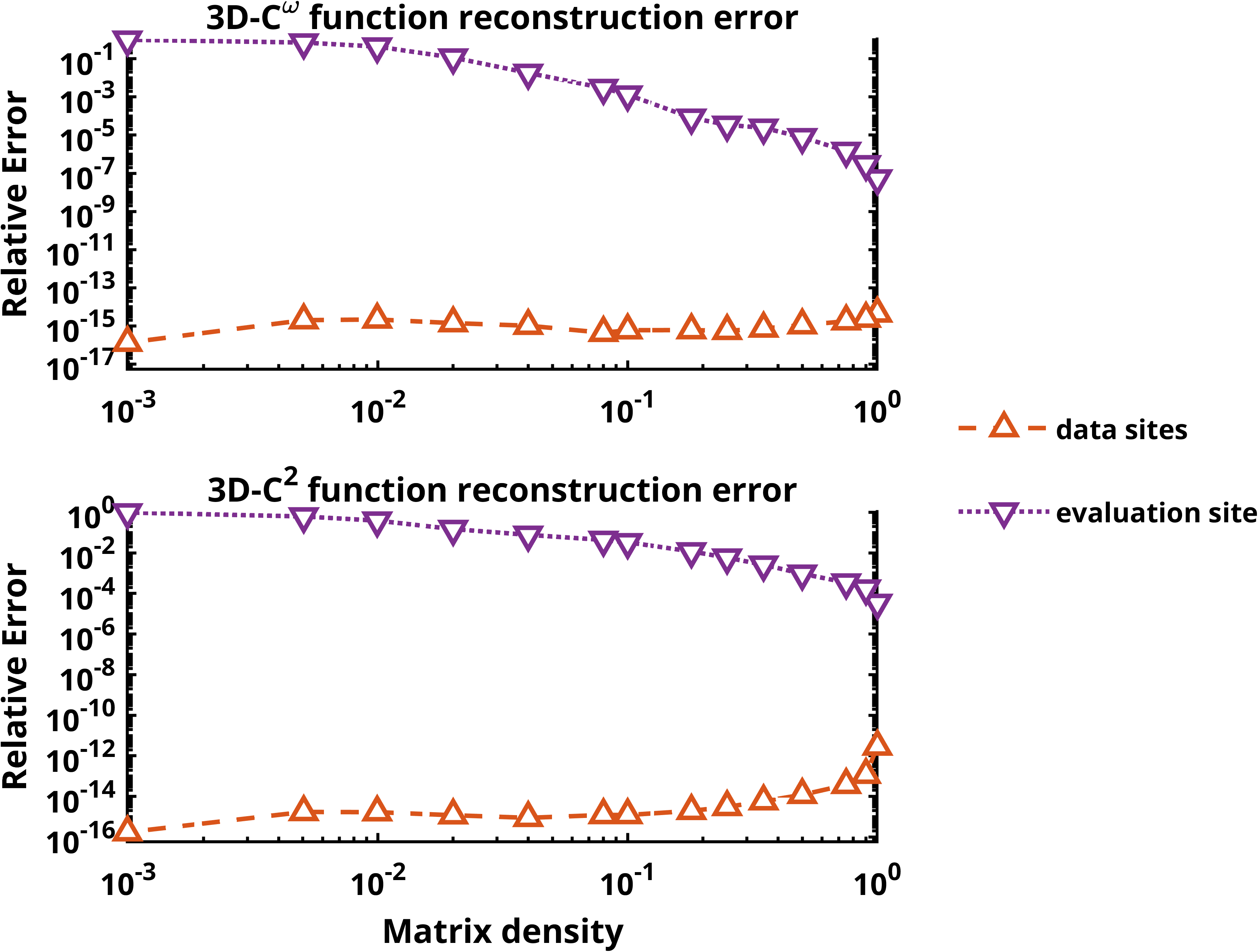}
    \caption{Reconstruction accuracy as a function of design matrix density for the test families in three dimensions. Results are shown for the analytic (\(C^\omega\)) and finitely smooth (\(C^2\)) functions.}
    \label{fig:frame_scarcity_combined_3d}
\end{figure}

\section{Interpolation Plots}
\label{supplement:Interpolation_Plots}
\begin{figure}[!htbp]
    \centering
    \includegraphics[width=0.8\textwidth]{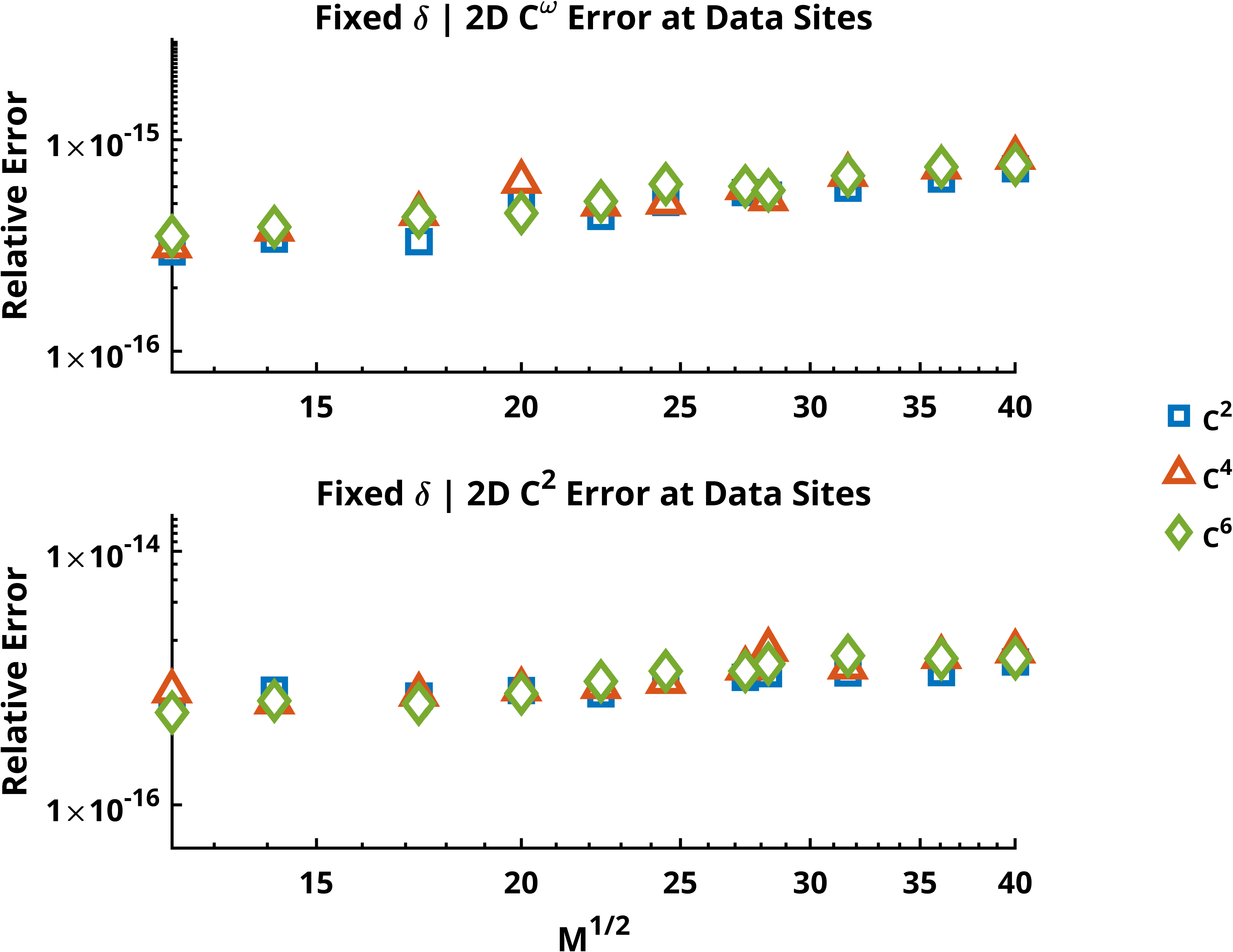}
    \caption{$\ell_2$ error at the decomposition (data) sites as a function of the number of decomposition points while holding the design matrix density at approximately $0.2$ (2D).}
    \label{fig:PC_N_Error_pin_density_on_node_2d}
\end{figure}

\begin{figure}[!htbp]
    \centering
    \includegraphics[width=0.8\textwidth]{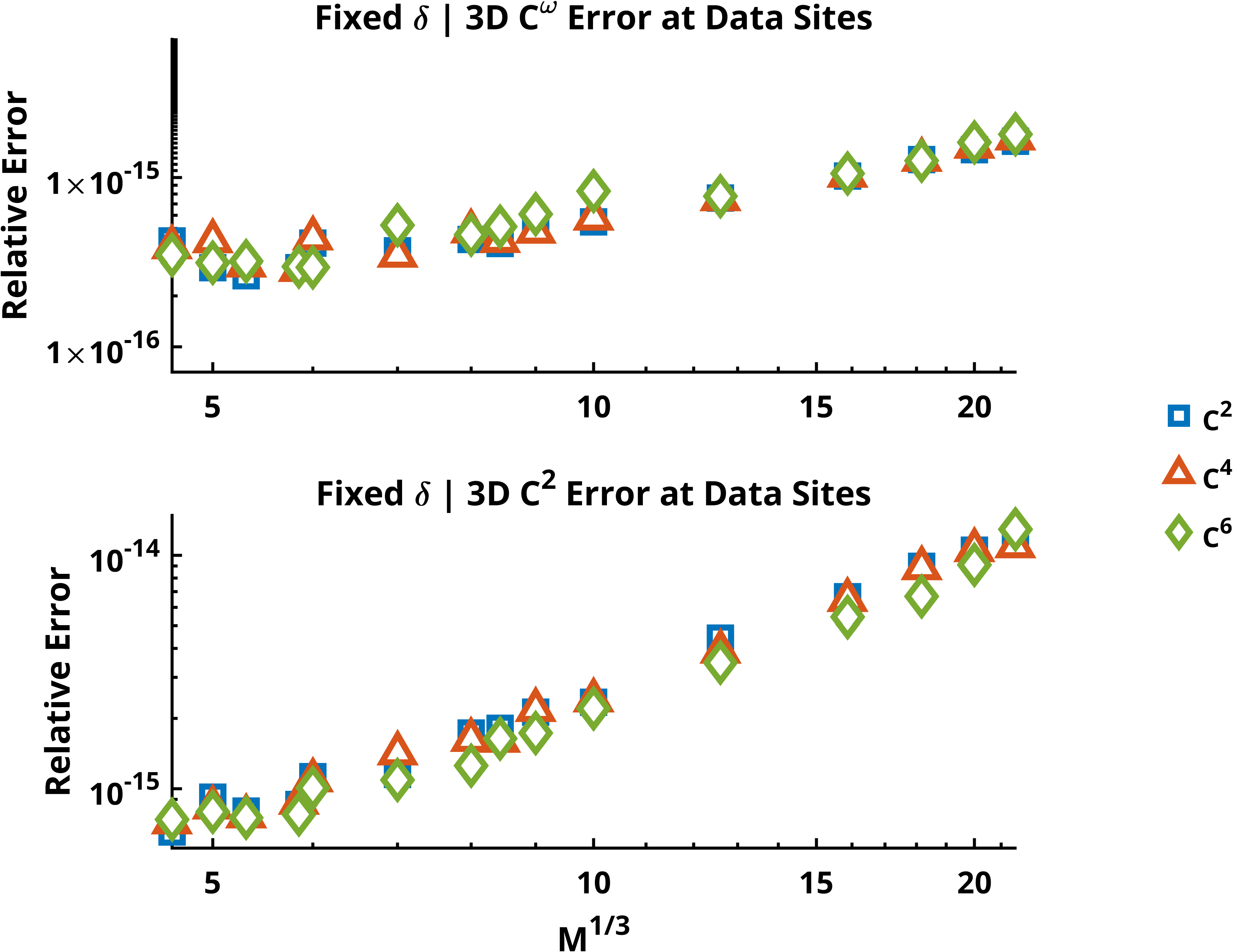}
    \caption{$\ell_2$ error at the decomposition (data) sites as a function of the number of decomposition points while holding the design matrix density at approximately $0.2$ (3D).}
    \label{fig:PC_N_Error_pin_density_on_node_3d}
\end{figure}
\begin{figure}[!htbp]
    \centering
    \includegraphics[width=0.8\textwidth]{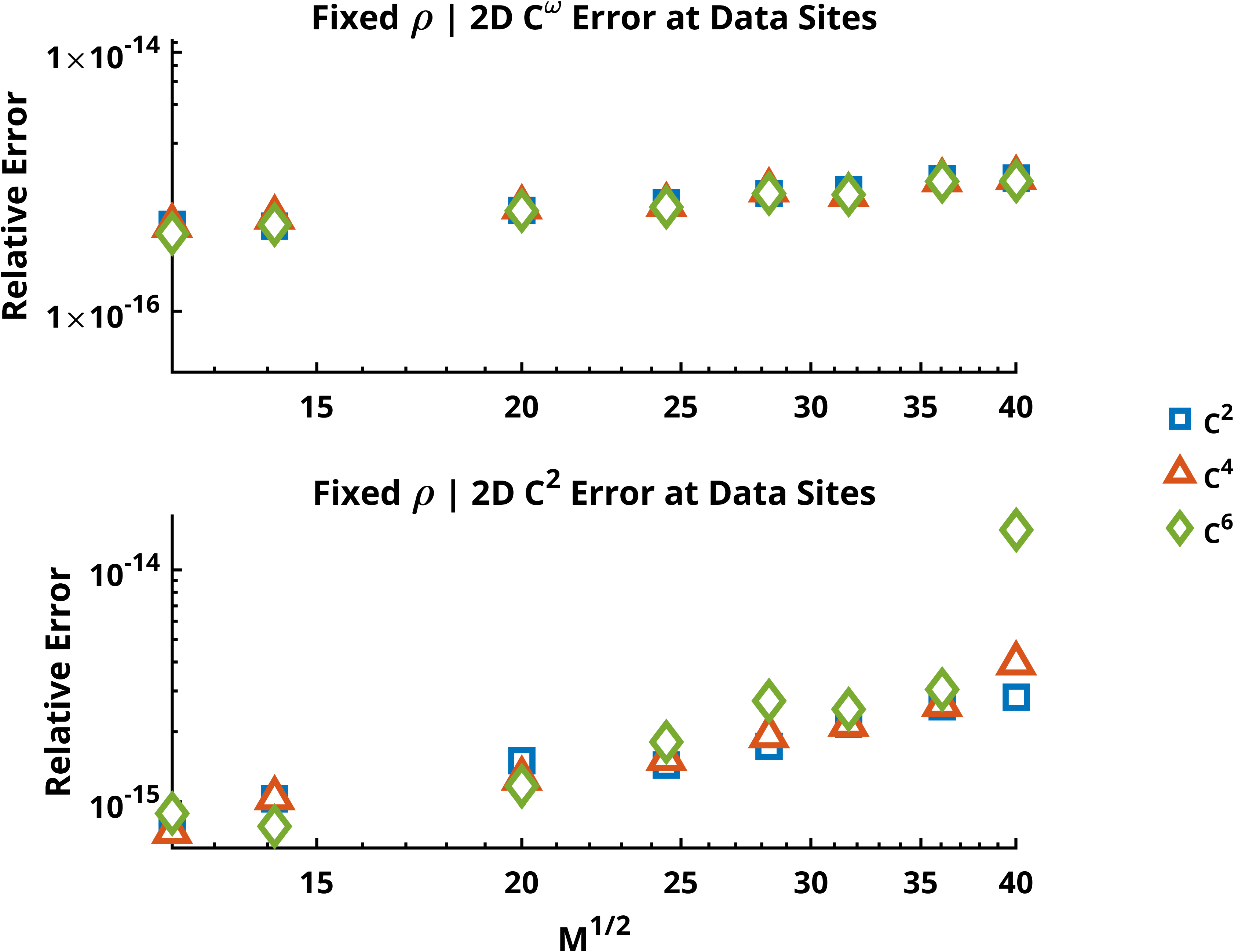}
    \caption{$\ell_2$ error as a function of the number of points in the domain. The kernel support size ($\rho$) is chosen on the finest scale to produce a design matrix density of approximately $0.2$, and this value is then held fixed as the number of points changes.}
    \label{fig:point_cloud_on_node_finest_scale_fixed_support_2d}
\end{figure}

\begin{figure}[!htbp]
    \centering
    \includegraphics[width=0.8\textwidth]{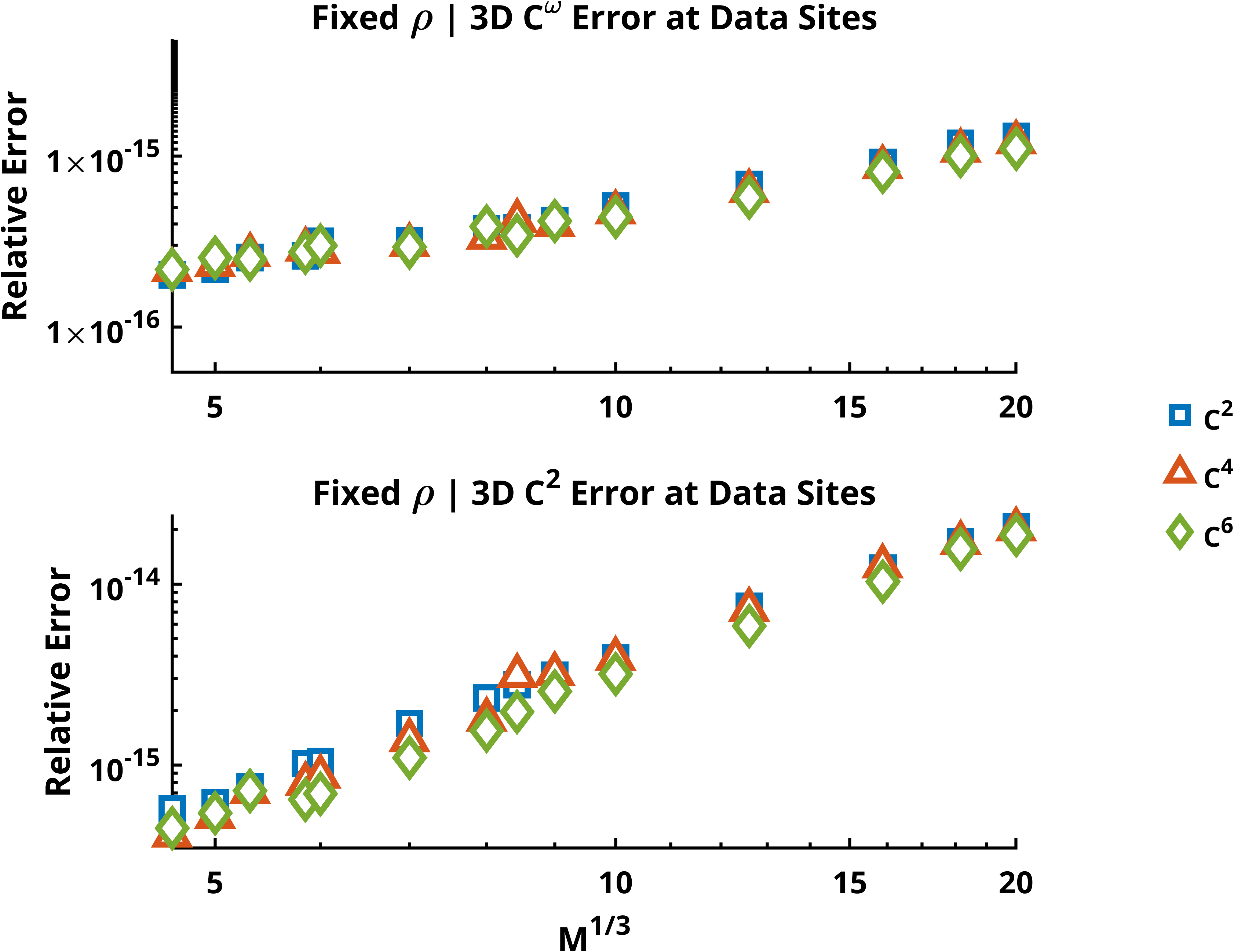}
    \caption{$\ell_2$ error as a function of the number of points in the domain. The kernel support size ($\rho$) is chosen on the finest scale to produce a design matrix density of approximately $0.2$, and this value is then held fixed as the number of points changes.}
    \label{fig:point_cloud_on_node_finest_scale_fixed_support_3d}
\end{figure}
These results show interpolation errors with $\delta$ fixed at 20\% for the constant density case, and $\rho$ chosen to achieve that density on the finest scale for the case of constant support radius. As shown in Figures ~\ref{fig:PC_N_Error_pin_density_on_node_2d} and ~\ref{fig:point_cloud_on_node_finest_scale_fixed_support_2d}, the reconstruction error at the 2D data sites increases slightly as $M$ grows, but stays close to machine precision, validating our theoretical result. Since $\delta$ and $\rho$ are held constant, increasing $M$ also increases the average number of neighboring points contained within each kernel. Figures ~\ref{fig:PC_N_Error_pin_density_on_node_3d} and~\ref{fig:point_cloud_on_node_finest_scale_fixed_support_3d} show the trend holding in 3D as well.%Consequently, each local approximation must simultaneously fit a larger number of data points, leading to a modest increase in interpolation error. We see the same trend, though slightly more pronounced,  in Fig.~\ref{fig:PC_N_Error_pin_density_on_node_3d}, and Fig.~\ref{fig:point_cloud_on_node_finest_scale_fixed_support_3d} for the 3D functions. However, these increases still leave our error well within the domain of machine precision.

\section{Operator Learning Multiscale Plots}
\label{supplement:multiscale_decomps}
We show multiscale decompositions (upon generalization) for some of the other test problems in our operator learning benchmark suite.
\begin{figure}[htbp]
    \centering

    \begin{subfigure}[b]{0.49\textwidth}
        \centering
        \includegraphics[width=\textwidth]{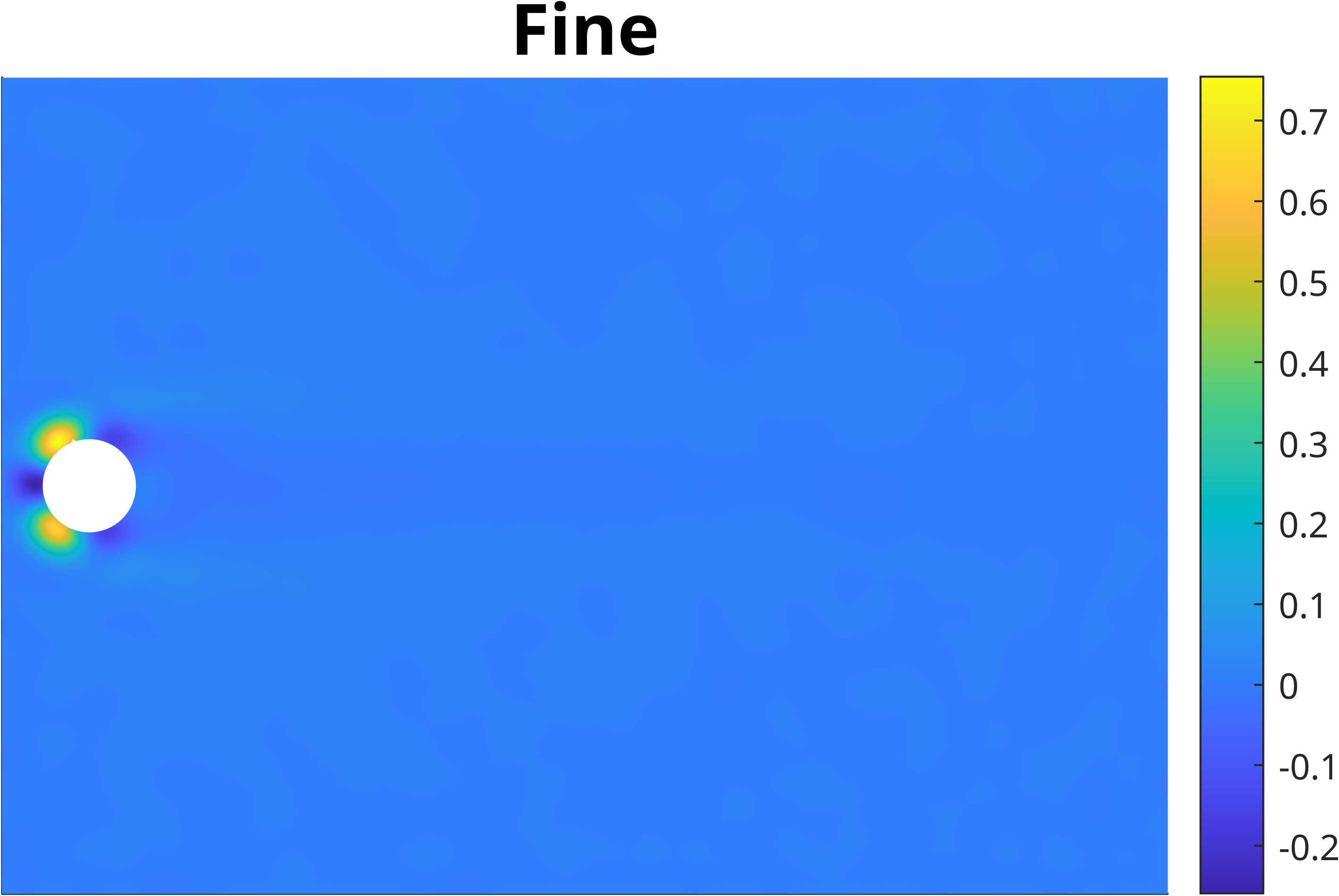}
        \caption{}
    \end{subfigure}
    \hfill
    \begin{subfigure}[b]{0.49\textwidth}
        \centering
        \includegraphics[width=\textwidth]{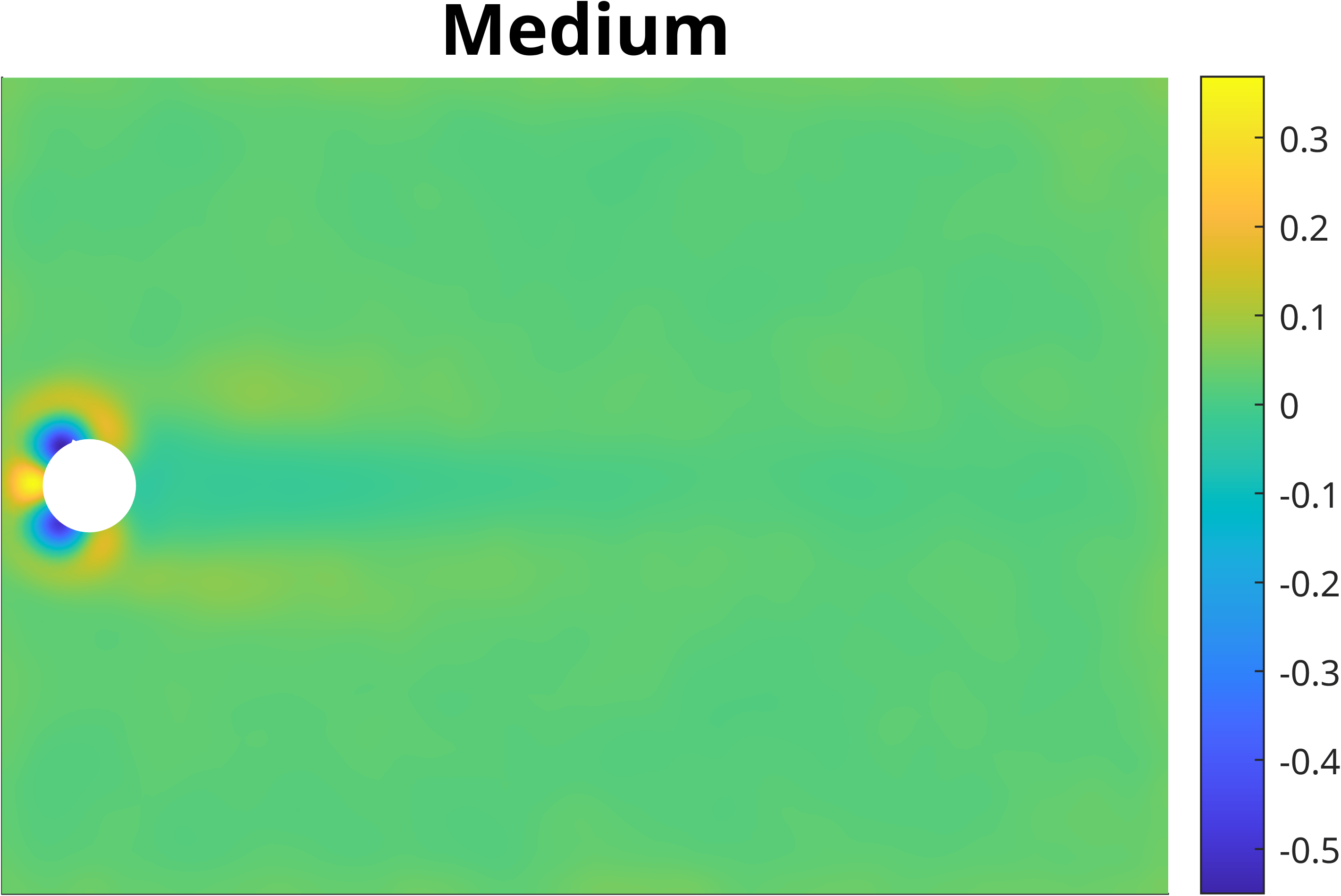}
        \caption{}
    \end{subfigure}
    \hfill
    \begin{subfigure}[b]{0.49\textwidth}
        \centering
        \includegraphics[width=\textwidth]{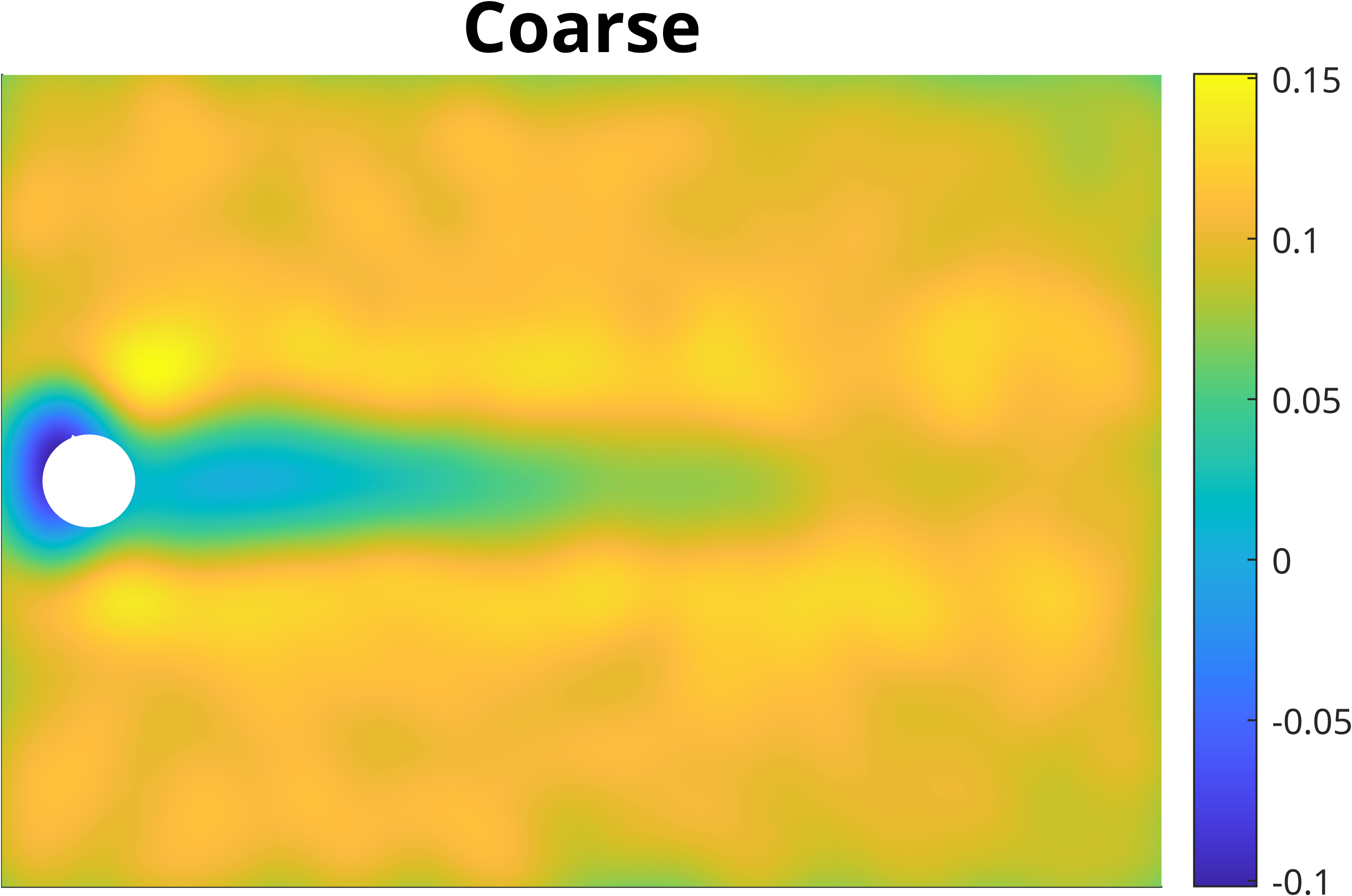}
        \caption{}
    \end{subfigure}
    \hfill
    \begin{subfigure}[b]{0.49\textwidth}
        \centering
        \includegraphics[width=\textwidth]{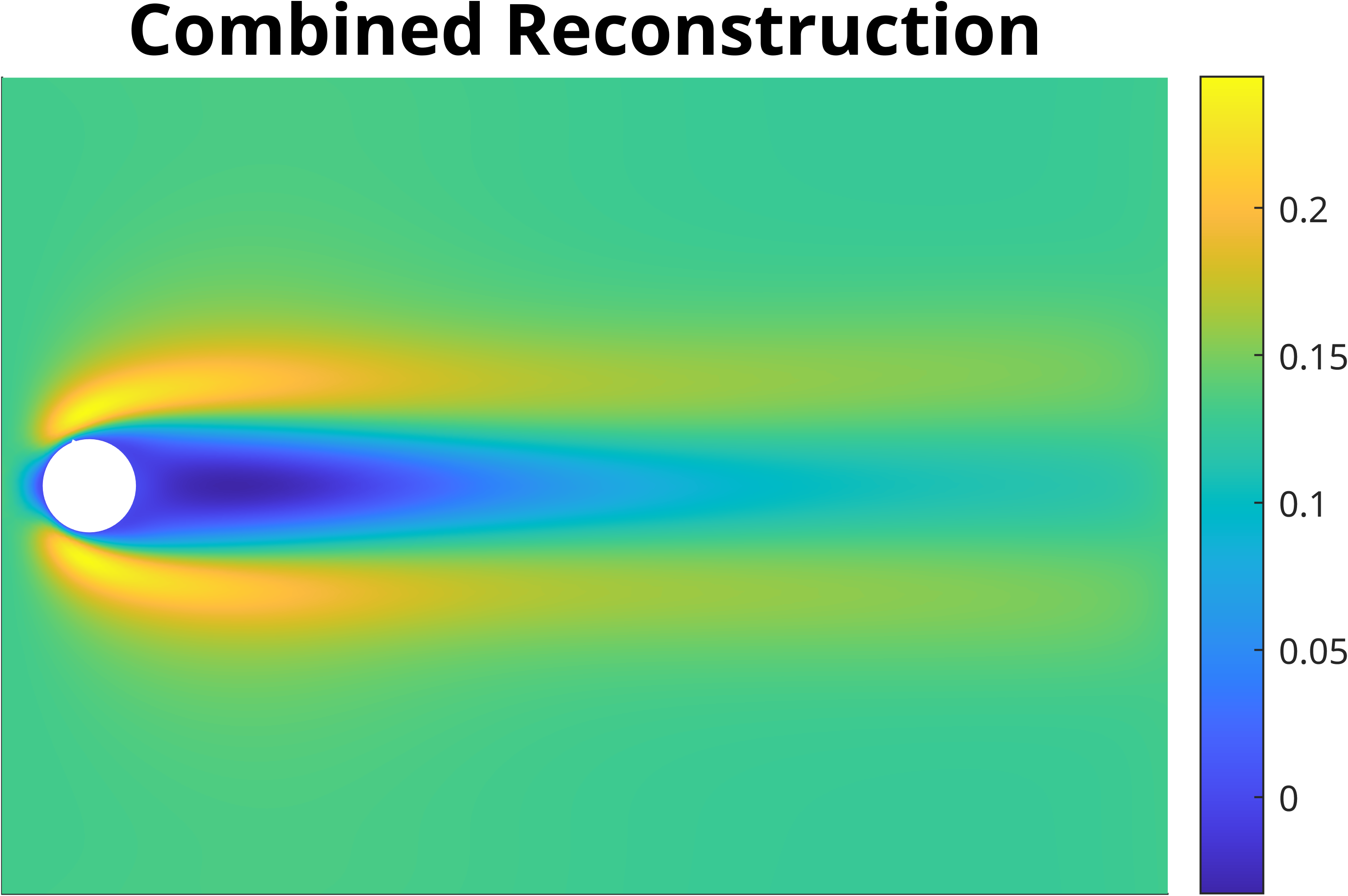}
        \caption{}
    \end{subfigure}

    \caption{Multiscale decomposition of the laminar cylinder flow problem at three different scales. The bottom right image shows the combined reconstruction.}
    \label{fig:cylinder_flow_laminar_scales}
\end{figure}
\paragraph{Cylinder Flow (Laminar)}
In Figure~\ref{fig:cylinder_flow_laminar_scales}, the coarse scale captures the dominant wake structure downstream of the cylinder; the medium scale highlights localized wake disturbances and flow deflection caused by the cylinder; and the fine scale isolates the smallest flow features concentrated near the cylinder surface.

\begin{figure}[htbp]
    \centering
    \begin{subfigure}[b]{0.49\textwidth}
        \centering
        \includegraphics[width=\textwidth]{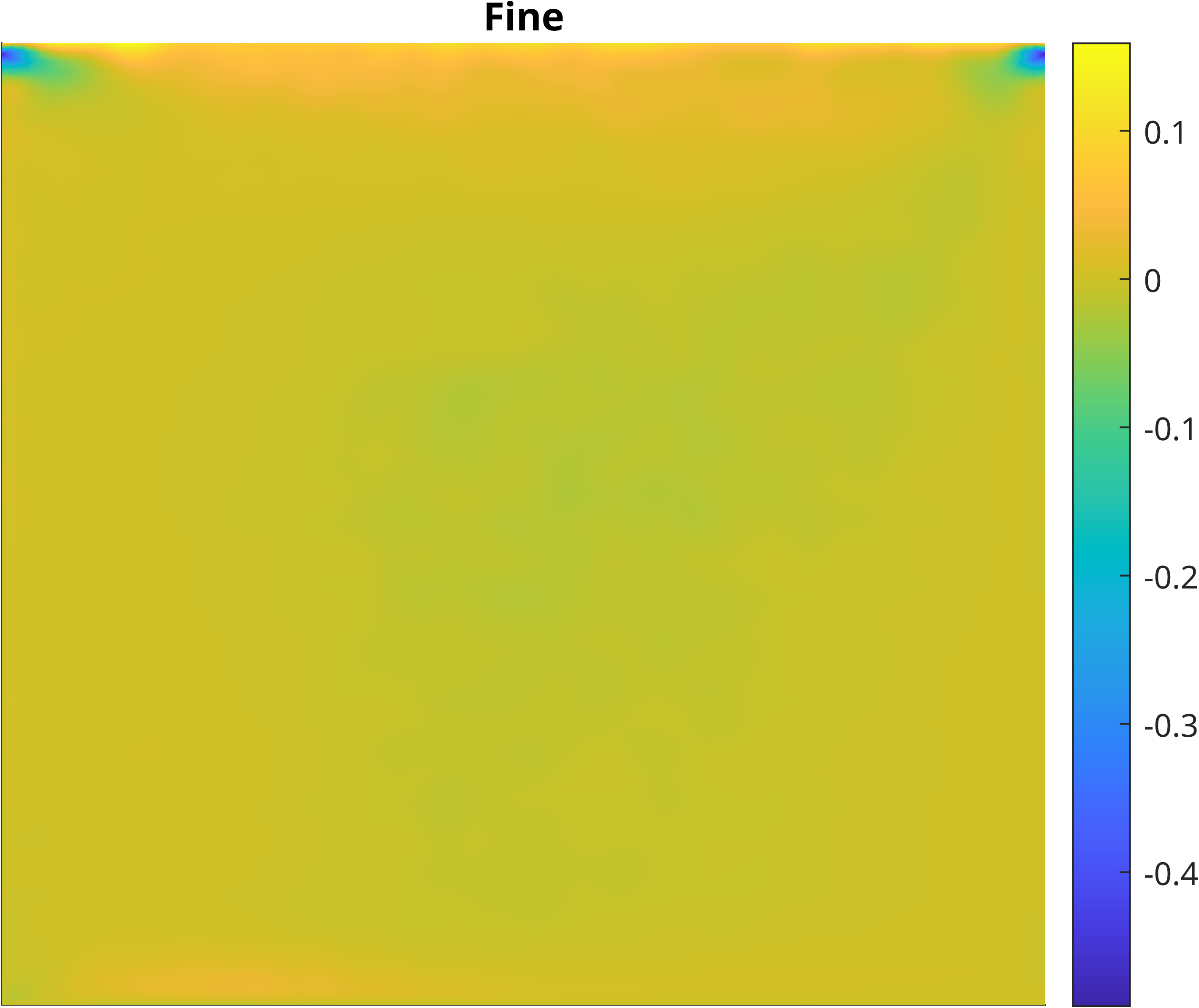}
        \caption{}
    \end{subfigure}
    \hfill
    \begin{subfigure}[b]{0.49\textwidth}
        \centering
        \includegraphics[width=\textwidth]{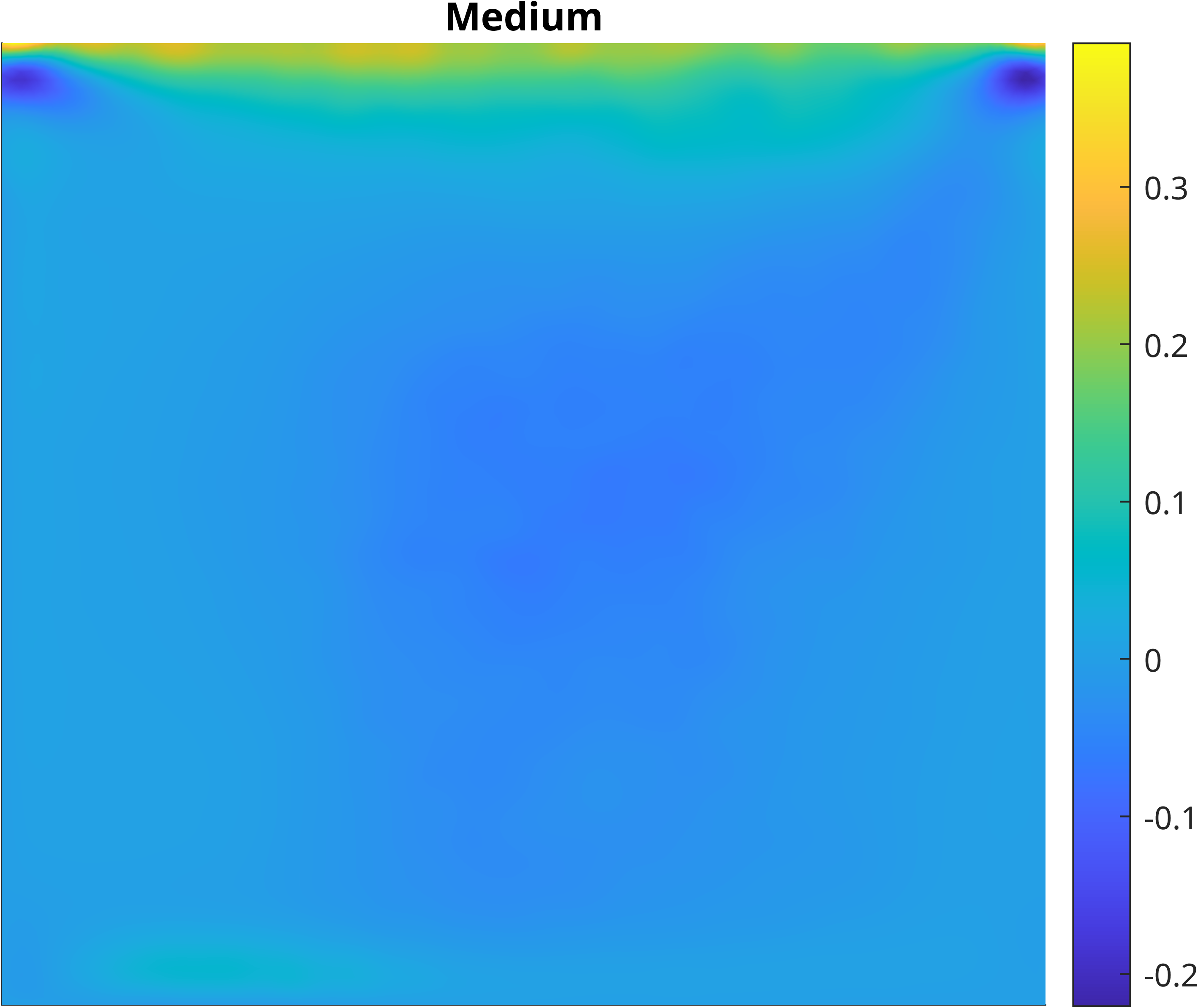}
        \caption{}
    \end{subfigure}
    \hfill
    \begin{subfigure}[b]{0.49\textwidth}
        \centering
        \includegraphics[width=\textwidth]{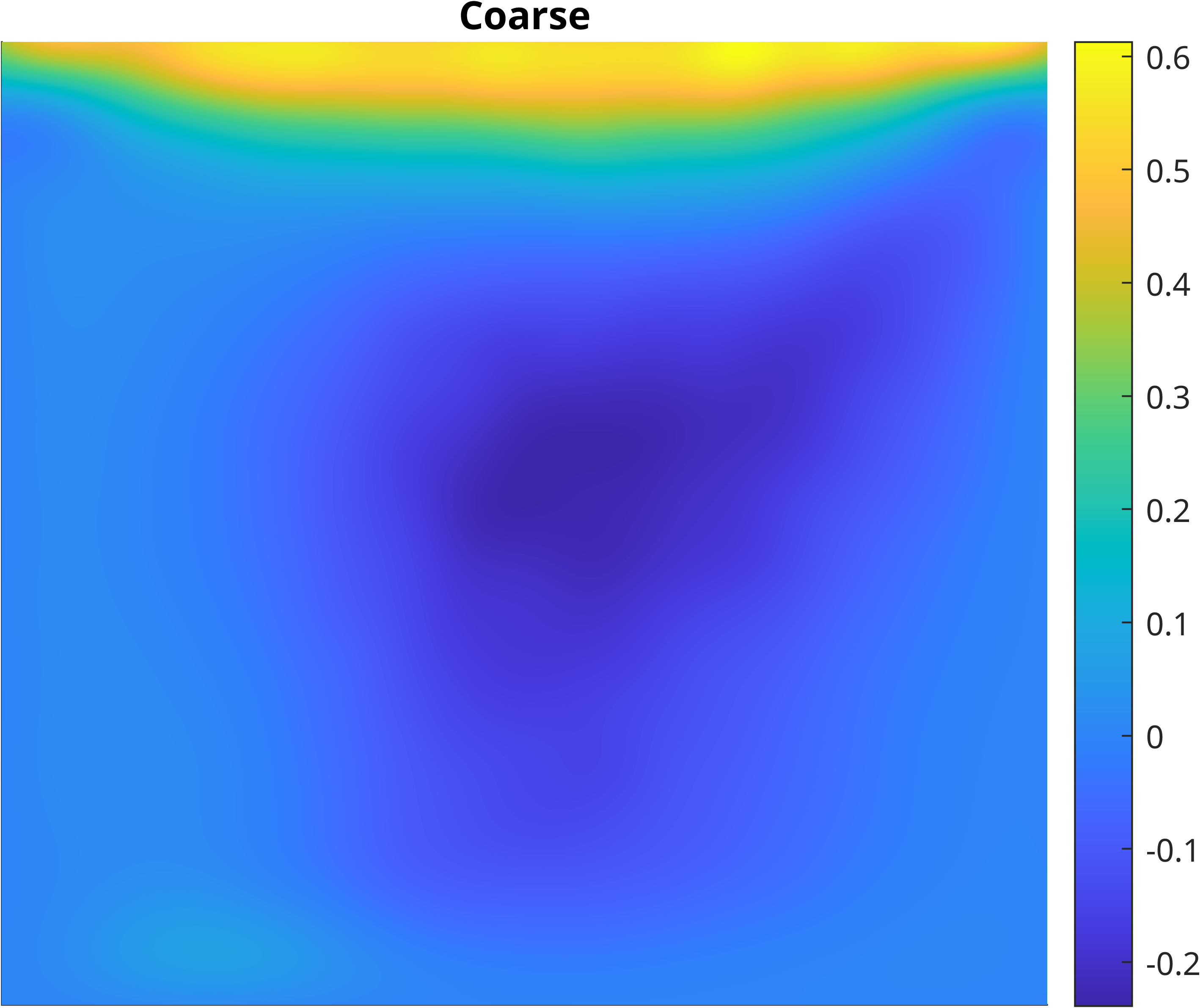}
        \caption{}
    \end{subfigure}
    \hfill
    \begin{subfigure}[b]{0.49\textwidth}
        \centering
        \includegraphics[width=\textwidth]{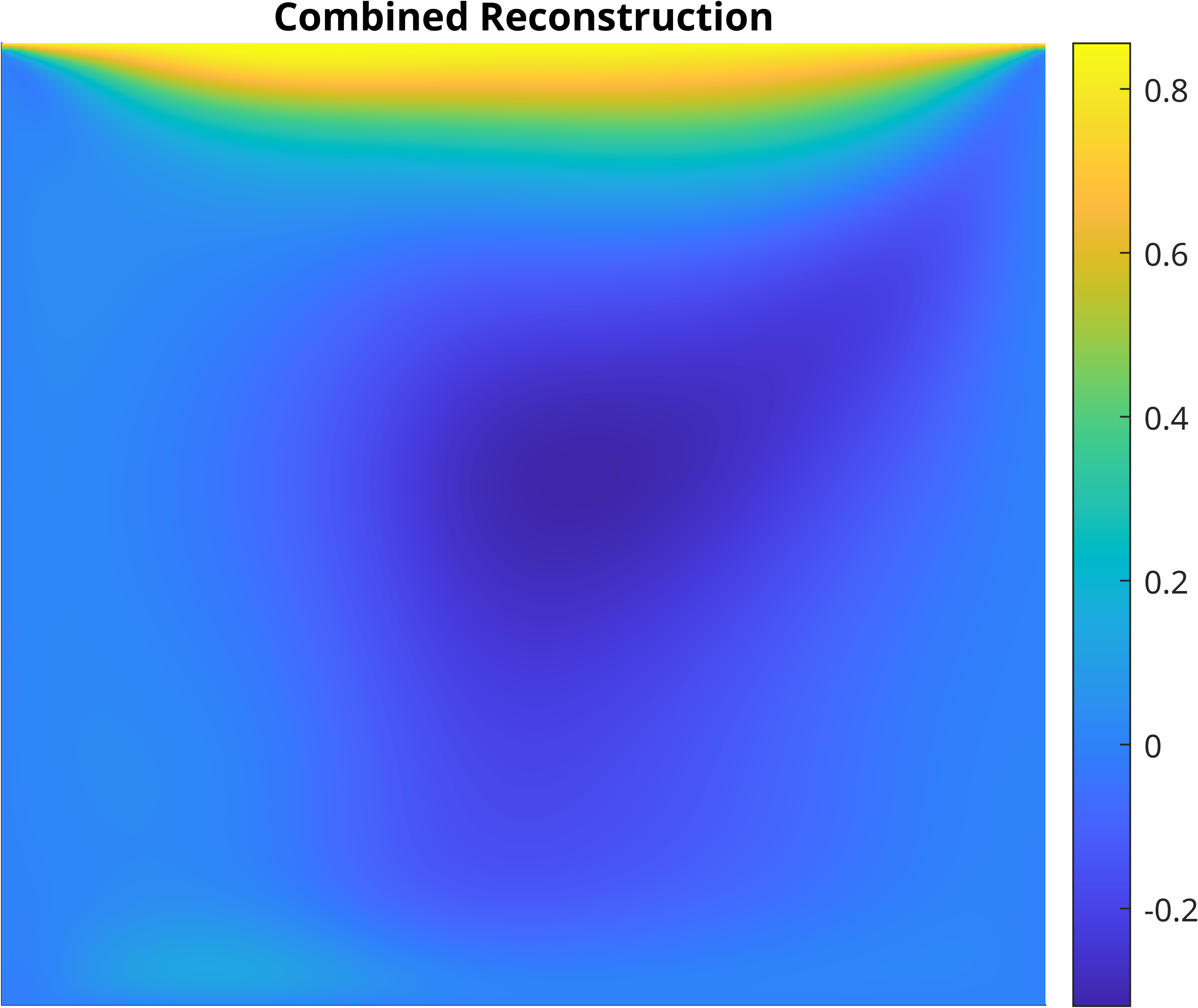}
        \caption{}
    \end{subfigure}

    \caption{Multiscale decomposition of the lid-driven cavity flow problem at three different scales. The bottom right image shows the combined reconstruction.}
    \label{fig:cavity_flow_ram_scales}
\end{figure}
\FloatBarrier
\paragraph{Cavity Flow}
In Figure~\ref{fig:cavity_flow_ram_scales}, the coarse scale captures the dominant behavior throughout most of the domain, while the medium and fine scales resolve increasingly localized features concentrated near the top boundary and the upper corners.

\begin{figure}[htbp]
    \centering

    \begin{subfigure}[b]{0.49\textwidth}
        \centering
        \includegraphics[width=\textwidth]{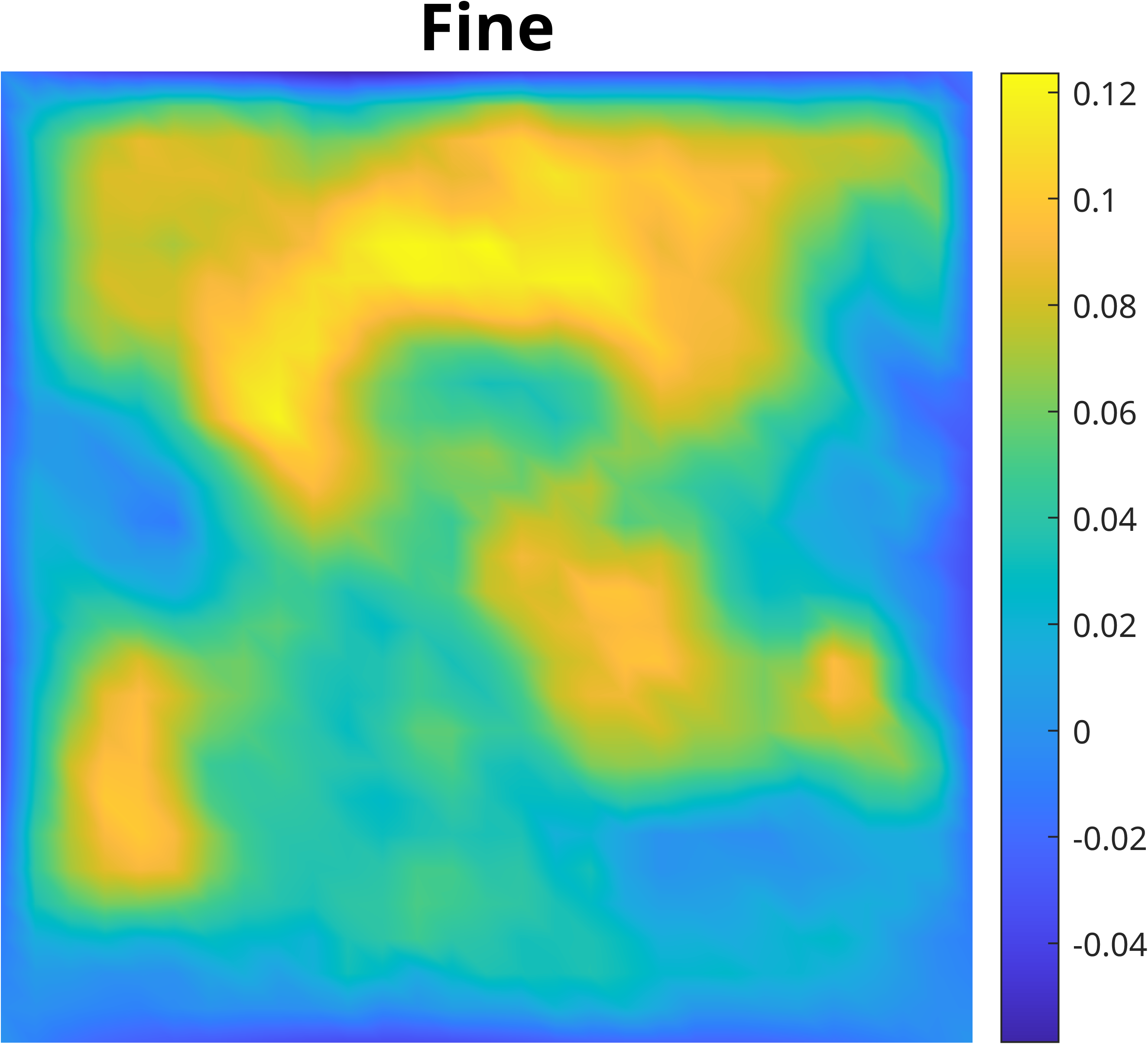}
        \caption{}
    \end{subfigure}
    \hfill
    \begin{subfigure}[b]{0.49\textwidth}
        \centering
        \includegraphics[width=\textwidth]{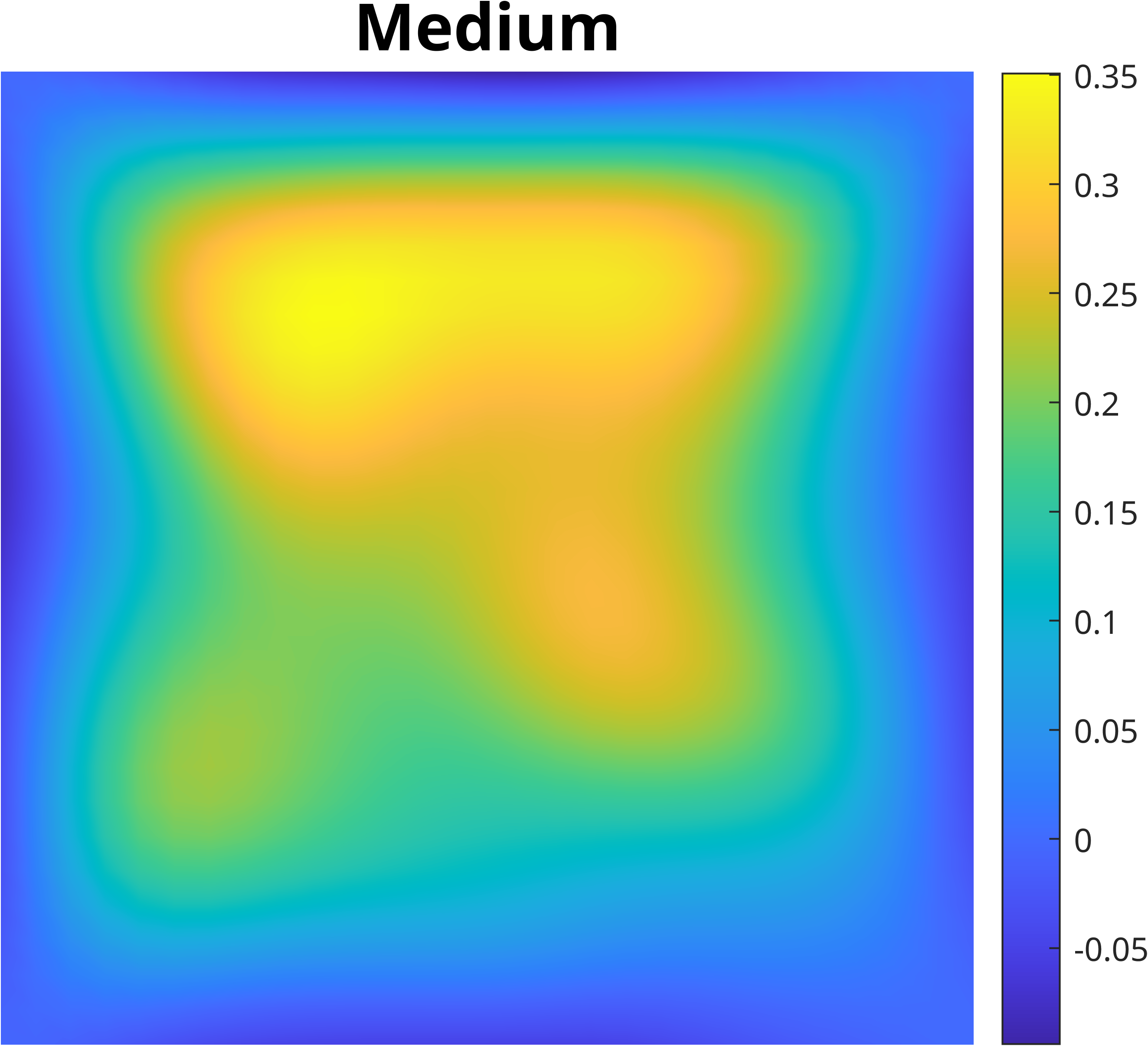}
        \caption{}
    \end{subfigure}
    \hfill
    \begin{subfigure}[b]{0.49\textwidth}
        \centering
        \includegraphics[width=\textwidth]{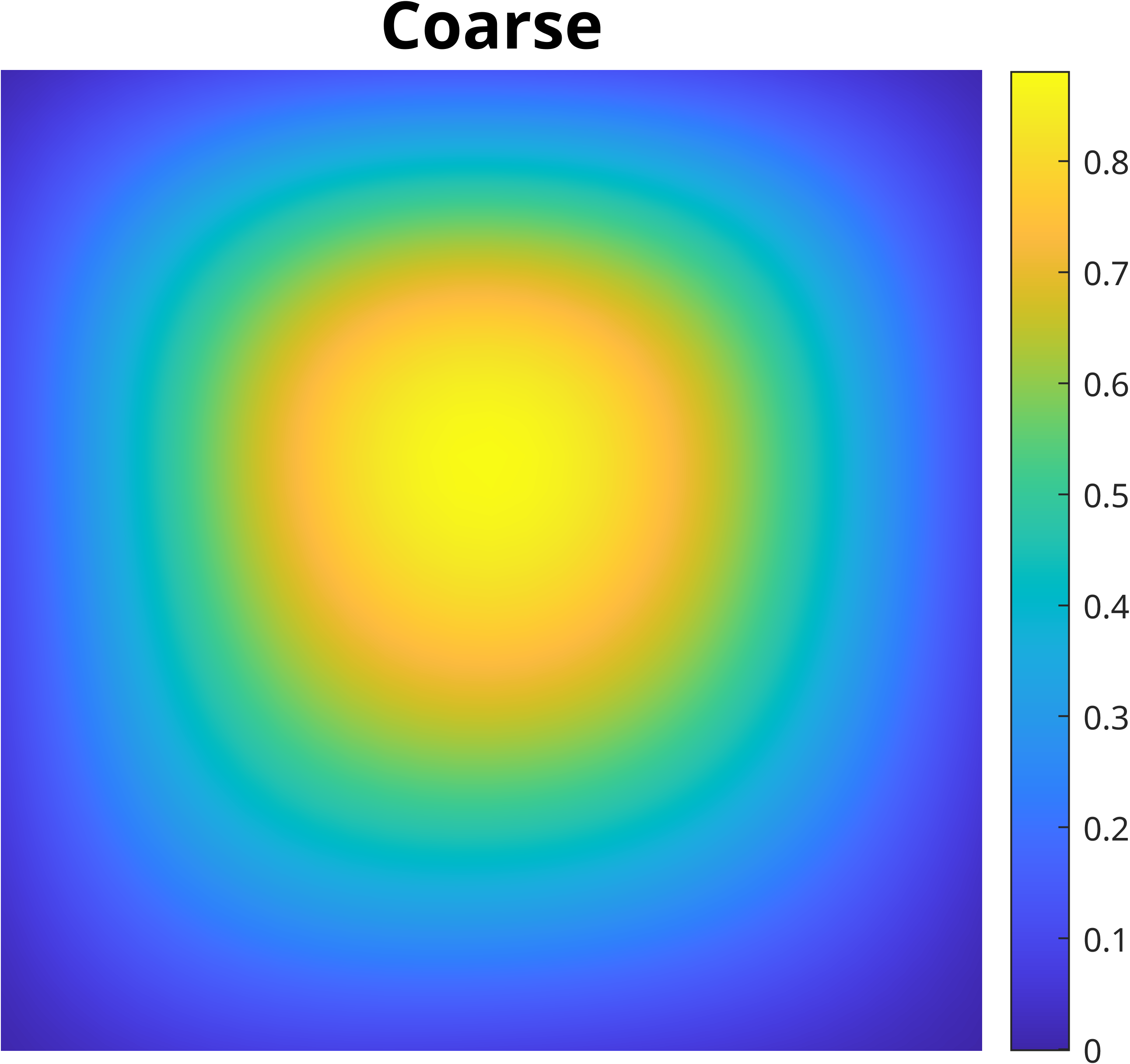}
        \caption{}
    \end{subfigure}
    \hfill
    \begin{subfigure}[b]{0.49\textwidth}
        \centering
        \includegraphics[width=\textwidth]{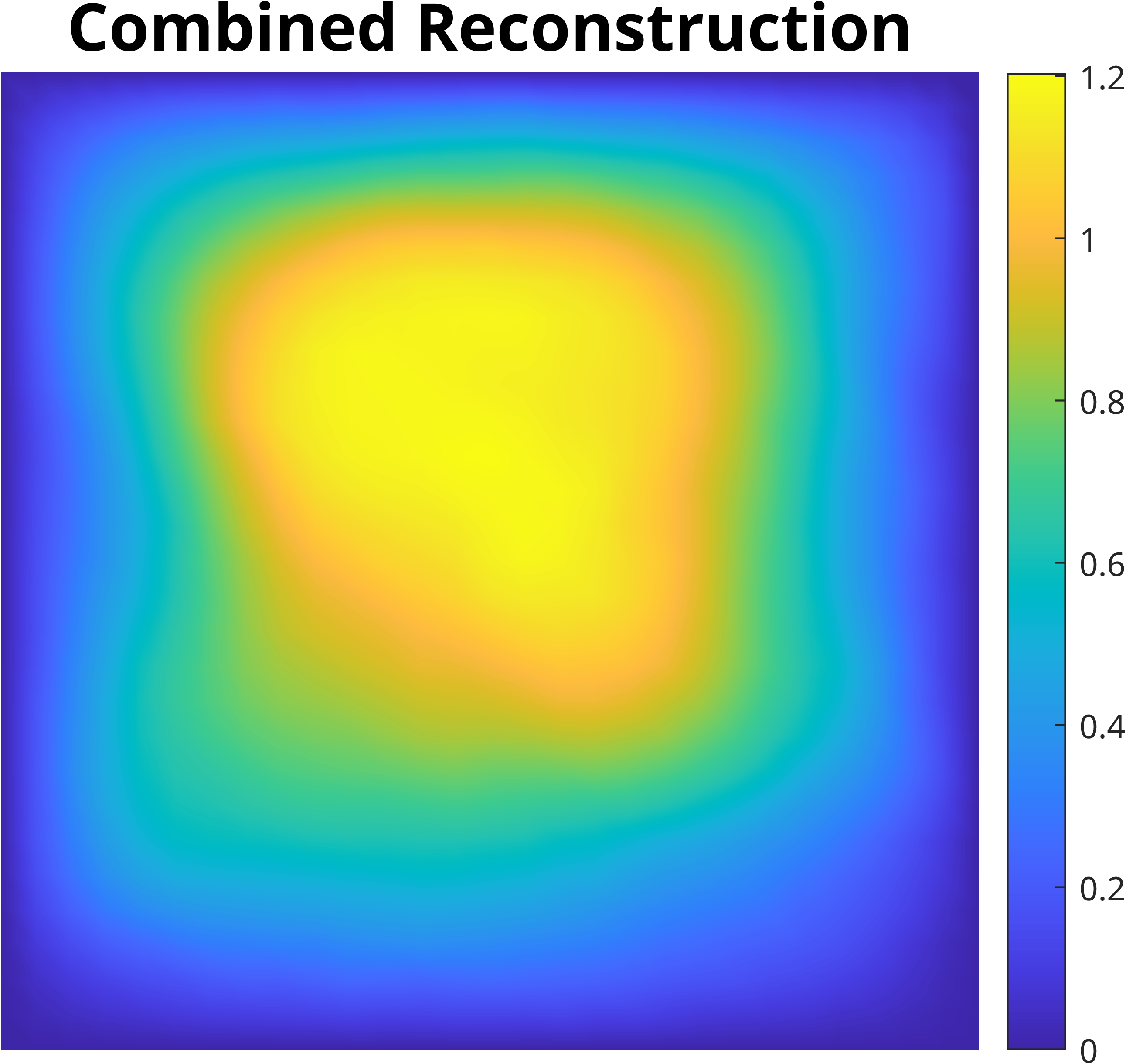}
        \caption{}
    \end{subfigure}

    \caption{Multiscale decomposition of the piecewise-constant Darcy flow problem at three different scales. The bottom right image shows the combined reconstruction.}
    \label{fig:darcy_pwc_scales}
\end{figure}
\FloatBarrier
\paragraph{Darcy PWC}
In Figure~\ref{fig:darcy_pwc_scales}, the coarse scale captures the overall smooth global trend trend of the solution, while the intermediate and fine scales progressively refine its shape by resolving the localized features needed to accurately reproduce the full solution.

\begin{figure}[htbp]
    \centering

    \begin{subfigure}[b]{0.49\textwidth}
        \centering
        \includegraphics[width=\textwidth]{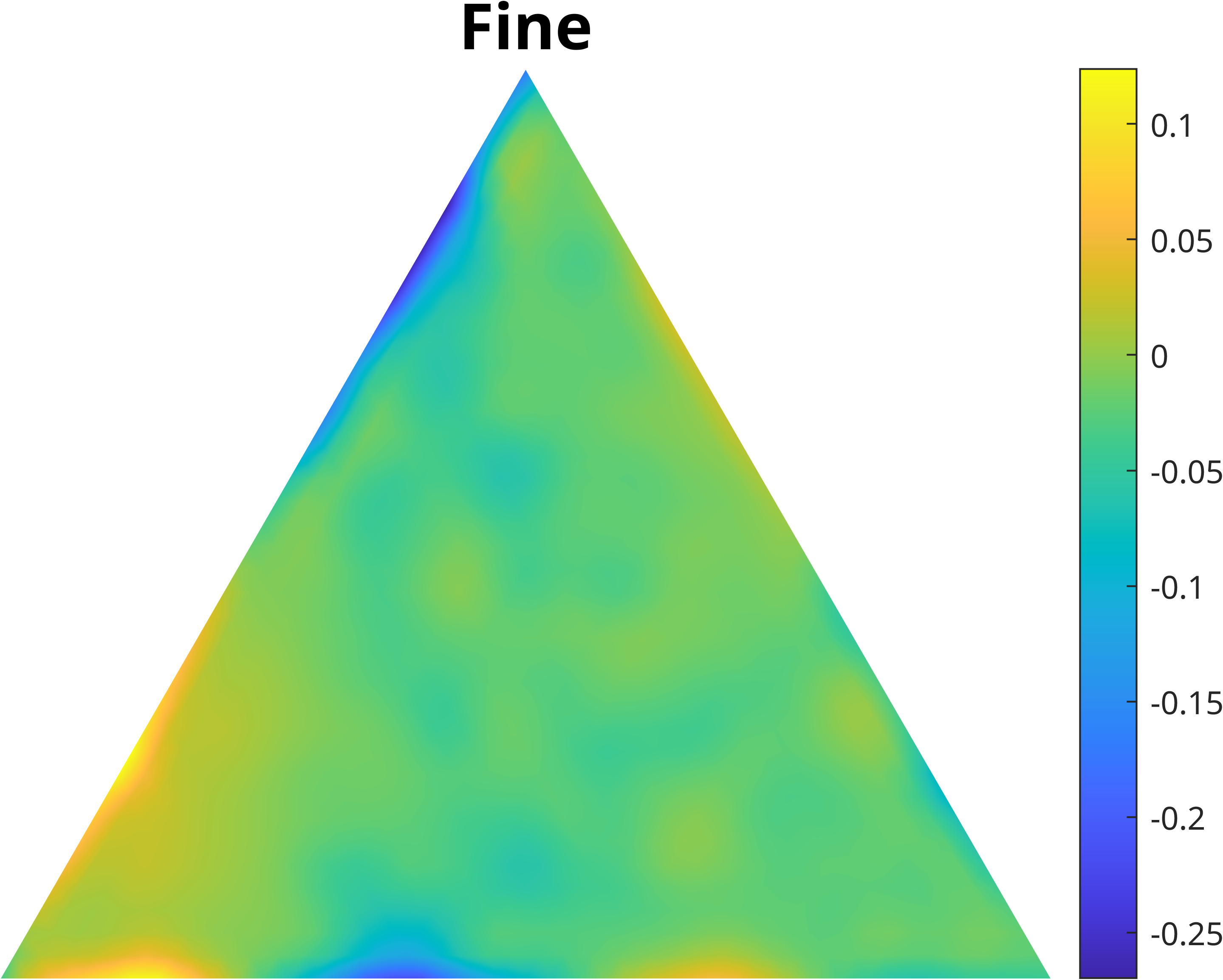}
        \caption{}
    \end{subfigure}
    \hfill
    \begin{subfigure}[b]{0.49\textwidth}
        \centering
        \includegraphics[width=\textwidth]{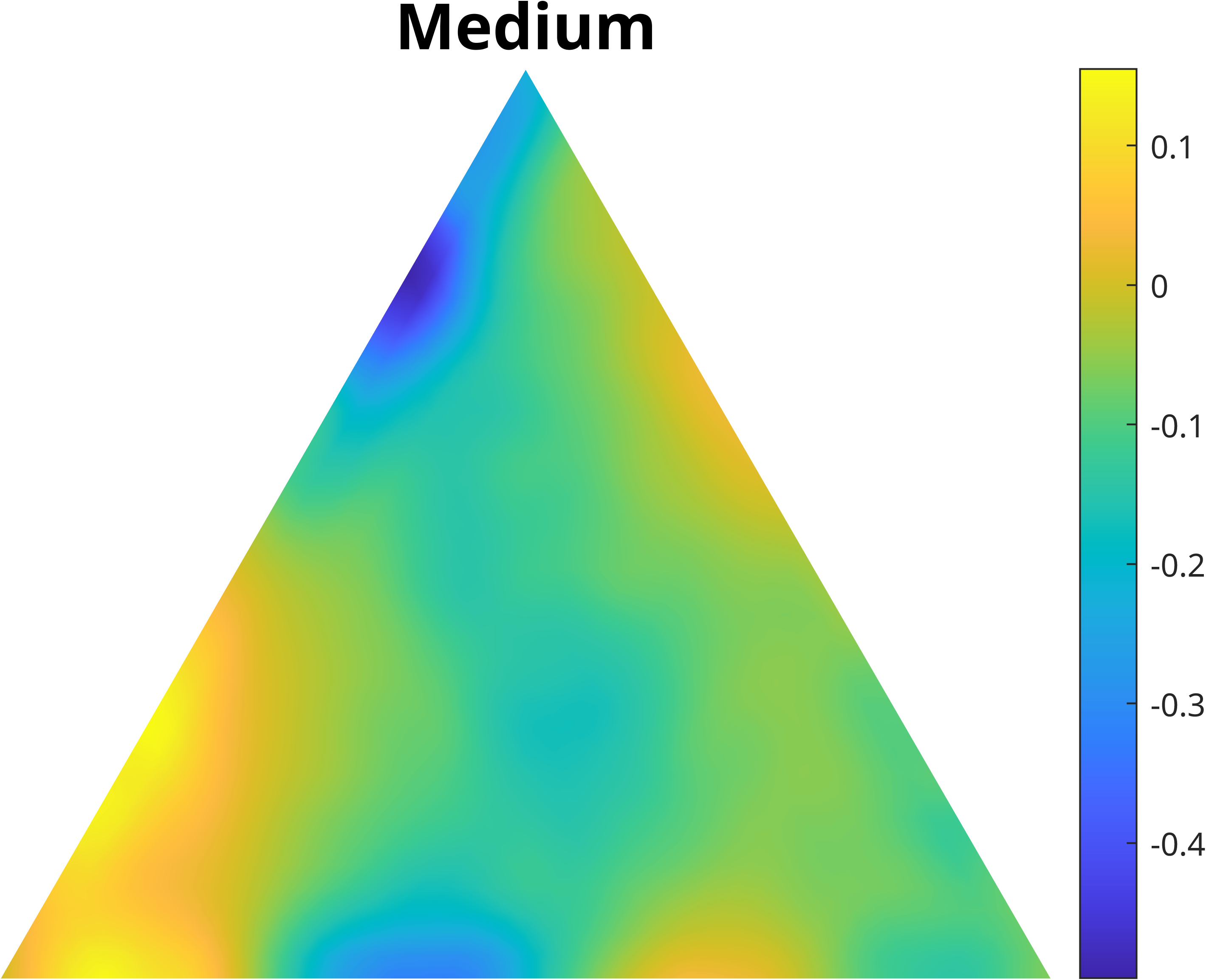}
        \caption{}
    \end{subfigure}
    \hfill
    \begin{subfigure}[b]{0.49\textwidth}
        \centering
        \includegraphics[width=\textwidth]{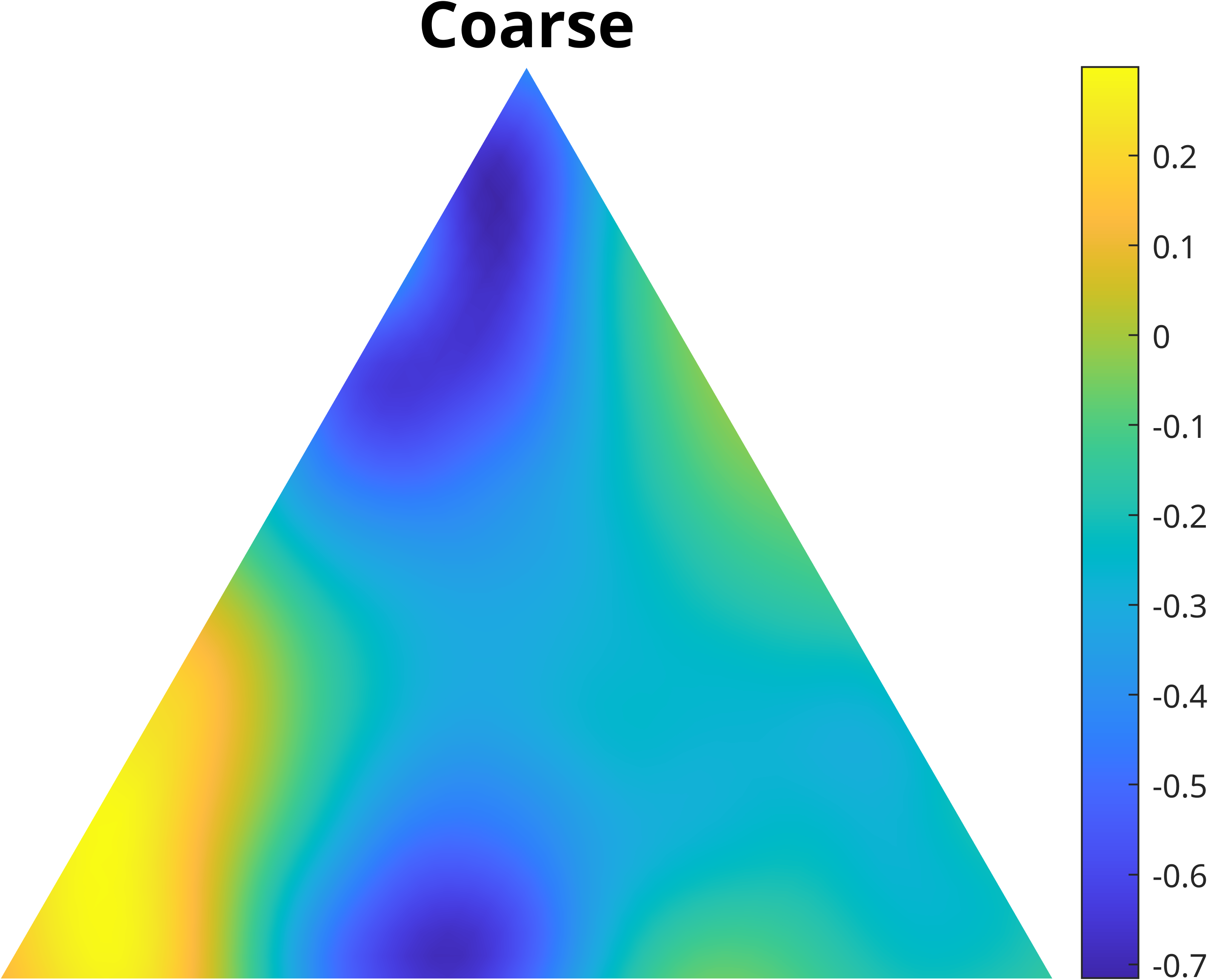}
        \caption{}
    \end{subfigure}
    \hfill
    \begin{subfigure}[b]{0.49\textwidth}
        \centering
        \includegraphics[width=\textwidth]{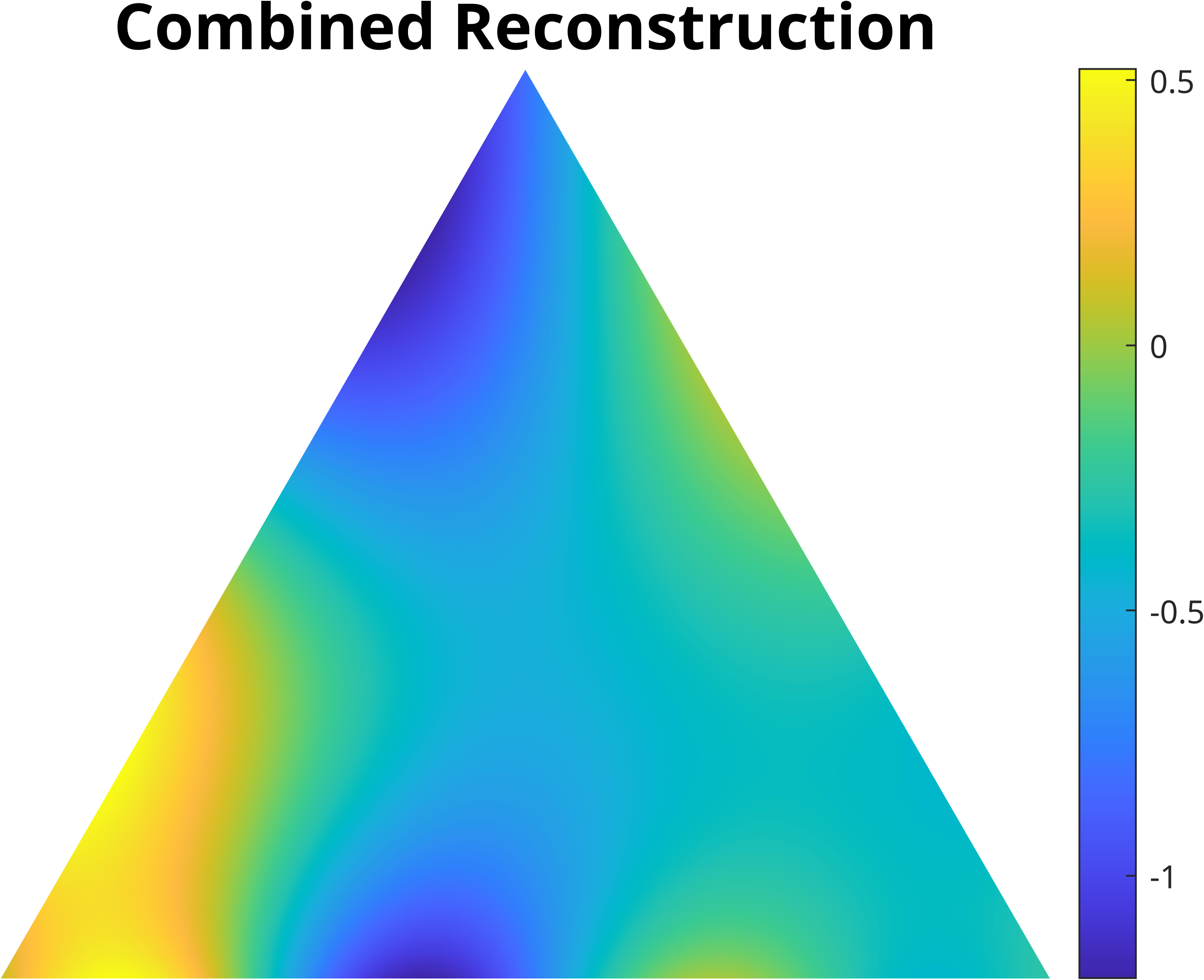}
        \caption{}
    \end{subfigure}
    \caption{Multiscale decomposition of the triangular-domain Darcy flow problem at three different scales. The bottom right image shows the combined reconstruction.}
    \label{fig:darcy_triangle_scales}
\end{figure}
\FloatBarrier
\paragraph{Darcy Triangle}
In Figure~\ref{fig:darcy_triangle_scales}, the coarse scale captures the dominant behavior throughout the interior of the domain, while the medium scale resolves localized boundary effects and emerging interior features. The fine scale further refines these localized structures, capturing the smallest variations near the boundary and throughout the interior.

\begin{figure}[htbp]
    \centering

    \begin{subfigure}[t]{0.49\textwidth}
        \centering
        \includegraphics[width=\linewidth]{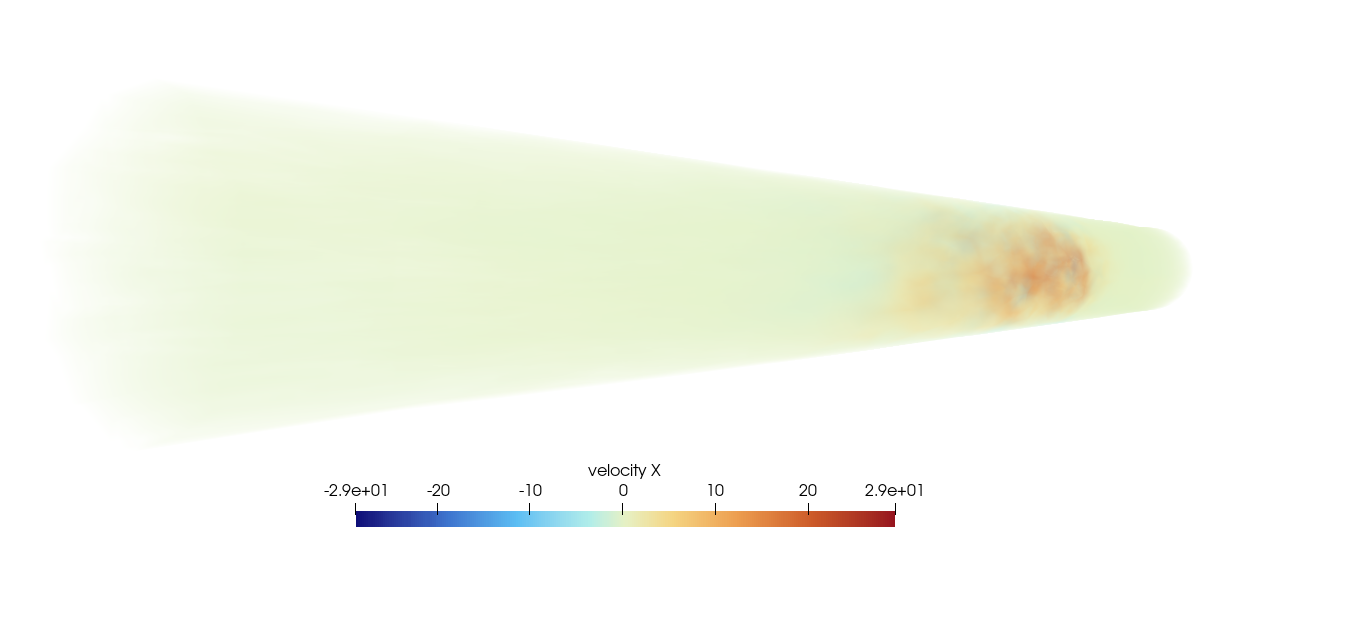}
        \caption{Fine}
    \end{subfigure}
    \hfill
    \begin{subfigure}[t]{0.49\textwidth}
        \centering
        \includegraphics[width=\linewidth]{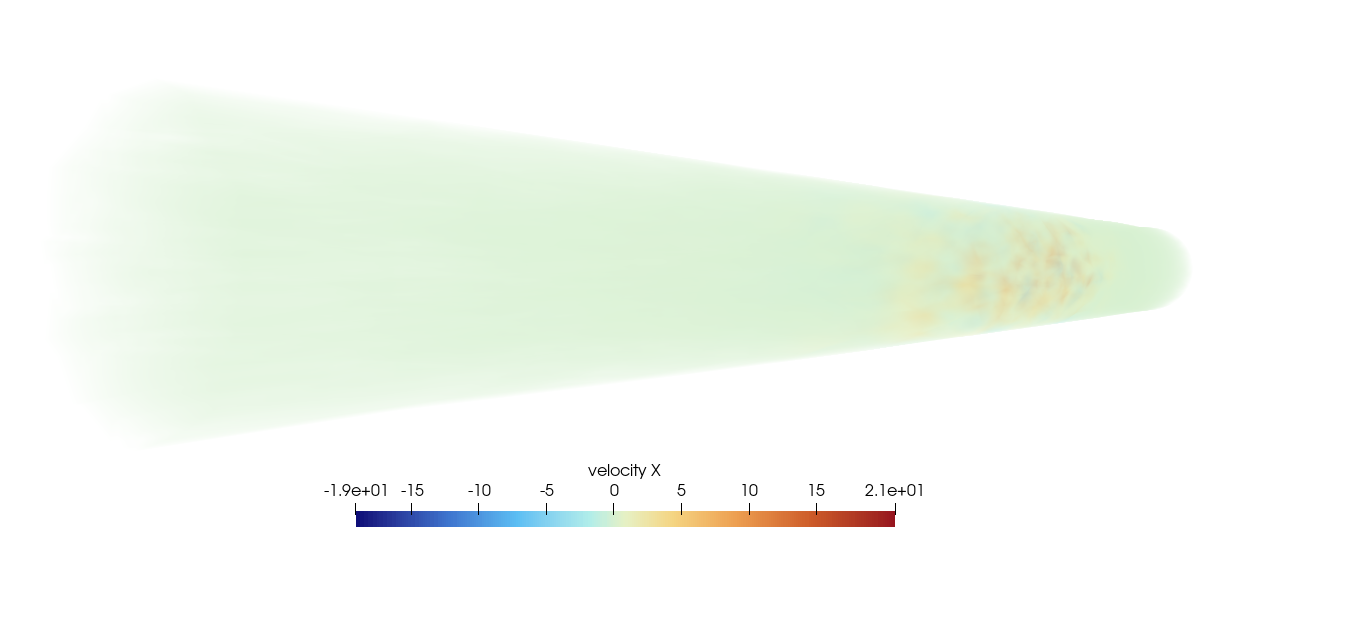}
        \caption{Medium}
    \end{subfigure}

    \begin{subfigure}[t]{0.49\textwidth}
        \centering
        \includegraphics[width=\linewidth]{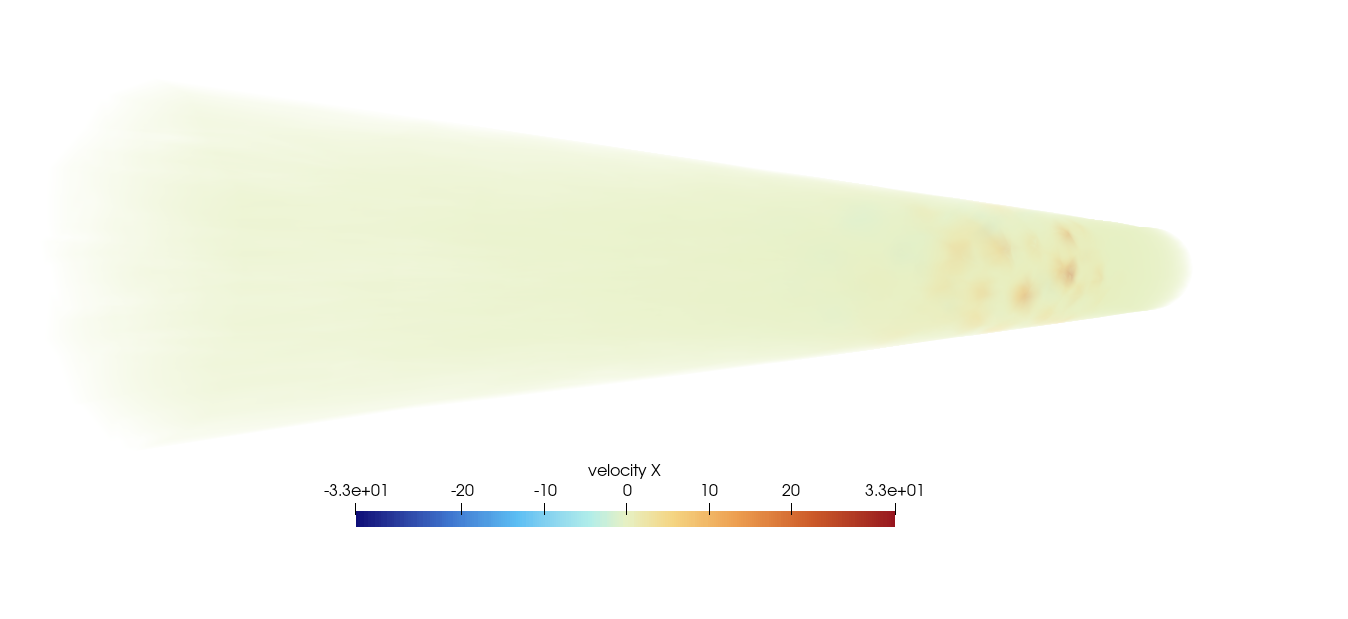}
        \caption{Coarse}
    \end{subfigure}
    \hfill
    \begin{subfigure}[t]{0.49\textwidth}
        \centering
        \includegraphics[width=\linewidth]{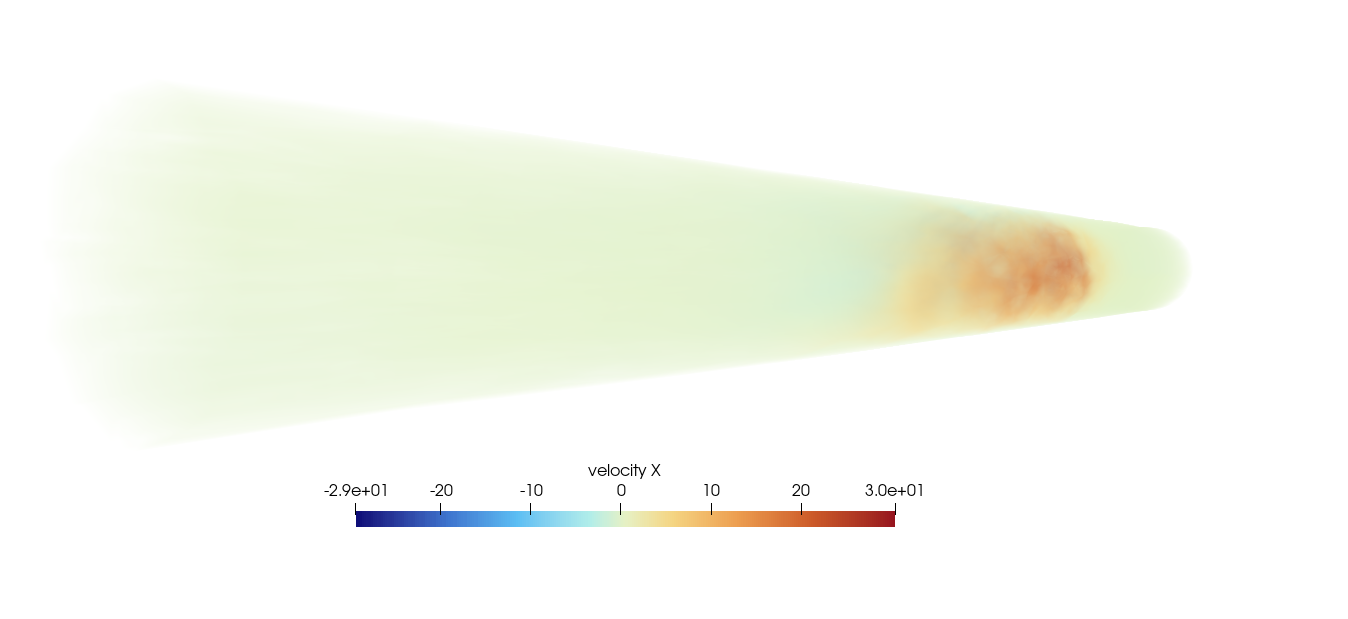}
        \caption{Combined}
    \end{subfigure}

    \caption{Multiscale decomposition of the species transport velocities in the $x$ direction at three different scales. The bottom right image shows the combined reconstruction.}
    \label{fig:multiscale_x}
\end{figure}
\FloatBarrier
\paragraph{Species Transport Along $x$} 
In Figure~\ref{fig:multiscale_x}, the coarse scale captures only weak large-scale variations, while most of the solution structure is resolved at the finest scale. This behavior reflects the relatively small extent of the flow in the $x$-direction, for which the coarse scale support is comparatively large. Despite this, the finer scales clearly capture the localized mixing induced by the blades as the gases are transported through the pipe.
\begin{figure}[htbp]
\centering

    \begin{subfigure}[t]{0.49\textwidth}
        \centering
        \includegraphics[width=\linewidth]{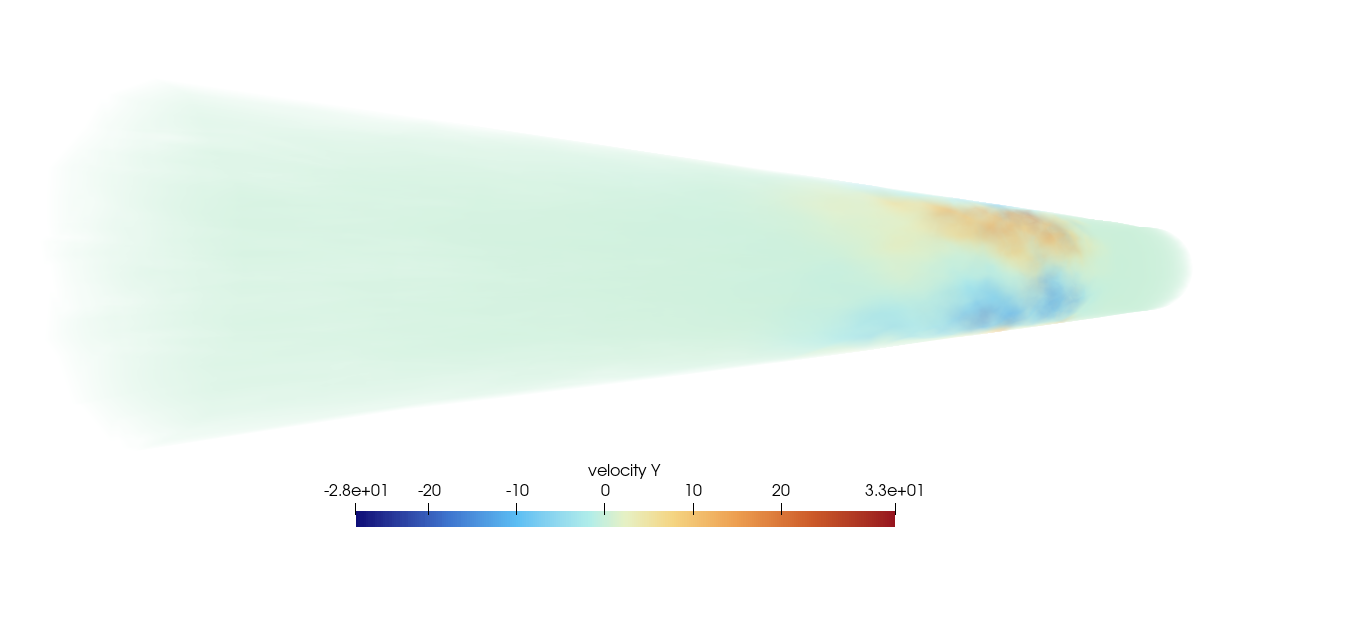}
        \caption{Fine}
    \end{subfigure}
    \hfill
    \begin{subfigure}[t]{0.49\textwidth}
        \centering
        \includegraphics[width=\linewidth]{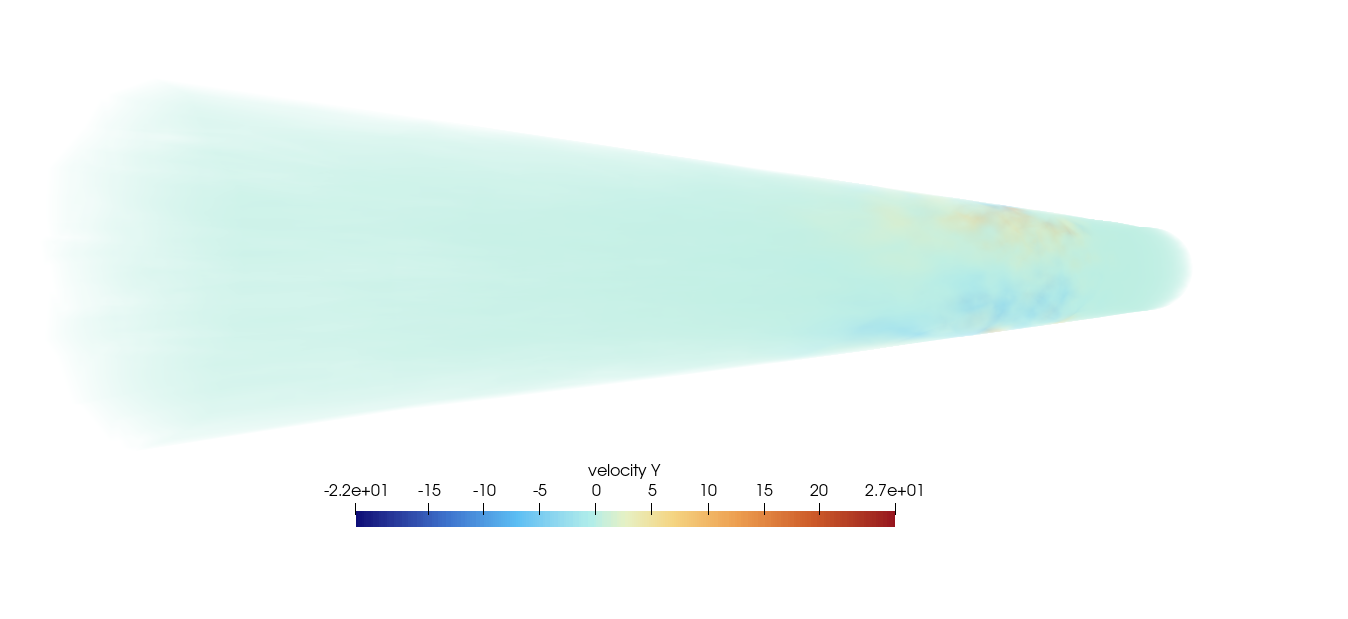}
        \caption{Medium}
    \end{subfigure}

    \begin{subfigure}[t]{0.49\textwidth}
        \centering
        \includegraphics[width=\linewidth]{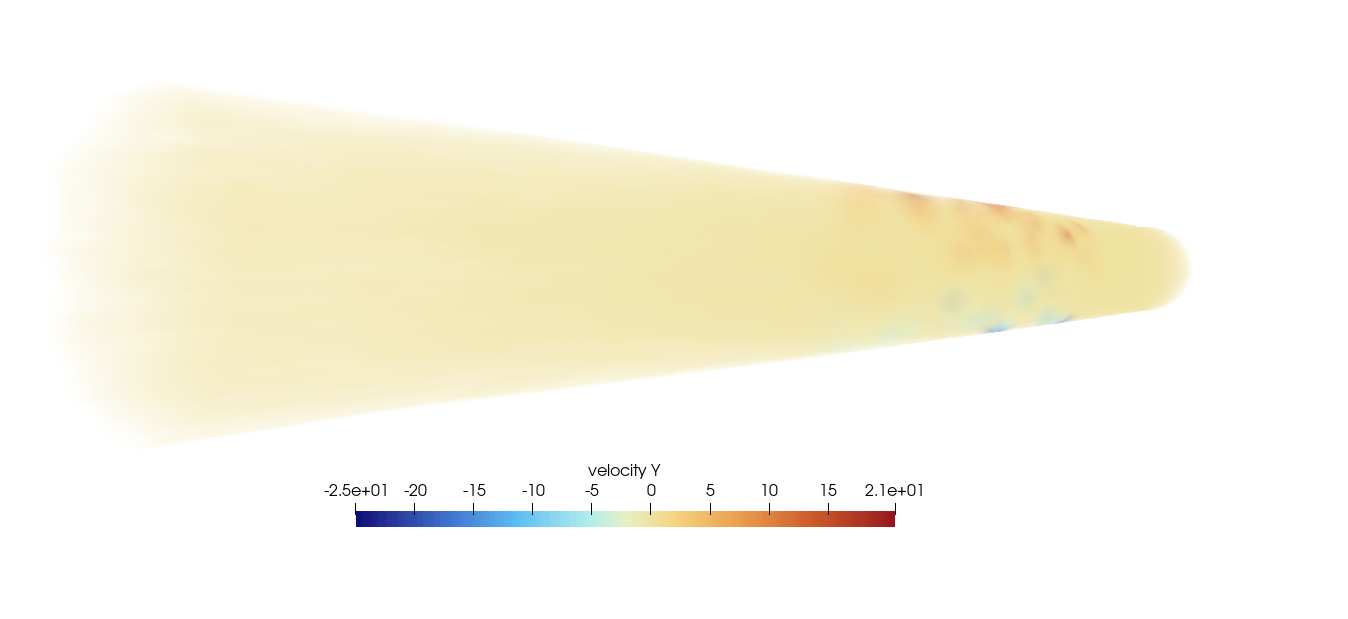}
        \caption{Coarse}
    \end{subfigure}
    \hfill
    \begin{subfigure}[t]{0.49\textwidth}\paragraph{Species transport along $y$}
        \centering
        \includegraphics[width=\linewidth]{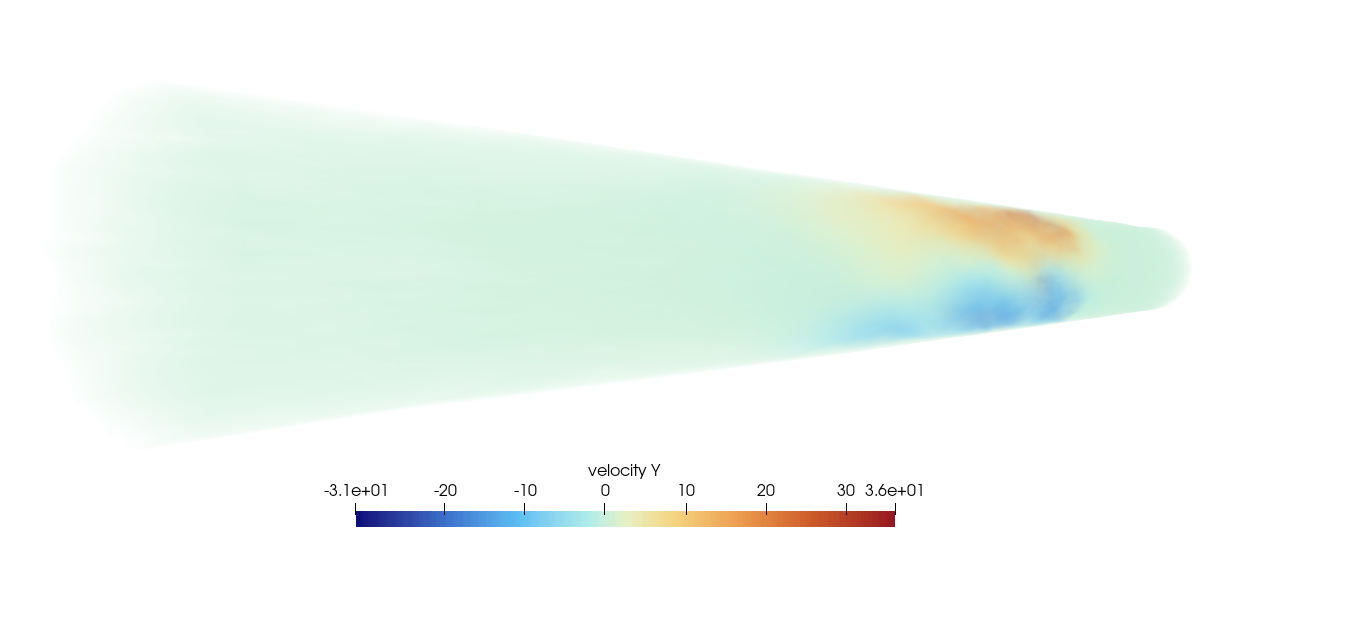}
        \caption{Combined}
    \end{subfigure}

    \caption{Multiscale decomposition of the species transport velocities in the $y$ direction at three different scales. The bottom right image shows the combined reconstruction}
    \label{fig:multiscale_y}
\end{figure}
\FloatBarrier
\paragraph{Species Transport Along $y$} 
In Figure~\ref{fig:multiscale_y}, the coarse scale captures only a few broad disturbances, primarily along the center of the pipe where the largest spatial scales are present. The intermediate and fine scales resolve the flow induced by the blades, capturing the localized mixing as the gases are driven upward from the bottom of the pipe and downward from the top, producing the characteristic rotational motion within the cross-section.

\end{document}